%% file: minimax_deep_regression_tit.tex
\documentclass[11pt,onecolumn,letterpaper,final]{IEEEtran}
\IEEEoverridecommandlockouts

\usepackage[colorlinks=true,citecolor=blue,linkcolor=blue]{hyperref}
\newcommand{\circled}[1]{\small{\raisebox{.6pt}{\textcircled{\raisebox{-.8pt}{#1}}}}}

\usepackage{amssymb}
\usepackage{amsmath,mathrsfs,dsfont}
\usepackage{nicefrac}
\usepackage{algorithm}
\usepackage{algorithmicx}
\usepackage{algpseudocode}

\usepackage{booktabs}

\usepackage{color}
\usepackage{enumitem}

\usepackage{array,cite}
\usepackage{graphicx,tikz}
\usepackage[mathscr]{euscript}
\usepackage{amsthm}
\usepackage{cite}

\usepackage{bm}
\usepackage{bbm}
\usepackage{color}

\usepackage{color}
\usepackage{epstopdf}
\usepackage{subcaption}
\usepackage{cleveref}
\usepackage{thmtools}
\usepackage{thm-restate}

\usepackage{bbding}
\usepackage{mathtools}
\usepackage{wrapfig}
\usepackage{colortbl}
\usepackage{slashbox}
\usepackage[htt]{hyphenat}
\allowdisplaybreaks
\newcommand{\cfrakR}{\mathfrak{R}} %\usepackage{eufrak}
\newcommand{\relu}[1]{\sigma\pth{#1}}

\newcommand{\tK}{\tilde K}

\newcommand{\tbx}{\tilde \bx}
\newcommand{\tby}{{\tilde \by}}
\newcommand{\tbu}{\tilde \bu}
\newcommand{\tbv}{\tilde \bv}

\newcommand{\bbx}{\overset{\rightharpoonup}{\bx}}
\newcommand{\tbbx}{\overset{\rightharpoonup}{\tilde \bx}}
\newcommand{\bbw}{\overset{\rightharpoonup}{\bw}}
\newcommand{\bbe}{\overset{\rightharpoonup}{\be}}

\newcommand{\cFnn}{\cF_{\mathop\mathrm{NN}}}

\newcommand{\poly}{\mathop\mathrm{poly}}

\title{Gradient Descent Finds Over-Parameterized Neural Networks with Sharp Generalization for Nonparametric Regression}

\author{Yingzhen Yang$^{1}$ and Ping Li$^{2}$,
\\
$^{1}$Arizona State University,
699 S Mill RD, Tempe, AZ 85281, USA\\
$^{2}$VecML Inc., Bellevue, WA 98004, USA \\
}
\graphicspath{{illustrations/}}
\input{symbol_tit.tex}

\begin{document}

\maketitle

\begin{abstract}
We study nonparametric regression by an over-parameterized two-layer neural network trained by gradient descent (GD) in this paper. We show that, if the neural network is trained by GD with early stopping, then the trained network renders a sharp rate of the nonparametric regression risk of $\cO(\eps_n^2)$,  which is the same rate as that for the classical kernel regression trained by GD with early stopping, where $\eps_n$ is the critical population rate of the Neural Tangent Kernel (NTK) associated with the network and $n$ is the size of the training data. It is remarked that our result does not require distributional assumptions about the covariate as long as the covariate is bounded, in a strong contrast with many existing results which rely on specific distributions of the covariates such as the spherical uniform data distribution or distributions satisfying certain restrictive conditions. The rate $\cO(\eps_n^2)$ is known to be minimax optimal for specific cases, such as the case that the NTK has a polynomial eigenvalue decay rate which happens under certain distributional assumptions on the covariates. Our result formally fills the gap between training a classical kernel regression model and training an over-parameterized but finite-width neural network by GD for nonparametric regression without distributional assumptions on the bounded covariate. We also provide confirmative answers to certain open questions or address particular concerns in the literature of training over-parameterized neural networks by GD with early stopping for nonparametric regression, including the characterization of the stopping time, the lower bound for the network width, and the constant learning rate used in GD.
\end{abstract}

\section{Introduction}
With the stunning success of deep learning in various areas of machine learning~\cite{YannLecunNature05-DeepLearning},
generalization analysis for neural networks is of central interest
for statistical learning learning and deep learning. Considerable efforts
have been made to analyze the optimization of deep neural networks showing that gradient descent (GD)
and stochastic gradient descent (SGD) provably achieve vanishing training loss
\cite{du2018gradient-gd-dnns,AllenZhuLS19-convergence-dnns,DuLL0Z19-GD-dnns,AroraDHLW19-fine-grained-two-layer,ZouG19,SuY19-convergence-spectral}.
There are also extensive efforts devoted to generalization analysis of deep neural networks (DNNs) with algorithmic guarantees,
that is, the generalization bounds for neural networks trained by gradient descent or its variants.
It has been shown that with sufficient over-parameterization, that is, with enough number of neurons in hidden layers,
the training dynamics of deep neural networks (DNNs) can be approximated by that of a kernel method with the kernel induced by the neural network architecture, termed the Neural Tangent Kernel (NTK)~\cite{JacotHG18-NTK}, while other studies such as \cite{YangH21-feature-learning-infinite-network-width} show that infinite-width neural networks can still learn features.
The key idea of NTK based generalization analysis is that, for highly over-parameterized networks, the network weights almost remain around their random initialization. As a result, one can use the first-order Taylor expansion around initialization to approximate the neural network functions and analyze their generalization capability~\cite{CaoG19a-sgd-wide-dnns,
AroraDHLW19-fine-grained-two-layer,Ghorbani2021-linearized-two-layer-nn}. %Recently, due to the limitation of the linear regime specified by NTK, there are works aiming to escape
%the linear regime of NTK. For example,~\cite{BaiL20-quadratic-NTK} uses higher-order approximation to the neural network function with better generalization than the linearization regime under certain conditions.
%In addition, some constructed DNNs have sparse connections which may not exist in the current over-parameterized network architecture.
%To the best of our knowledge, a mostly related work with algorithmic guarantee is~\cite{HuWLC21-regularization-minimax-uniform-spherical} which proves that an over-parameterized two-layer NN trained by GD with $\ell^2$-regularization achieves minimax optimal rate for nonparametric regression with spherical uniform data distribution.

Many existing works in generalization analysis of neural networks focus on clean data, but it is a central problem in statistical learning that how neural networks can obtain sharp convergence rates for the risk of nonparametric regression where the observed data are corrupted by noise. Considerable research has been conducted in this direction which shows that various types of DNNs achieve optimal convergence rates for smooth~\cite{Yarotsky17-error-approximation-dnns,Bauer2019-regression-dnns-curse-dim,SchmidtHieber2020-regression-relu-dnns,Jiao2023-regression-polynomial-prefactors,ZhangW23-weight=decay-dnns} or non-smooth~\cite{ImaizumiF19-nonsmooth-regression} target functions for nonparametric regression. However, many of these works do not have algorithmic guarantees, that is, the DNNs in these works are constructed specially to achieve optimal rates with no guarantees that an optimization algorithm, such as GD or its variants, can obtain such constructed DNNs.
To this end, efforts have been made in the literature
to study the minimax optimal  risk rates for nonparametric regression with over-parameterized neural networks trained by GD with either
early stopping~\cite{Li2024-edr-general-domain} or
$\ell^2$-regularization~\cite{HuWLC21-regularization-minimax-uniform-spherical,
SuhKH22-overparameterized-gd-minimax}. However, most existing works
either require spherical uniform distribution for the covariates
\cite{HuWLC21-regularization-minimax-uniform-spherical,
SuhKH22-overparameterized-gd-minimax} or certain restrictive conditions on the distribution of the covariates~\cite{Li2024-edr-general-domain}.

It remains an interesting and important question for the statistical learning and theoretical deep learning literature that if an over-parameterized neural network trained by GD can achieve sharp risk rates for nonparametric regression without assumptions or restrictions on the distribution of the covariates, so that theoretical guarantees can be obtained for data in more practical scenarios. In this paper, we give a confirmative answer to this question. We present sharp risk rate which is distribution-free in the bounded covariate for nonparametric regression with an over-parameterized two-layer NN trained by GD with early stopping. Throughout this paper, distribution-free in the bounded covariate means that there are no distributional assumptions about the covariate as long as
the covariate lies on a bounded (or compact) input space. Furthermore, our results give confirmative answers to certain open questions or address particular concerns in the literature of training over-parameterized
neural networks by GD with early stopping for nonparametric regression with minimax optimal rates, including the characterization of the stopping time in the early-stopping mechanism, the lower bound for the network width, and the constant learning rate used in GD. Benefiting from our analysis which is distribution-free in the bounded covariate, our answers to these open questions or concerns do not
require distributional assumptions about bounded covariate. Section~\ref{sec:summary-main-results} summarizes our main results with their significance and comparison to relevant existing works.

We organize this paper as follows. We first introduce the necessary notations in the remainder of this section. We then introduce in Section~\ref{sec:setup} the problem setup for nonparametric regression. Our main results are summarized in Section~\ref{sec:summary-main-results} and formally introduced in Section~\ref{sec:main-results}.
The training algorithm for the over-parameterized two-layer neural network is introduced in Section~\ref{sec:training}. The roadmap of proofs and the novel proof strategy of this work are presented in Section~\ref{sec:proof-roadmap}.
%We present simulation results in Section~\ref{sec:simulation} showing the effectiveness of GD compared to GD.
%\subsection{Notations}
%\label{sec:notations}

\vspace{0.1in}\noindent \textbf{Notations.} We use bold letters for matrices and vectors, and regular lower letter for scalars throughout this paper. The bold letter with a single superscript indicates the corresponding column of a matrix, e.g., $\bA^{(i)}$ is the $i$-th column of matrix $\bA$, and the bold letter with subscripts indicates the corresponding rows of elements of a matrix or vector. We put an arrow on top of
a letter with subscript if it denotes a vector, e.g.,
$\bbx_i$ denotes the $i$-th training
feature. $\norm{\cdot}{F}$ and
$\norm{\cdot}{p}$ denote the Frobenius norm and the vector $\ell^{p}$-norm or the matrix $p$-norm. $[m\relcolon n]$ denotes all the integers between $m$ and $n$ inclusively, and $[1\relcolon n]$ is also written as $[n]$. $\Var{\cdot}$ denotes the variance of a random variable. $\bI_n$ is a $n \times n$ identity matrix.  $\indict{E}$ is an indicator function which takes the value of $1$ if event $E$ happens, or $0$ otherwise. The complement of a set $A$ is denoted by $A^c$, and $\abth{A}$ is the cardinality of the set $A$. $\vect{\cdot}$ denotes the vectorization of a matrix or a set of vectors, and $\tr{\cdot}$ is the trace of a matrix.
We denote the unit sphere in $d$-dimensional Euclidean space by $\unitsphere{d-1} \defeq \{\bx \colon  \bx \in \RR^d, \ltwonorm{\bx} =1\}$. Let $\cX$ denote the input space, and
$L^2(\cX, \mu) $ denote the space of square-integrable functions on $\cX$ with probability measure $\mu$, and the inner product $\iprod{\cdot}{\cdot}_{L^2(\mu)}$ and $\norm{\cdot}{{L^2(\mu)}}^2$ are defined as $\iprod{f}{g}_{L^2(\mu)} \coloneqq \int_{\cX}f(x)g(x) \diff \mu(x)$ and $\norm{f}{L^2(\mu)}^2 \coloneqq \int_{\cX}f^2(x) \diff \mu (x) <\infty$. $\ball{}{\bx}{r}$ is the Euclidean closed ball centered at $\bx$ with radius $r$. Given a function $g \colon \cX \to \RR$, its $L^{\infty}$-norm is denoted by $\norm{g}{\infty} \defeq \sup_{\bx \in \cX} \abth{g(\bx)}$, and
$L^{\infty}$ is the function class whose elements have bounded $L^{\infty}$-norm. $\iprod{\cdot}{\cdot}_{\cH}$ and $\norm{\cdot}{\cH}$ denote the inner product and the norm in the Hilbert space $\cH$. $a = \cO(b)$ or $a \lsim b$ indicates that there exists a constant $c>0$ such that $a \le cb$. $\tilde \cO$ indicates there are specific requirements in the constants of the $\cO$ notation. $a = o(b)$ and $a = w(b)$ indicate that $\lim \abth{a/b}  = 0$ and $\lim \abth{a/b}  = \infty$, respectively. $a \asymp b$  or $a = \Theta(b)$ denotes that
there exists constants $c_1,c_2>0$ such that $c_1b \le a \le c_2b$. $\Unif{\unitsphere{d-1}}$ denotes the uniform distribution on $\unitsphere{d-1}$.
%hroughout this paper we let the input space be $\cX = \unitsphere{d-1}$, and
%${\rm supp}(\cdot)$ is the support of a vector, $\mathbb P_{\cS'}$ is the operator of orthogonal projection onto the subspace $\cS'$.
%$\sigma_{t}(\cdot)$ denotes the $t$-th largest singular value of a matrix, $\sigma_{\max}(\cdot)$ and $\sigma_{\min}(\cdot)$ indicate the largest and smallest singular value of a matrix respectively.
The constants defined throughout this paper may change from line to line.
For a Reproducing Kernel Hilbert Space $\cH$, we use $\cH(\mu_0)$ to denote
the ball centered at the origin with radius $\mu_0$ in $\cH$.
We use $\Expect{P}{\cdot}$ to denote the expectation with respect to the distribution $P$.
\vspace{-.05in}
\section{Problem Setup}
\label{sec:setup}
We introduce the problem setup for nonparametric regression in this section.
\subsection{Two-Layer Neural Network}
\label{sec:setup-two-layer-nn}

We are given the training data $\set{(\bbx_i,  y_i)}_{i=1}^n$ where each data point is a tuple of feature vector $\bbx_i \in \cX \subseteq \RR^d$ and its response $y_i \in \RR$. Throughout this paper we assume
that no two training features coincide, that is, $\bbx_i \neq \bbx_j$ for all $i,j \in [n]$ and $i \neq j$.
We denote the training feature vectors by $\bS = \set{\bbx_i}_{i=1}^n$, and denote by $P_n$ the empirical distribution over $\bS$. All the responses are stacked as a vector $ \by = [ y_1, \ldots,  y_n]^\top \in \RR^n$.
The response $ y_i$ is given by $ y_i= f^*(\bbx_i) + w_i$ for
$i \in [n]$,
where $\set{w_i}_{i=1}^n$ are i.i.d. sub-Gaussian random noise with mean $0$ and variance proxy $\sigma_0^2$, that is, $\Expect{}{\exp(\lambda w_i)} \le \exp(\lambda^2 \sigma_0^2/2)$ for any $\lambda \in \RR$.
$f^*$ is the target function to be detailed later. We define $\by \defeq \bth{y_1,\ldots,y_n}$, $\bw \defeq \bth{w_1,\ldots,w_n}^{\top}$, and use $f^*(\bS) \defeq \bth{f^*(\bbx_1),\ldots,f^*(\bbx_n)}^{\top}$ to denote the clean target labels.
The feature vectors in $\bS$ are drawn i.i.d. according to an underlying unknown data distribution $P$ supported on $\cX$ with the probability measure $\mu$. In this work, $P$  is absolutely continuous with respect to the usual Lebesgue measure in $\RR^d$.

For a vector $\bx \in \RR^d$, we let $\tbx = \bth{\bx^{\top} \,\, 1}^{\top} \in \RR^{d+1}$ obtained by appending $1$ at the last coordinate of $\bx$.
%The weakness of Assumption~\ref{assumption:weak-assump-covariates-target} is to be describe is the weakest assumption among the existing works about training neural networks for nonparametric regression with minimax optimal rates and algorithmic guarantees
%\vspace{-.05in}
We consider a two-layer neural network (NN) in this paper whose
mapping function is
\bal\label{eq:two-layer-nn}
&f(\bW,\bx) =  \frac{1}{\sqrt{m}}\sum_{r=1}^{m}
a_r \relu{{\bbw_r}^\top \tbx},
\eal%
where $\bx \in \cX$ is the input, $\sigma(\cdot) = \max\set{\cdot,0}$ is the ReLU activation function, $\bW = \set{\bbw_r}_{r=1}^m$ with $\bbw_r \in \RR^{d+1}$ for $r \in [m]$ denotes the weighting vectors in the first layer and $m$ is the number of neurons. $\ba = \bth{a_1, \ldots, a_m} \in \RR^m$ denotes the weights of the second layer.
It is noted that the construction of $\tbx$ allows for learning the bias $\bth{\bbw_r}_{d+1}$ in the first layer by ${\bbw_r}^\top \tbx = \bth{\bbw_r}_{1\colon d}^{\top} \bx + \bth{\bbw_r}_{d+1}$
for $r \in [m]$. Throughout this paper we also write $\bW,\bw_r$ as $\bW_{\bS},\bw_{\bS,r}$ from time to time so as to indicate that the weights are trained on the training features $\bS$.
Moreover, we let $\tbbx_i = \bth{\bbx_i^{\top} \,\, 1}^{\top} \in \RR^{d+1}$ for all $i \in [n]$.

\subsection{Kernel and Kernel Regression for Nonparametric Regression}
\label{sec:setup-kernels-target-function}

We define the kernel function
\bal\label{eq:kernel-two-layer}
&K(\bu, \bv) \defeq  \frac{ \iprod{\tbu}{\tbv}}{2\pi} \pth{\pi -\arccos \iprod{\tbu/\ltwonorm{\tbu}}{\tbv/\ltwonorm{\tbv}}}, ~~~ \forall ~ \bu, \bv \in \cX,
\eal%
where $\tbu = \bth{\bu^{\top} \,\, 1}^{\top} \in \RR^{d+1}$ and $\tbv = \bth{\bv^{\top} \,\, 1}^{\top} \in \RR^{d+1}$.
It can be verified that $K$ is the NTK~\cite{JacotHG18-NTK} associated with
the two-layer NN (\ref{eq:two-layer-nn}) with the constant second layer weights $\ba$, and $K$ is a positive-definite (PD) kernel. Let the gram matrix of $K$ over the training data $\bS$ be $\bK \in \RR^{n \times n}, \bK_{ij} = K(\bbx_i,\bbx_j)$ for $i,j \in [n]$,
and $\bK_n \defeq \bK/n$ is the empirical NTK matrix. Let the eigendecomposition  of $\bK_n$ be $\bK_n = \bU \bSigma {\bU}^{\top}$ where $\bU$ is a $n \times n$ orthogonal matrix, and $\bSigma$ is a diagonal matrix with its diagonal elements $\set{\hat \lambda_i}_{i=1}^n$ being eigenvalues of $\bK_n$ and sorted in a non-increasing order.  It is proved in existing works, such as~\cite{du2018gradient-gd-dnns}, that $\bK_n$ is non-singular.
Let $\cH_{K}$ be the Reproducing Kernel Hilbert Space (RKHS) associated with
$K$, and every element in $\cH_K$ is a function defined on $\cX$.
Since $K$ is continuous, it can be verified that if the $L^2$-norm of $K$
in the product space $L^2(\cX^2, \mu \otimes \mu )$ is finite, that is,
\bal\label{eq:K-L2-norm-bound}
\int_{\cX \times \cX} K^2(\bu,\bv) \diff \mu(\bu) \otimes \mu(\bv) < \infty,
\eal
then it follows from standard functional analysis for RKHS on both bounded and unbounded input spaces, such as \cite[Proposition 1]{RKHS-unbounded-domain-2005},
that the integral operator associated with $K$, $T_K \colon L^2(\cX,\mu) \to L^2(\cX,\mu)$,
$\allowdisplaybreaks (T_K f)(\bx) \defeq \int_{\cX} \allowdisplaybreaks K(\bx,\bx') f(\bx') \diff \mu(\bx')$ is a bounded, positive, self-adjoint, and compact operator on $L^2(\cX,\mu)$ By the spectral theorem, there is a countable orthonormal basis $\set{e_j}_{j \ge 1} \subseteq L^2(\cX,\mu)$ of $T_K(L^2(\cX,\mu))$. Moreover, $e_j$ is the eigenfunction of $T_K$ with
$\lambda_j$ being the corresponding eigenvalue, that is, $T_K e_j = \lambda_j e_j$ for all $j \ge 1$, and $\lambda_1 \ge \lambda_2 \ge \ldots > 0$.
%Let $\set{\mu_{\ell}}_{\ell \ge 1}$ be the distinct eigenvalues associated with $T_K$, and let
%$m_{\ell}$ be the be the sum of multiplicity of
%the eigenvalue $\set{\mu_{\ell'}}_{\ell'=1}^{\ell}$.
%That is, $m_{\ell'} - m_{\ell'-1}$ is the multiplicity
%of $\mu_{\ell'}$.
%It follows from the Mercer's theorem on bounded for unbounded domains, such as \cite[Theorem 2]{RKHS-unbounded-domain-2005}, that $K(\bu,\bv') = \sum_{j \ge 1} \lambda_j e_j(\bu) e_j(\bv)$.
It is well known that $\set{v_j = \sqrt {\lambda_j} e_j}_{j\ \ge 1}$ is an orthonormal basis of $\cH_K$. It follows from the Mercer's theorem that $K(\bu,\bv') = \sum_{j \ge 1} \lambda_j e_j(\bu) e_j(\bv)$. For a positive constant $\mu_0$, we define $\cH_{K}(\mu_0) \defeq \set{f \in \cH_{K} \colon \norm{f}{\cH} \le \mu_0}$ as the closed ball in $\cH_K$ centered at $0$ with radius $\mu_0$. We note that $\cH_{K}(\mu_0)$ is also specified by $\cH_{K}(\mu_0) = \{f \in L^2(\cX,\mu)\colon f = \sum_{j =1}^{\infty} \beta_j e_j,
\sum_{j = 1}^{\infty} \beta_j^2/\lambda_j \le \mu_0^2\}$.

We consider a bounded (or compact) input space $\cX$ throughout this paper, and we let  $\sup_{\bx \in \cX} \ltwonorm{\tbx} \le u_0$ with $u_0 > 1$ being a positive constant. Then $\sup_{\bu,\bv \in \cX} \abth{K(\bu, \bv)} \le u_0^2 /2$ which is finite, so that the condition (\ref{eq:K-L2-norm-bound}) holds and the above discussion about $K$ is valid.

\vspace{0.1in}\noindent \textbf{The task of nonparametric regression.} With $f^* \in \cH_K(\mu_0)$, the task of the analysis for nonparametric regression is to find an estimator $\hat f$ from the training data $\set{(\bbx_i,  y_i)}_{i=1}^n$ so that
the risk $\Expect{P}{\pth{\hat f - f^*}^2}$ converges to $0$ with a fast rate. We note that $\supnorm{f^*} \le \mu_0u_0/{\sqrt 2}$. We aim to establish a sharp rate of the risk where the estimator $\hat f$ is obtained from the over-parameterized neural network (\ref{eq:two-layer-nn}) trained by GD with early stopping.
%\begin{figure}[!htbp]
%    \centering
%    \begin{minipage}{.42\columnwidth}
%        \begin{algorithm}[H]
%        \renewcommand{\algorithmicrequire}{\textbf{Input:}}
%\renewcommand\algorithmicensure {\textbf{Output:} }
%\caption{Training the Two-Layer NN by GD}
%\label{alg:GD}
%\begin{algorithmic}[1]
%\State $\bW(T) \leftarrow$ Training-by-GD($T,\bW(0)$)
%\State \textbf{\bf input: } $T,\bW(0)$
%\State \textbf{\bf for } $t=1,\ldots,T$ \,\,\textbf{\bf do }
%\State \quad Perform the $t$-th step of GD by
%(\ref{eq:GD-two-layer-nn})
%\State \textbf{\bf end for }
%\State \textbf{\bf return} $\bW(T)$
%\end{algorithmic}
%\end{algorithm}
%\end{minipage}%
%    \hspace{2mm}
%    \begin{minipage}{.55\columnwidth}
%        \begin{algorithm}[H]
%\renewcommand{\algorithmicrequire}{\textbf{Input:}}
%\renewcommand\algorithmicensure {\textbf{Output:} }
%\caption{Training the Two-Layer NN by GD}
%\label{alg:GD}
%\begin{algorithmic}[1]
%\State $\bW(T) \leftarrow$ Training-by-GD($T,\bW(0),\bM$)
%\State \textbf{\bf input: } $T,\bW(0),\bM$
%\State \textbf{\bf for } $t=1,\ldots,T$ \,\,\textbf{\bf do }
%\State \quad Perform the $t$-th step of GD by
%(\ref{eq:GD-two-layer-nn})
%\State \textbf{\bf end for }
%\State \textbf{\bf return} $\bW(T)$
%\end{algorithmic}
%\end{algorithm}
%  \end{minipage}%
%\end{figure}
%

\vspace{0.1in}\noindent \textbf{Sharp rate of the risk of nonparametric regression using classical kernel regression.}
The statistical learning literature has established rich results in the
sharp convergence rates for the risk of nonparametric kernel regression~\cite{Stone1985,Yang1999-minimax-rates-convergence,
RaskuttiWY14-early-stopping-kernel-regression,
Yuan2016-minimax-additive-models}, with one representative result in
\cite{RaskuttiWY14-early-stopping-kernel-regression} about kernel regression trained by GD with early stopping.
Let $\eps_n$ be the critical population rate of the PD kernel $K$,
which is also referred to as the critical
radius~\cite{Wainwright2019-high-dim-statistics} of $K$.~\cite[Theorem 2]{RaskuttiWY14-early-stopping-kernel-regression} shows the following
sharp bound for the nonparametric regression risk of a kernel regression model trained by GD with early stopping when $f^* \in \cH_K(\mu_0)$. That is, with probability at least
$1-\Theta\pth{\exp(-\Theta(n\eps_n^2)}$,
\bal\label{eq:minimax-kernel-regression-bound}
\Expect{P}{\pth{f_{\hat T} - f^*}^2} \lsim \eps_n^2,
\eal
where $\hat T$ is the stopping time whose formal definition is deferred to Section~\ref{sec:kernel-complexity}, and $f_{\hat T} $ is the kernel regressor at the $\hat T$-th step of GD for the optimization problem of kernel regression.
The risk bound (\ref{eq:minimax-kernel-regression-bound}) is rather sharp, since it is minimax optimal in several popular learning setups, such as the setup where the eigenvalues $\set{\lambda_i}_{i \ge 1}$ exhibit a
certain polynomial decay. Such risk bound
 (\ref{eq:minimax-kernel-regression-bound}) also holds for a general PD kernel rather than the NTK (\ref{eq:kernel-two-layer}), and
the risk bound (\ref{eq:minimax-kernel-regression-bound}) is also minimax optimal when the PD kernel is low rank. It is also remarked that the risk bound (\ref{eq:minimax-kernel-regression-bound}) is distribution-free in the bounded covariate which does not require distributional assumptions on the covariates.
Interested readers are referred to ~\cite{RaskuttiWY14-early-stopping-kernel-regression} for more details.

The main result of this paper
is that the
over-parameterized two-layer NN (\ref{eq:two-layer-nn})
trained by GD with early stopping achieves the same order of risk rate as that in
(\ref{eq:minimax-kernel-regression-bound}) without additional distributional assumptions on the covariates as long as the input space $\cX$ is bounded or compact, which is summarized in the next section.

\section{Summary of Main Results}
\label{sec:summary-main-results}
We present a summary of our main results with a detailed discussion about the relevant literature in this section.
First, Theorem~\ref{theorem:LRC-population-NN-fixed-point} in Section~\ref{sec:main-results-details} shows that when the input space is bounded or compact, then the neural network (\ref{eq:two-layer-nn}) trained by GD with early stopping using Algorithm~\ref{alg:GD} enjoys a sharp rate of the nonparametric regression risk, $\cO\pth{\eps_n^2}$, which is the same as that for
the classical kernel regression in (\ref{eq:minimax-kernel-regression-bound}).
Such rate of  nonparametric regression risk in Theorem~\ref{theorem:LRC-population-NN-fixed-point} is distribution-free in the bounded covariate, which is close to practical scenarios without distributional assumptions on the bounded covariates. Table~\ref{table:main-results-comparison} compares our work
 to several existing works for nonparametric regression with a common setup, that is, $f^* \in \cH_{\tK}$ and the responses $\set{y_i}_{i=1}^n$ are corrupted by i.i.d. Gaussian or sub-Gaussian noise. In  Table~\ref{table:main-results-comparison},
the target function $f^*$ belongs to $\cH_{\tK}$, the RKHS associated with the NTK $\tK$ of the network in each particular existing work. We note that $\tK$ is the NTK of the network considered in a particular existing work which may not be the same as our NTK (\ref{eq:kernel-two-layer}).
%Theorem~\ref{theorem:LRC-population-NN-fixed-point} holds when 1) $P = \Unif{\cX}$ in \cite{HuWLC21-regularization-minimax-uniform-spherical,SuhKH22-overparameterized-gd-minimax}; 2) $P$ has bounded support; 3) $P$ is sub-Gaussian \cite[Proposition 13]{Li2024-edr-general-domain} or Gaussian.

Theorem~\ref{theorem:LRC-population-NN-fixed-point} immediately leads to minimax optimal rates obtained by several existing works, such as \cite[Theorem 5.2]{HuWLC21-regularization-minimax-uniform-spherical} and \cite[Theorem 3.11]{SuhKH22-overparameterized-gd-minimax}, as its special cases. In particular, when the eigenvalues of the integral operator associated with $K$ has a particular polynomial eigenvalue decay rate (EDR), that is,
$\lambda_j \asymp j^{-2\alpha}$ for all $j \ge 1$, then Corollary~\ref{corollary:minimax-nonparametric-regression} as a direct consequence of Theorem~\ref{theorem:LRC-population-NN-fixed-point} shows the sharp risk of the order
$\cO(n^{-2\alpha/(2\alpha+1)})$, which is minimax optimal for such polynomial EDR
\cite{Stone1985,Yang1999-minimax-rates-convergence,Yuan2016-minimax-additive-models}. The literature in training over-parameterized neural networks for nonparametric regression has
explored various distribution assumptions about $P$ so that such polynomial EDR holds. As an example of the polynomial EDR,
$\lambda_j \asymp j^{-(d+1)/d}$ for $j \ge 1$
happens for our NTK (\ref{eq:kernel-two-layer}) and the NTK in \cite{Li2024-edr-general-domain} if $\cX = \unitsphere{d-1}$ and $P = \Unif{\unitsphere{d-1}}$, or more generally
the probability density function of $P$ satisfies $p(\bx) \lsim (1+\ltwonorm{\bx}^2)^{-{(d+3)}}$ for all $\bx \in \cX$. Please refer to Proposition~\ref{proposition:EDR-unitsphere-Sd} for a proof for such polynomial EDR of the NTK (\ref{eq:kernel-two-layer}) which is based on the proofs in \cite{BiettiM19,BiettiB21,Li2024-edr-general-domain}. As another example of the polynomial EDR, Remark~\ref{remark:edr-on-unitsphere-in-Rd} explains that when $\cX = \unitsphere{d-1}$, $P = \Unif{\unitsphere{d-1}}$, and the bias of the neural network (\ref{eq:two-layer-nn}) is not learned (that is, $\bth{\bbw_r}_{d+1} = 0$ for all $r \in [m]$ when training the neural network by GD) and $f^*$ is in the RKHS associated with the corresponding NTK ($K_1$ defined in Remark~\ref{remark:edr-on-unitsphere-in-Rd}) with a bounded RKHS-norm, then the polynomial EDR
$\lambda_j \asymp j^{-d/(d-1)}$ for $j \ge 1$ holds for the corresponding NTK and
our main results still hold, leading to the minimax optimal rate of the order $\cO(n^{-d/(2d-1)})$.

It is remarked that all the results and discussions about the polynomial EDR in this paper are for the setting with fixed dimension $d \ge 4$, which is a widely adopted setting in the existing works such as
\cite{HuWLC21-regularization-minimax-uniform-spherical,Li2024-edr-general-domain}. The polynomial EDR, $\lambda_j \asymp j^{-d/(d-1)}$ for $j \ge 1$, holds for the NTK considered in
\cite{HuWLC21-regularization-minimax-uniform-spherical,SuhKH22-overparameterized-gd-minimax} under the assumption that $\cX = \unitsphere{d-1}$ and $P = \Unif{\unitsphere{d-1}}$. Under such assumption, the minimax optimal rate $\cO(n^{-d/(2d-1)})$ is obtained by \cite[Theorem 5.2]{HuWLC21-regularization-minimax-uniform-spherical} and \cite[Theorem 3.11]{SuhKH22-overparameterized-gd-minimax}, which is in fact derived as a special case of
our Corollary~\ref{corollary:minimax-nonparametric-regression} with $2\alpha = d/(d-1)$ under the same assumption with the minor changes mentioned above and detailed in
Remark~\ref{remark:edr-on-unitsphere-in-Rd}. In fact, without such changes Corollary~\ref{corollary:minimax-nonparametric-regression} leads to a slightly sharper and minimax optimal rate of the order $\cO(n^{-(d+1)/(2d+1)})$ with $2\alpha = (d+1)/d$. As another example, under the assumption that $P$ is sub-Gaussian supported on an unbounded $\cX \subseteq \RR^d$,
\cite[Proposition 13]{Li2024-edr-general-domain} shows the polynomial EDR, $\lambda_j \asymp j^{-(d+1)/d}$ for $j \ge 1$, holds, and in this case the minimax optimal rate is
$\cO(({\log n}/n)^{(d+1)/(2d+1)})$ for an unbounded input space according to \cite[Theorem 1]{Caponnetto2007-optimal-least-reguarlized} (with $c=1,b=(d+1)/d$ in that theorem).
The nearly-optimal rate of the order
$\cO(\log^2(1/\delta) \cdot n^{-(d+1)/(2d+1)})$
with probability $1-\delta$ and $\delta \in (0,1)$ is achieved by \cite[Proposition 13 with $s=1$]{Li2024-edr-general-domain}. $s=1$ in \cite[Proposition 13]{Li2024-edr-general-domain} ensures that the target function $f^*$
 is in the RKHS associated with the NTK $\tK$ of the DNN considered in \cite{Li2024-edr-general-domain} with a bounded RKHS-norm, which is the setup considered in this paper. We note that the same rate
 $\cO(\log^2(1/\delta) \cdot n^{-(d+1)/(2d+1)})$  can be obtained by applying the same proof strategy of \cite[Proposition 13]{Li2024-edr-general-domain} to the case with bounded input space when the probability density function $p$ of the covariate distribution $P$ satisfies a restrictive condition, that is,
 $p(\bx) \lsim (1+\ltwonorm{\bx}^2)^{-{(d+3)}}$ for all $\bx \in \cX$ according to \cite[Theorem 8 and Theorem 10]{Li2024-edr-general-domain}. In contrast, with a bounded input space
 $\cX$ and under the same assumption that $p(\bx) \lsim (1+\ltwonorm{\bx}^2)^{-{(d+3)}}$ for all $\bx \in \cX$, Proposition~\ref{proposition:EDR-unitsphere-Sd} shows that the
 polynomial EDR, $\lambda_j \asymp j^{-(d+1)/d}$
 for $j \ge 1$, holds for our NTK (\ref{eq:kernel-two-layer}), so that our Corollary~\ref{corollary:minimax-nonparametric-regression} leads to a sharper rate of
 the order $\cO(n^{-(d+1)/(2d+1)})$ which is also minimax optimal for a bounded input space. On the other hand, \cite[Proposition 13]{Li2024-edr-general-domain}
achieves the rate $\cO(\log^2(1/\delta) \cdot n^{-(s(d+1))/(s(d+1)+d)})$
for the target function in an
interpolation Hilbert space $\bth{\cH_{\tK}}^s$ for $s \ge 0$, and $\bth{\cH_{\tK}}^1 = \cH_{\tK}$.

%in this case $\eps^2_n \asymp n^{-\frac{d+1}{2d+1}}$ according to
%\cite[Corollary 3]{RaskuttiWY14-early-stopping-kernel-regression},
%and  Theorem~\ref{theorem:LRC-population-NN-fixed-point} renders
%the rate of the nonparametric regression risk of
%$\cO(n^{-\frac{d+1}{2d+1}})$ which is minimax optimal for this special case
%\cite{Stone1985,Yang1999-minimax-rates-convergence,Yuan2016-minimax-additive-models}
%We remark that one can always normalize the covariates to make them on $\cX$.

\begin{table}[t!]
        \centering
        \caption{Comparison between our result and the existing works with the minimax optimal risk rates and the distributional assumptions on the covariate
        for nonparametric regression by training over-parameterized neural networks with algorithmic guarantees. The listed results here are under a common and popular setup that $f^* \in \cH_{\tK}$ with a bounded RKHS-norm, where $\tK$ is the NTK of the network considered in a particular existing work, and the responses $\set{y_i}_{i=1}^n$ are corrupted by i.i.d. Gaussian or sub-Gaussian noise with zero mean. Discussions about more relevant works are presented in Section~\ref{sec:summary-main-results}. }
        \resizebox{1\linewidth}{!}{
        \begin{tabular}{|l|c|c|c|c|}
                \hline
                \textbf{Existing Works and Our Result}
                & \textbf{Distributional Assumptions on the Covariates} & \textbf{Eigenvalue Decay Rate (EDR)} &\textbf{Rate of Nonparametric Regression Risk }
\\ \hline
\cite[Theorem 2]{KuzborskijS21-minimax-early-stopping}
 & No  & -- &Not minimax optimal, $\sigma_0^2 + \cO(n^{-{2}/(2+d)})$ \\ \hline
\begin{tabular}{@{}c@{}}
\cite[Theorem 5.2]{HuWLC21-regularization-minimax-uniform-spherical},\\
 ~\cite[Theorem 3.11]{SuhKH22-overparameterized-gd-minimax}
\end{tabular}
&$P$ is $\Unif{\cX}$. & $\lambda_j \asymp j^{-d/(d-1)}$
&minimax optimal, $\cO(n^{-{d}/(2d-1)})$ \\ \hline
%\cite[Proposition 13 with $f^* \in \cH_{\tK}$ ]{Li2024-edr-general-domain}
%&\begin{tabular}{@{}c@{}}
%$P$ satisfies \\
%a restrictive condition: \\ the density $p(\bx)$ for $\bx \in \RR^d$ satisfies\\
%$p(\bx) \lsim (1+\ltwonorm{\bx}^2)^{-{(d+3)/2}}$. \end{tabular}
%%$P$ is sub-Gaussian.
%  &$\lambda_j \asymp j^{-\frac{d}{d-1}}$ &nearly minimax optimal,$ \cO((\log 1/\delta)^2 n^{-\frac{d+1}{2d+1}})$  \\ \hline
\begin{tabular}{@{}c@{}}
Our Results: Theorem~\ref{theorem:LRC-population-NN-fixed-point} \\
and Corollary~\ref{corollary:minimax-nonparametric-regression}
\end{tabular}
 &\cellcolor{blue!15}No assumption about $P$ as long as $\cX$ is bounded.  &\cellcolor{blue!15}
{No requirement for  EDR} &\cellcolor{blue!15} \begin{tabular}{@{}c@{}}
$\cO\pth{\eps_n^2}$, which leads to minimax optimal rates, \\
such as those claimed in \cite{HuWLC21-regularization-minimax-uniform-spherical,SuhKH22-overparameterized-gd-minimax} as special cases.
\end{tabular}
\\ \hline
       \end{tabular}
        }
\label{table:main-results-comparison}
\end{table}

We further note that \cite{KuzborskijS21-minimax-early-stopping} considers
a Lipschitz continuous target function $f^*$. Although the result in
\cite[Theorem 2]{KuzborskijS21-minimax-early-stopping} does not require distributional assumptions, its risk rate under this common setup ($f^* \in \cH_{\tK}$ and responses are corrupted by i.i.d. Gaussian or sub-Gaussian noise) is not minimax optimal due to the term $\sigma_0^2$ in the risk bound. In fact, when $ f^*$ is Lipschitz continuous,  the minimax optimal rate is $\cO(n^{-{2}/(2+d)})$~\cite{Gyorfi2002-distri-free-nonpara-regression}. We note that
\cite[Theorem 1]{KuzborskijS21-minimax-early-stopping} shows
the minimax optimal rate of $\cO(n^{-{2}/(2+d)})$, however, this rate is derived for the noiseless case where
the responses are not corrupted by noise. Furthermore, the other term
$\cO(n^{-{2}/(2+d)})$ in its risk bound suffers from the curse of dimension with a slow rate to $0$ for high-dimensional data.

%When $P$ is the uniform distribution on $\cX$,
%it is shown in~\cite{BiettiM19,BiettiB21} that
%$\lambda_j \asymp j^{-\frac{d}{d-1}}$ for $j \ge 1$,

Second, our results provide confirmative answers to several outstanding open questions or address particular concerns in the existing literature about training over-parameterized neural networks for nonparametric regression by GD with early stopping and sharp risk rates, which are detailed below.

\vspace{0.1in}\noindent \textbf{Stopping time in the early-stopping mechanism.}
An open question raised in
\cite{KuzborskijS21-minimax-early-stopping,
HuWLC21-regularization-minimax-uniform-spherical} is how to characterize the stopping time in the early-stopping mechanism when training the over-parameterized network by GD. Let $\hat T$ be the stopping time,
\cite[Proposition 13 with $s=1$]{Li2024-edr-general-domain} shows that the stopping time should satisfy $\hat T \asymp n^{(d+1)/(2d+1)}$ under the assumption
that $P$ is a sub-Gaussian distribution. In contrast,
Theorem~\ref{theorem:LRC-population-NN-fixed-point} provides a characterization of $\hat T$ showing that $\hat T \asymp \eps_n^{-2}$, which is distribution-free in the bounded covariate. Under such distributional assumption required by~\cite{Li2024-edr-general-domain},
it follows from~Proposition~\ref{proposition:EDR-unitsphere-Sd}
 that the polynomial EDR, $\lambda_j \asymp j^{-(d+1)/d}$ for all $j \ge 1$, holds for our
 NTK (\ref{eq:kernel-two-layer}), so that
$\eps_n^{-2} \asymp n^{(d+1)/(2d+1)}$ by \cite[Corollary 3]{RaskuttiWY14-early-stopping-kernel-regression}. As a result,  the stopping time  established by our Corollary~\ref{corollary:minimax-nonparametric-regression} is of the same order $\Theta(n^{(d+1)/(2d+1)})$ as that in~\cite[Proposition 13]{Li2024-edr-general-domain} with $s=1$.

\vspace{0.1in}\noindent \textbf{Lower bound for the network width $m$.} Our main result,
Theorem~\ref{theorem:LRC-population-NN-fixed-point}, requires that
the network width $m$, which is the number of  neurons in the first layer
of the two-layer NN (\ref{eq:two-layer-nn}), satisfies $m \gsim {d^{\frac 52}}/{\eps_n^{25}}$. Such lower bound for $m$ solely depends on $d$ and $\eps_n$. Under the polynomial EDR,
Corollary~\ref{corollary:minimax-nonparametric-regression}, as a direct consequence of Theorem~\ref{theorem:LRC-population-NN-fixed-point}, shows that
$m$ should satisfy $m \gsim
n^{\frac{25\alpha}{2\alpha+1}} d^{\frac 52}$ so that GD with early stopping leads to the minimax rate of $\cO(n^{-{2\alpha}/(2\alpha+1)})$. We remark that this is the first time that the lower bound for the network width $m$ is specified only in terms of $n$ and $d$ under the polynomial EDR with a minimax optimal risk rate for nonparametric regression, which can be easily estimated from the training data. In contrast, under the same polynomial EDR,
all the existing works
~\cite{HuWLC21-regularization-minimax-uniform-spherical,
SuhKH22-overparameterized-gd-minimax,Li2024-edr-general-domain} require
$m \gsim \poly(n,1/\hlambda_n)$. The problem with the lower bound $\poly(n,1/\hlambda_n)$ is that one needs additional assumptions on the training data~\cite{Bartlett_Montanari_Rakhlin_2021,NguyenMM21-smallest-eigen-NTK} to find the lower bound for $\hlambda_n$, which is the minimal eigenvalue of
the empirical NTK matrix $\bK_n$, to further estimate the lower bound for $m$ using the training data.

Corollary~\ref{corollary:minimax-nonparametric-regression} also gives a competitive and smaller lower bound for the network width $m$ than some existing works which give explicit orders of the lower bound for $m$. For example, under the
assumption of uniform spherical distribution of the covariates,~\cite[Theorem 3.11]{SuhKH22-overparameterized-gd-minimax}
requires that $m/\log^3 m \gsim L^{20} n^{24}$ where $L$ is the number of layers of the DNN used in that work, and $m/\log^3 m \gsim 2^{20} n^{24}$ even with $L=2$ for the two-layer NN (\ref{eq:two-layer-nn}) used in our work. Furthermore, the proof of~\cite[Proposition 13]{Li2024-edr-general-domain} suggests that
$m \gsim n^{24} (\log m)^{12}$. Both lower bounds for $m$
in~\cite[Theorem 3.11]{SuhKH22-overparameterized-gd-minimax} and~\cite[Proposition 13]{Li2024-edr-general-domain} are  much larger than our lower bound for $m$, $n^{\frac{25\alpha}{2\alpha+1}} d^{\frac 52}$, when $n \to \infty$ and
$d$ is fixed, which is the setup considered in
\cite{Li2024-edr-general-domain}. It is worthwhile to mention that
\cite{SuhKH22-overparameterized-gd-minimax,Li2024-edr-general-domain}
use DNNs with multiple layers for nonparametric regression. Through our careful analysis, a shallow two-layer NN (\ref{eq:two-layer-nn}) exhibits the same
minimax risk rate as its deeper counterpart under the same assumptions with much smaller network width. This observation further supports the claim
in~\cite{BiettiB21} that shallow over-parameterized
neural networks with ReLU activations exhibit the same approximation
 properties as their deeper counterparts in our
 nonparametric regression setup.

\vspace{0.1in}\noindent \textbf{Training the neural network with constant learning rate
$\eta = \Theta(1)$.} It is also worthwhile to mention that our main result,
Theorem~\ref{theorem:LRC-population-NN-fixed-point}, suggests that a constant learning rate $\eta = \Theta(1)$ can be used for GD when training the two-layer NN (\ref{eq:two-layer-nn}), which could lead to better empirical optimization performance in practice. This is because any $\eta \in (0,2/u_0^2)$ can be used as the learning rate with $u_0$ being a positive constant. Some existing works in fact require an infinitesimal $\eta$. For example, when $\cX$ is bounded, \cite[Theorem 5.2]{HuWLC21-regularization-minimax-uniform-spherical}
requires the learning rates for both the squared loss and the $\ell^2$-regularization term to have the order of $o(n^{-(3d-1)/(2d-1)}) \to 0$
as $n \to \infty$. \cite[Theorem 3.11]{SuhKH22-overparameterized-gd-minimax}
uses the learning rate of $\cO(1/(n^2 L^2 m)) \to 0$ as $n \to \infty$, where $L$ is thd depth of the neural network. Furthermore, \cite[Proposition 13]{Li2024-edr-general-domain} uses the gradient flow where $\eta \to 0$ instead of the practical GD to train the neural network for an unbounded input space. We note that \cite{NitandaS21} also employs constant learning rate in SGD to train neural networks.

\vspace{0.1in}\noindent
\textbf{More discussion about this work and the relevant literature. } We herein provide more discussion about the results of this work and comparison
to the existing relevant works with sharp rates for nonparametric regression. While this paper establishes sharp rate which is distribution-free
in the bounded covariate, such rate still depends on a bounded input space and the condition that the target function $f^* \in \cH_K(\mu_0)$.
Some other existing works consider certain target function $f^*$ not belonging to the RKHS ball centered at the origin with constant or low radius, such as \cite{HaasHLS23,Bordelon2024}.
However, the target functions in \cite{HaasHLS23,Bordelon2024} escape the finite norm or low norm regime of RKHS at the cost of  restrictive conditions on the probability density function
of the covariate distribution or the training process. In particular, \cite[Theorem G.5]{HaasHLS23} requires the condition for a bounded probability density function (in its condition (D3)) of the distribution $P$, which is not required by our result. Moreover, the training process of the model in \cite{Bordelon2024} requires information about
the target function (in its Eq. (4)) and certain distribution $P$ which admits certain polynomial EDR, that is, $\lambda_j \asymp j^{-\alpha}$ with $\alpha > 1$, which happens under certain restrictive conditions
on $P$.

We also note that in this work, only the first layer of an over-parameterized two-layer neural network is trained, while the weights of the second layer are randomly initialized and then fixed
in the training process. In
existing works such as \cite{HuWLC21-regularization-minimax-uniform-spherical,SuhKH22-overparameterized-gd-minimax,Allen-ZhuLL19-generalization-dnns},
all the layers of a deep neural network with more than two layers are trained by GD or its variants. However, this work shows that only training the first layer still leads to
a sharp rate for nonparametric regression, which supports the claim in \cite{BiettiB21} that a shallow over-parameterized
neural network with ReLU activations exhibits the same approximation
properties as its deeper counterpart.

%\begin{wrapfigure}{L}{0.5\columnwidth}
%    \begin{minipage}{0.5\textwidth}
%\begin{algorithm}[H]

%\end{minipage}
 %\end{wrapfigure}
\section{Training by Gradient Descent and Preconditioned Gradient Descent}
\label{sec:training}
In the training process of our two-layer NN (\ref{eq:two-layer-nn}), only $\bW$ is optimized with $\ba$ randomly initialized to $\pm 1$ with equal probabilities and then fixed. The following quadratic loss function is minimized during the training process:
\bal \label{eq:obj-dnns}
%&L(\bW) \defeq \frac{1}{2n {\sqrt N}}
%\ltwonorm{\bK_{N,n} \pth{f(\bW,\bS)-\by}}^2.
L(\bW) \defeq \frac{1}{2n} \sum_{i=1}^{n} \pth{f(\bW,\bbx_i) - y_i}^2.
\eal%

In the $(t+1)$-th step of GD with $t \ge 0$, the weights of the neural network, $\bW_{\bS}$, are updated by one-step of GD through
\bal\label{eq:GD-two-layer-nn}
\vect{\bW_{\bS}(t+1)} - \vect{\bW_{\bS}(t)} = - \frac{\eta}{n} \bZ_{\bS}(t)
(\hat \by(t) -  \by),
\eal
where $\by_i = y_i$, $\hat \by(t) \in \RR^n$ with $\bth{\hat \by(t)}_i = f(\bW(t),\bbx_i)$. The notations with the subscript $\bS$ indicate the dependence on the training features $\bS$. We also denote $f(\bW(t),\cdot)$ as
$f_t(\cdot)$ as the neural network function with the weighting vectors $\bW(t)$ obtained right after the $t$-th step of GD.
We define $\bZ_{\bS}(t) \in \RR^{m(d+1) \times n}$ which is computed by
\bal\label{eq:GD-Z-two-layer-nn}
\bth{\bZ_{\bS}(t)}_{[(r-1)(d+1)+1:r(d+1)]i} = \frac {1}{{\sqrt m}}
\indict{\bbw_r(t)^\top \tbbx_i \ge 0}  \tbbx_i a_r, \, i \in [n], r \in [m],
\eal%
where $\bth{\bZ_{\bS}(t)}_{[(r-1)(d+1)+1:r(d+1)]i} \in \RR^{d+1}$ is a vector with elements in the $i$-th column of $\bZ_{\bS}(t)$ with indices in
$[(r-1)(d+1)+1:r(d+1)]$. We employ the following particular type of random initialization so that $\hat \by(0) = \bzero$, which has been used in earlier works such as~\cite{Chizat2019-lazy-training-differentiable-programming}. In our two-layer NN, $m$ is even, $\set{\bbw_{2r'}(0)}_{r'=1}^{m/2}$ and $\set{a_{2r'}}_{r'=1}^{m/2}$ are initialized randomly and independently according to
\bal\label{eq:random-init}
\bbw_{2r'}(0) \sim \cN(\bzero,\kappa^2 \bI_{d+1}), a_{2r'} \sim {\textup {unif}}\pth{\left\{-1,1\right\}}, \quad \forall r' \in [m/2],
\eal%
where $\cN(\bmu,\bSigma)$ denotes a Gaussian distribution with mean $\bmu$ and covariance $\bSigma$, ${\textup {unif}}\pth{\left\{-1,1\right\}}$ denotes a uniform distribution over $\set{1,-1}$, $0<\kappa \le 1$ controls the magnitude of initialization, and $\kappa \asymp 1$. We set $\bbw_{2r'-1}(0) = \bbw_{2r'}(0)$ and $a_{2r'-1} = -a_{2r}$ for all $r' \in [m/2]$. It then can be verified that $\hat \by(0) = \bzero$, that is, the initial output of the two-layer NN (\ref{eq:two-layer-nn}) is zero. Once randomly initialized, $\ba$ is fixed during the training. We use $\bW(0)$ to denote the set of all the random weighting vectors at initialization, that is, \ $\bW(0) = \set{\bbw_r(0)}_{r=1}^m$.
We run Algorithm~\ref{alg:GD} to train the two-layer NN by GD, where
$T$ is the total number of steps for GD. Early stopping is enforced in Algorithm~\ref{alg:GD} through a bounded $T$ via $T \le \hat T$.

%It is remarked that in our Algorithm~\ref{alg:GD} for GD, we do not need to
%actually run $T$ steps of GD by (\ref{eq:GD-two-layer-nn}). As the preconditioner $\bM$ is independent of the initialization $\bW(0)$
%and the weights of the neural network, we can perform Algorithm~\ref{alg:GD} to train the network by GD, and then multiply the weights difference
%${\textup{vec}(\tilde \bW(T)) - \vect{\bW(0)}}$ by $\bM$ and add it to $\bW(0)$ to obtain the weights of the neural network trained by GD. In this manner, the preconditioner is only involved in one matrix
%multiplication instead of $T$ multiplications, considerably improving the efficiency of training by GD.

\section{Main Results}
\label{sec:main-results}
We present the definition of kernel complexity in this section, and then introduce the main results for nonparametric regression of this paper.
%There is broad interest in the machine learning community
%to consider function class of finite dimension in $\cH_{K}$ or $L^2(\cX,\mu)$.
%For example,~\cite{SuY19-convergence-spectral} considers optimization of a two-layer
%neural network with target functions in
%$\cH_{K}(\mu_0) \cap L^2_{k_0}$ so that the exponent in the
%linear convergence of the training loss is well bounded.
%The assumption about the target function $f^*$ is presented below.
%Throughout this paper we assume that $m \ge \max\set{d,n}$.

\subsection{Kernel Complexity}
\label{sec:kernel-complexity}
The local kernel complexity has been studied by
\cite{bartlett2005,koltchinskii2006,Mendelson02-geometric-kernel-machine}. For the PD kernel $K$, we define the empirical kernel complexity $\hat R_K$
and the population kernel complexity $R_K$ as
\bal\label{eq:kernel-LRC-empirical}
\hat R_{K}(\eps) \defeq \sqrt{\frac 1n
\sum\limits_{i=1}^n \min\set{\hat \lambda_i,\eps^2}},
\quad
R_{K}(\eps) \defeq \sqrt{\frac 1n \sum\limits_{i=1}^{\infty} \min\set{\lambda_i,\eps^2}}.
\eal
It can be verified that both
 $\sigma_0 R_K(\eps)$ and $\sigma_0 \hat R_K(\eps)$ are
 sub-root functions~\cite{bartlett2005} in terms of $\eps^2$.
 The formal definition of sub-root functions is deferred to
 Definition~\ref{def:sub-root-function} in  the appendix.
For a given noise ratio $\sigma_0$, the critical empirical radius
$\hat \eps_n > 0$ is the smallest positive solution to the inequality
$\hat R_K(\eps) \le {\eps^2}/{\sigma_0}$, where
$\hat \eps_n^2$ is the also the fixed point of $\sigma_0 \hat R_K(\eps)$ as a function
of $\eps^2$: $\sigma_0 \hat R_K(\hat \eps_n) = \hat \eps_n^2$.
Similarly, the critical population rate $\eps_n$ is
defined to be the smallest positive solution to the inequality
$R_K(\eps) \le {\eps^2}/{\sigma}$, where
$\eps_n^2$ is the fixed point of $\sigma_0 \hat R_K(
\eps)$ as a function of
$\eps^2$: $\sigma_0 R_K(\eps_n) = \eps_n^2$.
%In this paper we consider the case that $n\eps_n^2 \to \infty$ as $n \to \infty$, which is also used in standard analysis of nonparametric regression with minimax rates by kernel regression ~\cite{RaskuttiWY14-early-stopping-kernel-regression}.
%which covers most popular positive semi-definite kernels including the dot-product
% kernel (\ref{eq:kernel-two-layer}) and a broad range of data distributions %\cite{Yang2017-fast-sketch-kernel}, and it
%\bal\label{eq:kernel-LRC-empirical-inequality}
%\hat R_K(\eps) \le \frac{\eps^2}{\sigma}.
%\eal

Let $\eta_t \defeq \eta t$ for all $t > 0$, we then define the
 stopping time $\hat T$  as
\bal\label{eq:stopping-time-hatT}
\hat T
 \defeq \min\set{t \colon \hat R_{K}(\sqrt{1/\eta_t}) > (\sigma \eta_t)^{-1}}-1.
\eal
The stopping time in fact is the upper bound for the number of steps $T$ for Algorithm~\ref{alg:GD} as to be shown in Section~\ref{sec:main-results-details}, which in turn enforces the early stopping mechanism.

\subsection{Main Results}
\label{sec:main-results-details}
\begin{algorithm}[!htbp]
\renewcommand{\algorithmicrequire}{\textbf{Input:}}
\renewcommand\algorithmicensure {\textbf{Output:} }
\caption{Training the Two-Layer NN by GD}
\label{alg:GD}
\begin{algorithmic}[1]
\State $\bW(T) \leftarrow$ Training-by-GD($T,\bW(0)$)
\State \textbf{\bf input: } $T,\bW(0),\eta$
\State \textbf{\bf for } $t=1,\ldots,T$ \,\,\textbf{\bf do }
\State \quad Perform the $t$-th step of GD by
(\ref{eq:GD-two-layer-nn})
\State \textbf{\bf end for }
\end{algorithmic}
\end{algorithm}

The main results of this paper are presented in this section.

\begin{theorem}\label{theorem:LRC-population-NN-fixed-point}
Suppose that $\cX$ is bounded and $n \gsim \max\set{1/\lambda_1,\sigma_0^2 u_0^2/2}$.
 Let $c_T, c_t \in (0,1]$ be arbitrary positive constants, and
$c_T \hat T \le T \le \hat T$. Suppose
%$m$ satisfies the conditions in Theorem~\ref{theorem:good-random-initialization} as well as
%$m \ge \max\set{n^{1/d}, \Theta(T^2)}$, $m/\log^2 m \ge d$,
$m$ satisfies
\bal\label{eq:m-cond-LRC-population-NN-fixed-point}
m \gsim \frac{d^{\frac 52}}{\eps_n^{25}},
\eal
and the neural network
$f(\bW(t),\cdot)$ is
trained by GD using Algorithm~\ref{alg:GD} with the learning rate $\eta = \Theta(1)$ such that $\eta=\Theta(1) \in (0,2/u_0^2)$ and $T \le \hat T$.
Then for every $t \in [c_t T \colon T]$, with probability at least
$1 -  \exp\pth{-\Theta(n)}
- 7\exp\pth{-\Theta(n \eps_n^2)} -2/n$ over the random noise $\bw$, the random training features $\bS$ and
the random initialization $\bW(0)$,
the stopping time satisfies
$\hat T \asymp  \eps_n^{-2}$, and $f(\bW(t),\cdot) = f_t$ satisfies
\bal\label{eq:LRC-population-NN-fixed-point}
&\Expect{P}{(f_t-f^*)^2}
\lsim \eps^2_n.
\eal
%where the bound $\lsim$ hides an absolute positive constant depending on $\mu_0$ and $\sigma$.
\end{theorem}

When the polynomial EDR holds, we can apply Theorem~\ref{theorem:LRC-population-NN-fixed-point} to obtain the following corollary.
\begin{corollary}
[Applying Theorem~\ref{theorem:LRC-population-NN-fixed-point} to the
special case of polynomial EDR]
\label{corollary:minimax-nonparametric-regression}
Under the conditions of
Theorem~\ref{theorem:LRC-population-NN-fixed-point},
suppose that
$\lambda_j \asymp j^{-2\alpha}$ for $j \ge 1$ and $\alpha > 1/2$.
and $m$ satisfies
\bal\label{eq:m-N-cond-LRC-population-NN-concrete}
m \gsim
n^{\frac{25\alpha}{2\alpha+1}} d^{\frac 52}.
\eal
Let the neural network
$f(\bW(t),\cdot)$ be
trained by GD using Algorithm~\ref{alg:GD} with the learning rate
$\eta = \Theta(1)$ such that $\eta \in (0,2/u_0^2)$ and $T \le \hat T$.
Then for every $t \in [c_t T \colon T]$, with probability at least
$1 -  \exp\pth{-\Theta(n)}
- 7\exp\pth{-\Theta(n^{1/(2\alpha+1)})}  - 2/n$
over the random noise $\bw$, the random training features $\bS$ and
the random initialization $\bW(0)$, the stopping time satisfies
$\hat T \asymp n^{\frac{2\alpha}{2\alpha+1}} $, and
\bal\label{eq:minimax-nonparametric-regression}
\Expect{P}{(f_t-f^*)^2} \lsim
\pth{\frac{1}{n}}^{\frac{2\alpha}{2\alpha+1}}.
\eal
\end{corollary}

% \begin{corollary}
% [Applying Theorem~\ref{theorem:LRC-population-NN-fixed-point} to the
% special case of polynomial EDR]
% \label{corollary:minimax-nonparametric-regression}
% Suppose that $n \gsim \max\set{1/\lambda_1,\sigma_0^2 u_0^2/2}$, and $\lambda_j \asymp j^{-2\alpha}$ for $j \ge 1$ and $\alpha > 1/2$.
% Let $c_T, c_t \in (0,1]$ be positive constants, and
% $c_T \hat T \le T \le \hat T$. Suppose $m$ satisfies
% \bal\label{eq:m-N-cond-LRC-population-NN-concrete}
% m \gsim
% n^{\frac{16\alpha}{2\alpha+1}} d^2,
% \eal
% and the neural network
% $f(\bW(t),\cdot)$ is
% trained by GD using Algorithm~\ref{alg:GD} with the learning rate
% $\eta \in (0,2/u_0^2)$ and $T \le \hat T$.
% Then for every $t \in [c_t T \colon T]$, with probability at least
% $1 -  \exp\pth{-\Theta(n)}
% - 7\exp\pth{-\Theta(n \eps_n^2)}  - 2/n$
% over the random noise $\bw$, the random training features $\bS$ and
% the random initialization $\bW(0)$, the stopping time satisfies
% $\hat T \asymp n^{\frac{2\alpha}{2\alpha+1}} $, and
% \bal\label{eq:minimax-nonparametric-regression}
% \Expect{P}{(f_t-f^*)^2} \lsim
% \pth{\frac{1}{n}}^{\frac{2\alpha}{2\alpha+1}}.
% \eal
% \end{corollary}

The significance of
Theorem~\ref{theorem:LRC-population-NN-fixed-point} and Corollary~\ref{corollary:minimax-nonparametric-regression} and comparison to existing works are presented in Section~\ref{sec:summary-main-results}. To the best of our knowledge,
Theorem~\ref{theorem:LRC-population-NN-fixed-point} is the first theoretical result which proves that over-parameterized
neural network trained by gradient descent with early stopping achieves the sharp rate of $\cO(\eps_n^2)$
\textit{without distributional assumption on the covariate} as long as the input space $\cX$ is bounded. Moreover, we present simulation results Section~\ref{sec:simulation} of the appendix, where the two-layer NN in (\ref{eq:two-layer-nn}) is trained by GD using Algoirthm~\ref{alg:GD}  and the early-stopping time theoretically predicted by
Corollary~\ref{corollary:minimax-nonparametric-regression} is studied.

\section{Roadmap of Proofs}
\label{sec:proof-roadmap}
We present the roadmap of our theoretical results which lead to the main result, Theorem~\ref{theorem:LRC-population-NN-fixed-point} in Section~\ref{sec:main-results}. We first present in Section~\ref{sec:uniform-convergence-ntk-more} our results about the uniform convergence to the NTK
(\ref{eq:kernel-two-layer}) and more, which are crucial in the analysis of training dynamics by GD. We then introduce the basic definitions in
Section~\ref{sec:proofs-definitions}, and the detailed roadmap and key technical results with our novel proof strategy for this work in Section~\ref{sec:detailed-roadmap-key-results} which lead to the main result in
Theorem~\ref{theorem:LRC-population-NN-fixed-point}. The proofs of
Theorem~\ref{theorem:LRC-population-NN-fixed-point} and Corollary~\ref{corollary:minimax-nonparametric-regression} are presented in
Section~\ref{sec:proofs-main-results}, and Section~\ref{sec:proofs-key-results} presents the proofs of the key results in
Section~\ref{sec:detailed-roadmap-key-results}.

\subsection{Uniform Convergence to the NTK (\ref{eq:kernel-two-layer}) and More}
\label{sec:uniform-convergence-ntk-more}
We define the following functions with $\bW = \set{\bw_r}_{r=1}^m$:
\bal
h(\bw,\bu,\bv) &\defeq \tbu^{\top} \tbv \indict{\bw^{\top} \tbu \ge 0} \indict{\bw^{\top} \tbv \ge 0}, \quad &\hat h(\bW,\bu,\bv) &\defeq \frac {1}{m} \sum\limits_{r=1}^m h(\bbw_r,\bu,\bv), \label{eq:h-hat-h}  \\
v_R(\bw,\bu) &\defeq \indict{\abth{\bw^{\top}\tbu} \le R}, \quad &\hat v_R(\bW,\bu) &\defeq  \frac 1m \sum\limits_{r=1}^m v_R(\bbw_r,\bu), \label{eq:v-hat-v}
\eal%
where $\bu,\bv \in \RR^d$ and $\tbu = \bth{\bu^{\top},1}^{\top}$, $\tbv = \bth{\bv^{\top},1}^{\top}$.
Then we have the following theorem stating the uniform convergence of $\hat h(\bW(0),\cdot,\cdot)$ to $K(\cdot,\cdot)$ and uniform convergence
of $\hat v_R(\bW(0),\cdot)$ to $\frac{2R}{\sqrt {2\pi} \kappa}$ for a
positive number $R\lsim \eta u_0 T /{\sqrt m}$, and $R$ is formally defined in (\ref{eq:def-R}).
It is remarked that while existing works such as \cite{Li2024-edr-general-domain} also has uniform convergence results for over-parameterized neural network,
our result does not depend on the H\"older continuity of the NTK.
\begin{theorem}\label{theorem:good-random-initialization}
The following results hold with
$\eta \lsim 1$, $m \gsim \max\set{n^{2/(d+1)},\Theta(T^{\frac 53})}$, and $m/\log m \ge d$.
\begin{itemize}[leftmargin=.26in]
\item[(1)] With probability at least $1-1/n$ over the random initialization $\bW(0) = \set{\bbw_r(0)}_{r=1}^m$,
\bal\label{eq:good-initialization-sup-hat-h}
\sup_{\bu \in \cX,\bv \in \cX} \abth{ K(\bu,\bv) - \hat h(\bW(0),\bu,\bv) } \le  u_0^2 C_1(m/2,d,1/n) \lsim u_0^2 \sqrt{\frac{d \log m}{m}}.
\eal%
\item[(2)] With probability at least $1-1/n$ over the random initialization $\bW(0) = \set{\bbw_r(0)}_{r=1}^m$,
\bal\label{eq:good-initialization-sup-hat-V_R}
&\sup_{\bu \in \cX}\hat v_R(\bW(0),\bu)
\le \frac{2R}{\sqrt {2\pi} \kappa} + C_2(m/2,d,1/n) \lsim \sqrt{d} m^{-\frac 15} T^{\frac 12},
\eal%
 \end{itemize}
 where $C_1(m/2,d,1/n),C_2(m/2,d,1/n)$ are two positive numbers depending on $(m,d,n)$, with their formal definitions deferred to
(\ref{eq:C1}) and (\ref{eq:C2}) in Section~\ref{sec:proofs-key-results}.
\end{theorem}
\begin{proof}%[\textup{Proof of Theorem~\ref{theorem:good-random-initialization}}]
This theorem follows from Theorem~\ref{theorem:sup-hat-g}
and Theorem~\ref{theorem:V_R} in Section~\ref{sec:proofs-key-results}.
Note that
$K(\bu,\bv) = \ltwonorm{\tbu} \ltwonorm{\tbv} \tK(\tbu/\ltwonorm{\tbu},\tbv/\ltwonorm{\tbv})$, $h(\bw,\bu,\bv) = \ltwonorm{\tbu} \ltwonorm{\tbv}
\tilde h(\bw,\tbu/\ltwonorm{\tbu},\tbv/\ltwonorm{\tbv})$ with $\tK$ and $\tilde h$ defined in Theorem~\ref{theorem:sup-hat-g}, and
\bals
\hat h(\bW,\bu,\bv) = \frac {1}{m} \sum\limits_{r=1}^m h(\bbw_r,\bu,\bv) = \frac {1}{m/2} \sum\limits_{r'=1}^{m/2} h(\bbw_{2r}(0),\bu,\bv),
\eals
then part (1) directly follows from Theorem~\ref{theorem:sup-hat-g}.  Furthermore,
since $\ltwonorm{\tbu} \ge 1$, we have
\bals
\hat v_R(\bW(0),\bu) =
\sum\limits_{r=1}^m \indict{\abth{\bbw_r(0)^{\top}\tbu} \le R}/m
\le \sum\limits_{r=1}^m \indict{\abth{\bbw_r(0)^{\top}\tbu/\ltwonorm{\tbu}} \le R/\ltwonorm{\tbu}}/m
\le  \tilde v_R({\bW(0)},\tbu/\ltwonorm{\tbu}),
\eals
where $\tilde v_R$ is defined in Theorem~\ref{theorem:V_R},
so that part (2) directly follows from Theorem~\ref{theorem:V_R} by noting that $R \lsim m^{-\frac 15} T^{\frac 12}$ when $m \ge \Theta(T^{\frac 53})$.
\end{proof}
%\vspace{0.1in}\noindent \textbf{Comments for the Uniform Convergence Rate in (\ref{eq:good-initialization-sup-hat-h}).}
%Theorem~\ref{theorem:sup-hat-g}
%and Theorem~\ref{theorem:V_R} uses a carefully designed and novel
%proof strategy
%Which contribute to the lower bound for the network width $m$ (see detailed discussion in Section~\ref{sec:summary-main-results}) compared to several existing works which provide explicit order of the network width for nonparametric regression. In particular, when $d$ is a fixed constant as
%the setup in~\cite{Li2024-edr-general-domain},
%the uniform convergence rate in
%(\ref{eq:good-initialization-sup-hat-h}) is $\cO(\sqrt{{\log m}/{m}})$, while the uniform convergence rate given by
%\cite[Lemma 29]{Li2024-edr-general-domain} is
%$\cO(m^{-1/4})$.
We define
\bal\label{eq:set-of-good-random-initialization}
\cW_0 \defeq \set{\bW(0) \colon (\ref{eq:good-initialization-sup-hat-h}) ,(\ref{eq:good-initialization-sup-hat-V_R}) \textup { hold}}
\eal%
as the set of all the good random initializations which satisfy (\ref{eq:good-initialization-sup-hat-h})  and (\ref{eq:good-initialization-sup-hat-V_R}) in Theorem~\ref{theorem:good-random-initialization}.
Theorem~\ref{theorem:good-random-initialization} shows that we have good random initialization with high probability, that is, $\Prob{\bW(0)
\in \cW_0} \ge 1-2/n$. When $\bW(0) \in \cW_0 $, the uniform convergence results, (\ref{eq:good-initialization-sup-hat-h})
and (\ref{eq:good-initialization-sup-hat-V_R}), hold with high probability, which is important for the analysis of the training dynamics of the two-layer NN
(\ref{eq:two-layer-nn}) by GD.

\subsection{Basic Definitions}
\label{sec:proofs-definitions}
We introduce the following definitions for our analysis.
We define
\bal\label{eq:ut}
\bu(t) \defeq \hat \by(t) - \by
\eal
as the difference between the network output
$\hat \by(t)$ and the training response vector $\by$ right after the $t$-th step of GD.
Let $\tau \le 1$ be a positive number. For $t \ge 0$ and $T \ge 1$ we define the following quantities:
%$c_{\bu} \defeq {\mu_0}/{\min\set{2,\sqrt {2e \eta }}} + \sigma + \tau + 1$,
$c_{\bu} \defeq \max\set{{\mu_0}/{\sqrt {2e \eta }}, \mu_0u_0/{\sqrt 2}} + \sigma_0 + \tau + 1$,
\bal\label{eq:def-R}
&R \defeq \frac{\eta c_{\bu} u_0 T}{\sqrt m},
\eal
\bal\label{eq:cV_S}
%&\cV_t \defeq \set{\bv \in \RR^n \colon \bv \in \Span(\bU), \bv = -\pth{\bI-\frac{\eta}{n} \bK }^t f^*(\bS)}.
\cV_t \defeq \set{\bv \in \RR^n \colon \bv = -\pth{\bI_n- \eta \bK_n }^t f^*(\bS)},
\eal
\bal\label{eq:cE_S}
%\cE_{t,\tau} \defeq &\left\{\be \colon \be = \bbe_1 + \bbe_2, \bbe_1,\bbe_2 \in \RR^n, \right. \nonumber \\
%&\left. \bbe_1 = -\pth{\bI-\eta\bK_n}^t \bw,
%\ltwonorm{\bbe_2} \le \pth{c_{\bu}-c_{\bu} + \tau} {\sqrt n} \right\}.
\cE_{t,\tau} \defeq &\set{\be \colon \be = \bbe_1 + \bbe_2 \in \RR^n, \bbe_1 = -\pth{\bI_n-\eta\bK_n}^t \bw,
\ltwonorm{\bbe_2} \le {\sqrt n} \tau }.
\eal
In particular, Theorem~\ref{theorem:empirical-loss-convergence} in the next subsection shows that with high probability over the random noise $\bw$, the distance of every weighting vector $\bw_r(t)$ to its initialization $\bw_r(0)$  is bounded by $R$. In addition, $\bu(t)$ can be composed into two vectors,
$\bu(t) = \bv(t) + \be(t)$ such that $\bv(t) \in \cV_t$
and $\be(t) \in \cE_{t,\tau}$.

We then define the set of the neural network weights during the training by GD using Algorithm~\ref{alg:GD} as follows:
%\vphantom{ \int_1^2 }
\bal\label{eq:weights-nn-on-good-training}
&\cW(\bS,\bW(0),T) \defeq \left\{\bW \colon \exists t \in [T] {\textup{ s.t. }}\vect{\bW} = \vect{\bW(0)} - \sum_{t'=0}^{t-1} \frac{\eta}{n}  \bZ_{\bS}(t') \bu(t'), \right. \nonumber \\
& \left. \bu(t') \in \RR^{n}, \bu(t') = \bv(t') + \be(t'),
\bv(t') \in \cV_{t'}, \be(t') \in \cE_{t',\tau}, {\textup { for all } } t' \in [0,t-1] \vphantom{\frac12}  \right\}.
\eal%

We will also show by Theorem~\ref{theorem:empirical-loss-convergence} that with high probability over $\bw$, $\cW(\bS,\bW(0),T)$ is the set of the weights of the two-layer NN  (\ref{eq:two-layer-nn}) trained by GD on the training data $\bS$ with the random initialization $\bW(0)$ and the number of steps of GD not greater than $T$.
%By (\ref{eq:weights-nn-on-good-training}), the vectorization of any weight in $\cW(\bS,\cW_0,T)$ is expressed by $\vect{\bW(0)} - \sum_{t'=0}^{t-1} \frac{\eta}{n}  \bZ_{\bS}(t') \bu(t')$.
The set of the functions represented by the neural network with weights in $\cW(\bS,\bW(0),T)$ is then defined as
\bal\label{eq:random-function-class}
\cFnn(\bS,\bW(0),T) \defeq \set{f_t = f(\bW(t),\cdot) \colon \exists \, t \in [T], \bW(t) \in \cW(\bS,\bW(0),T)}.
\eal%
We also define the function class $\cF(B,w)$ for any $B,w > 0$ as
\bal
\cF(B,w) &\defeq \set{f \colon f = h+e, h \in \cH_{K}(B),
\supnorm{e} \le w}. \label{eq:def-cF-ext-general}
\eal
We will show by
Theorem~\ref{theorem:bounded-NN-class} in the next subsection
that with high probability over $\bw$,
$\cFnn(\bS,\bW(0),T)$ is a subset of $\cF(B_h,w)$, where a smaller
$w$ requires a larger network width $m$, and
\bal
B_h &\defeq \mu_0 +1+ {\sqrt 2}. \label{eq:B_h}
\eal

The Rademacher complexity of a function class and its empirical version are defined below.
\begin{definition}\label{def:RC}
Let $\{\sigma_i\}_{i=1}^n$ be $n$ i.i.d. random variables such that $\Pr[\sigma_i = 1] = \Pr[\sigma_i = -1] = \frac{1}{2}$. The Rademacher complexity of a function class $\cF$ is defined as
\bal\label{eq:RC}
&\cfrakR(\cF) = \Expect{\set{\bbx_i}_{i=1}^n, \set{\sigma_i}_{i=1}^n}{\sup_{f \in \cF} {\frac{1}{n} \sum\limits_{i=1}^n {\sigma_i}{f(\bbx_i)}} }.
\eal%
The empirical Rademacher complexity is defined as
\bal\label{eq:empirical-RC}
&\hat \cfrakR(\cF) = \Expect{\set{\sigma_i}_{i=1}^n} { \sup_{f \in \cF} {\frac{1}{n} \sum\limits_{i=1}^n {\sigma_i}{f(\bbx_i)}} },
\eal%
For simplicity of notations, Rademacher complexity and empirical Rademacher complexity are also denoted by ${\E}\left[\sup_{f \in \cF} {\frac{1}{n} \sum\limits_{i=1}^n {\sigma_i}{f(\bbx_i)}} \right]$ and ${\E}_{\sigma}\left[\sup_{f \in \cF} {\frac{1}{n} \sum\limits_{i=1}^n {\sigma_i}{f(\bbx_i)}} \right]$ respectively. %We also denote the Rademacher complexity by $\Expect{\bS,\set{\sigma_i}_{i=1}^n}{}$ to emphasize the variables with respect to which the expectation is computed.
\end{definition}

For data $\set{\bbx}_{i=1}^n$ and a function class $\cF$, we define the notation $R_n \cF$ by $R_n \cF \coloneqq \sup_{f \in \cF} \frac{1}{n} \sum\limits_{i=1}^n \sigma_i f(\bbx_i)$.

\subsection{Detailed Roadmap and Key Results}
\label{sec:detailed-roadmap-key-results}

Because our main result,
Theorem~\ref{theorem:LRC-population-NN-fixed-point}, is proved by
Theorem~\ref{theorem:LRC-population-NN-eigenvalue} and Theorem~\ref{theorem:empirical-loss-bound} deferred to Section~\ref{sec:detailed-key-results},
we illustrate in Figure~\ref{fig:proof-roadmap} the roadmap containing the intermediate key theoretical results which lead
to
Theorem~\ref{theorem:LRC-population-NN-fixed-point}.

\begin{figure}[!htbp]
\begin{center}
\includegraphics[width=0.92\textwidth]{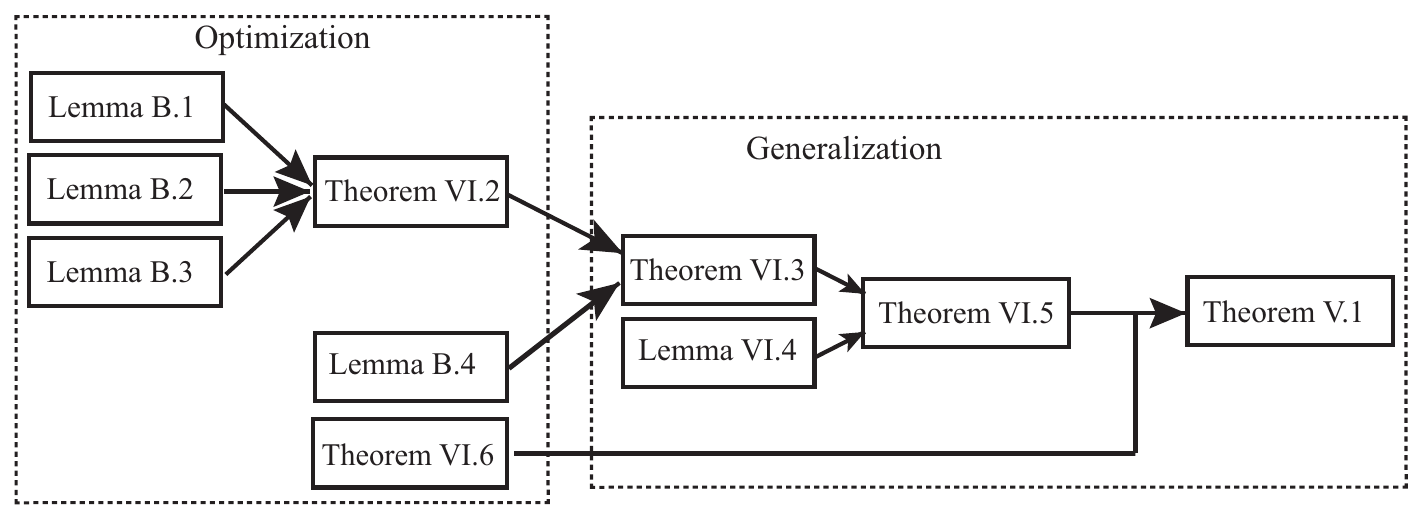}
\end{center}
\caption{Roadmap of key results leading to the main result, Theorem~\ref{theorem:LRC-population-NN-fixed-point}. The uniform convergence results in
Theorem~\ref{theorem:good-random-initialization} are used in all the optimization results and Theorem~\ref{theorem:bounded-NN-class}.}
\label{fig:proof-roadmap}
\end{figure}

%\vspace{0.1in}\noindent \textbf
\subsubsection{Summary of the  approaches and key technical results in the proofs.}
Our results are built upon two significant technical results  of independent interest. First, uniform convergence to the NTK (\ref{eq:kernel-two-layer}) is established during the training process by GD, so that we can have a nice decomposition of the neural network function at any step of GD  into a function in the RKHS associated with the NTK (\ref{eq:kernel-two-layer})  and an error function with a small $L^{\infty}$-norm.
In particular, with the uniform convergence by
Theorem~\ref{theorem:good-random-initialization} and the optimization results in Theorem~\ref{theorem:empirical-loss-convergence} and
Lemma~\ref{lemma:bounded-Linfty-vt-sum-et},
Theorem~\ref{theorem:bounded-NN-class} shows that with high probability, the neural network function $f(\bW(t),\cdot)$ right after the $t$-th step of GD can be decomposed into
two functions by $f(\bW(t),\cdot) = f_t = h  + e$, where $h \in \cH_{K}$ is a function in the RKHS associated with $K$ with a bounded $\cH_{K}$-norm.
The error function $e$ has a small $L^{\infty}$-norm, that is, $\supnorm{e} \le w$ with $w$ being a small number controlled by the network width
$m$, and larger $m$ leads to smaller $w$. Second,
local Rademacher complexity is employed
to tightly bound the risk of nonparametric regression in
Theorem~\ref{theorem:LRC-population-NN-eigenvalue}, which is based on
 the Rademacher complexity of a localized subset of the function class $\cF(B_h,w)$ in
Lemma~\ref{lemma:LRC-population-NN}.

We  use
Theorem~\ref{theorem:bounded-NN-class} and Lemma~\ref{lemma:LRC-population-NN}
to derive Theorem~\ref{theorem:LRC-population-NN-eigenvalue}. Theorem~\ref{theorem:LRC-population-NN-eigenvalue}
states that if $m$ is sufficiently large, then for every $t \in [T]$, with high probability, the risk
of $f_t$  satisfies
\bal\label{eq:theorem10-demo}
\Expect{P}{(f_t-f^*)^2} - 2 \Expect{P_n}{(f_t-f^*)^2}
\lsim \eps_n^2 + w.
\eal
We then obtain Theorem~\ref{theorem:LRC-population-NN-fixed-point} using (\ref{eq:theorem10-demo}) where $w$ is set to $\eps_n^2$, with the empirical loss $\Expect{P_n}{(f_t-f^*)^2}$ bounded by
$\Theta(1/(\eta t)) \asymp \eps_n^2$ with high probability by Theorem~\ref{theorem:empirical-loss-bound}.

%The fundamental reason that
%the existing works~\cite{HuWLC21-regularization-minimax-uniform-spherical,
%SuhKH22-overparameterized-gd-minimax,Li2024-edr-general-domain} require distributional assumptions is that
%$\norm{\hat f - f^*}{L^2}$ is bounded by a minimax rate under certain distributional assumptions, such as the distributions for which the polynomial EDR of the NTK holds.
\vspace{0.1in}\noindent \textbf{Novel proof strategy of this work.} We remark that the proof strategy of our main result,
Theorem~\ref{theorem:LRC-population-NN-fixed-point}, summarized above is significantly different from the existing works in training over-parameterized neural networks for nonparametric regression with minimax rates
\cite{HuWLC21-regularization-minimax-uniform-spherical,
SuhKH22-overparameterized-gd-minimax,Li2024-edr-general-domain}. In particular, the common proof strategy in these works uses the decomposition
$f_t-f^* = (f_t - \hat f_t^{\textup{(NTK)}}) + (\hat f_t^{\textup{(NTK)}} - f^*)$ and then shows that
both $\norm{f_t - \hat f_t^{\textup{(NTK)}}}{L^2}$ and $\norm{\hat f_t^{\textup{(NTK)}}- f^*}{L^2}$ are bounded by certain sharp rate (minimax optimal or nearly minimax optimal rate), where
$\hat f_t^{\textup{(NTK)}}$ is the kernel regressor obtained by either kernel
ridge regression~\cite{HuWLC21-regularization-minimax-uniform-spherical,
SuhKH22-overparameterized-gd-minimax} or GD with
early stopping~\cite{Li2024-edr-general-domain}. The remark after Theorem~\ref{theorem:bounded-NN-class} details a formulation of $\hat f_t^{\textup{(NTK)}}$.
$\norm{\hat f_t^{\textup{(NTK)}}- f^*}{L^2}$ is bounded by such sharp rate under certain distributional assumptions in the covariate, and this is one reason for the distributional assumptions about the covariate in the existing works such as \cite{HuWLC21-regularization-minimax-uniform-spherical,SuhKH22-overparameterized-gd-minimax,Li2024-edr-general-domain}.
In a strong contrast, our analysis does not rely on such decomposition of $f_t-f^*$. Instead of approximating $f_t$ by $\hat f_t^{\textup{(NTK)}}$, we have a new decomposition of $f_t$ by $f_t = h_t  + e_t$ where
$f_t$ is approximated by $h_t$ with $e_t$ being the approximation error. As suggested by the remark after Theorem~\ref{theorem:bounded-NN-class},
we have $h_t = \hat f_t^{\textup{(NTK)}}  + \tilde e_2(\cdot,t)$ so that
$f_t = \hat f_t^{\textup{(NTK)}}  + \tilde e_2(\cdot,t)+ e_t$. Our analysis only requires the network width
$m$ to be suitably large so that the $\cH_K$-norm of $\tilde e_2(\cdot,t)$
is bounded by a positive constant and $\supnorm{e_t } \le w$ with $w$ set to the sharp rate, while the common proof strategy in\cite{HuWLC21-regularization-minimax-uniform-spherical,
SuhKH22-overparameterized-gd-minimax,Li2024-edr-general-domain} needs $m$ to be sufficiently large so that both
$\supnorm{\tilde e_2(\cdot,t)}$ and $\supnorm{e_t}$ are bounded by the sharp rate such as $\cO(n^{-(d+1)/(2d+1)})$ and then $\norm{f_t - \hat f_t^{\textup{(NTK)}}}{L^2}$ is bounded by the same sharp rate. Detailed in Section~\ref{sec:summary-main-results}, such novel proof strategy leads to a smaller lower bound for $m$ in our main result compared to some existing works.
Importantly, the sharp rate in our Theorem~\ref{theorem:LRC-population-NN-fixed-point} is distribution-free in the bounded covariate fundamentally because local Rademacher complexity \cite{bartlett2005} based analysis employed in this work does not need distributional assumptions about the bounded covariate.

It is also worthwhile to mention that our proof strategy is also significantly different from the existing works in the classical kernel regression such as
\cite{RaskuttiWY14-early-stopping-kernel-regression}. It is remarked that \cite{RaskuttiWY14-early-stopping-kernel-regression} uses the off-the-shelf local Rademacher complexity based generalization bound \cite{bartlett2005} to derive the risk of the order $\mathcal O(\epsilon_n^2)$ for the usual kernel regressor (e.g., in \cite[Lemma 10]{RaskuttiWY14-early-stopping-kernel-regression} ). For the NN function $f_t$ that uniformly converges to the usual kernel regressor, this work establishes a new theoretical framework with a novel decomposition of $f_t$ detailed above, and then the tight risk bound is derived by a tight bound for the Rademacher complexity of a localized subsest of the function class, $\cF(B_h,w)$ which comprises all the two-layer NN functions trained by GD, in Lemma~\ref{lemma:LRC-population-NN} and Theorem~\ref{theorem:LRC-population-NN-eigenvalue}.
%\vspace{0.1in}\noindent \textbf
\subsubsection{Key Technical Results}
\label{sec:detailed-key-results}
We present our key technical results regarding optimization and generalization of the two-layer NN (\ref{eq:two-layer-nn}) trained by GD with early
stopping.

Theorem~\ref{theorem:empirical-loss-convergence} is our main result about the optimization of the network (\ref{eq:two-layer-nn}), which
states that with high probability over the random noise $\bw$, the weights of the network $\bW(t)$ obtained right after the $t$-th step of GD
using Algorithm~\ref{alg:GD} belongs to
$\cW(\bS,\bW(0),T)$. Furthermore, every weighing vector $\bw_r$ has bounded distance to the initialization $\bw_r(0)$.
The proof of Theorem~\ref{theorem:empirical-loss-convergence} is
based on
Lemma~\ref{lemma:yt-y-bound},
Lemma~\ref{lemma:empirical-loss-convergence-contraction},
and Lemma~\ref{lemma:weight-vector-movement} deferred to
Section~\ref{sec:lemmas-main-results} of the appendix.

\begin{theorem}\label{theorem:empirical-loss-convergence}
Suppose
\bal\label{eq:m-cond-empirical-loss-convergence}
m \gsim  {(\eta c_{\bu})^5 u_0^{10} T^{\frac {15}{2}} d^{\frac 52}}/{\tau^5},
\eal
the neural network $f(\bW(t),\cdot)$ trained by GD
using Algorithm~\ref{alg:GD} with the learning rate $\eta \in (0,2/u_0^2)$, the random initialization $\bW(0) \in \cW_0$. Then with probability at least
$1 -  \exp\pth{-\Theta(n)}$ over the random noise $\bw$,
$\bW(t) \in \cW(\bS,\bW(0),T)$ for every $t \in [T]$.
Moreover, for every $t \in [0,T]$, $\bu(t) = \bv(t) + \be(t)$ where
$\bu(t) = \hat \by(t) -  \by$, $\bv(t) \in \cV_t$,
$\be(t) \in \cE_{t,\tau}$,  $\ltwonorm{\bu(t)} \le c_{\bu} \sqrt n$, and
$\ltwonorm{\bbw_r(t) - \bbw_r(0)} \le R$.
\end{theorem}

The following theorem, Theorem~\ref{theorem:bounded-NN-class},
states that with high probability over $\bw$,
$\cFnn(\bS,\bW(0),T) \subseteq \cF(B_h,w)$, with the early stopping mechanism such that
$T \le \hat T$.
\begin{theorem}\label{theorem:bounded-NN-class}
Suppose $w  \in (0,1)$,
\bal\label{eq:m-cond-bounded-NN-class}
m \gsim
\max\set{ {T^{\frac {15}{2}} d^{\frac 52}}/{w^5}, T^{\frac{25}{2}} d^{\frac 52}},
\eal
and the neural network
$f_t = f(\bW(t),\cdot) $ is trained by GD using Algorithm~\ref{alg:GD} with the learning rate $\eta \in (0,2/u_0^2)$, the random initialization $\bW(0) \in \cW_0$. Then for every $t \in [T]$ with $T \le \hat T$, with probability at least
$1 -  \exp\pth{-\Theta(n)}
- \exp\pth{-\Theta(n \hat\eps_n^2)}$ over the random noise $\bw$,
$f_t\in \cFnn(\bS,\bW(0),T)$, and $f_t$ has the following decomposition on $\cX$:
\bal\label{eq:nn-function-class-decomposition}
f_t = h_t + e_t \in \cF(B_h,w),
\eal
where $h_t \in \cH_{K}(B_h)$ with $B_h$ defined in (\ref{eq:B_h}),
$e_t  \in L^{\infty}$ with $\supnorm{e_t } \le w$.
%In addition,
%\bal\label{eq:bounded-Linfty-function-class}
%\supnorm{f_t}  \le \frac{B_h}{\sqrt 2} + w.
%\eal
\end{theorem}
\begin{remark*}%\label{remark::bounded-NN-class}
We consider the kernel regression problem with the training loss
$L(\balpha) = 1/2 \cdot \ltwonorm{\bK_n \balpha  - \by}^2$. Letting
$\bbeta  = \bK_n^{1/2} \balpha$ and then performing GD on $\bbeta$ with this training loss and the learning rate $\eta$,
 it  can be verified that the kernel regressor right after the
$t$-th step of GD is
\bal\label{eq:hat-f-kernel-regressor}
\hat f_t^{\textup{(NTK)}}  = \frac {\eta}n \sum\limits_{t'=0}^{t-1}
\sum\limits_{i=1}^{n} K(\cdot,\bbx_i) \balpha^{(t')}_i,
\eal
where $\balpha^{(t')} =  \pth{\bI_n - \eta\bK_n}^{t'} \by$. Following from the proof of Lemma~\ref{lemma:weight-vector-movement} and
Theorem~\ref{theorem:bounded-NN-class}, under the conditions of
Theorem~\ref{theorem:bounded-NN-class} we have
\bals
%\supnorm{h - {\hat f}_t} \le \frac{1}{\sqrt 2}.
h_t = \hat f_t^{\textup{(NTK)}}  + \tilde e_2(\cdot,t),
\eals
where
$\tilde e_2(\cdot,t) =\frac{\eta}{n}
\sum_{t'=0}^{t-1}\sum_{j=1}^n  K(\cdot, \bbx_j) \bth{\bbe_2(t')}_j$
and $\bbe_2(t')$ appears in the definition of
$\cE_{t,\tau} $ in (\ref{eq:cE_S}).
It is remarked that in our analysis, we approximate $f_t$ by $h_t \in \cH_{K}(B_h)$ with a small approximation error $w$, and we do not need to approximate $f_t$ by the kernel regressor $\hat f_t^{\textup{(NTK)}}$ with a sufficiently small approximation error which is the common strategy used in existing works~\cite{HuWLC21-regularization-minimax-uniform-spherical,
SuhKH22-overparameterized-gd-minimax,Li2024-edr-general-domain}. In fact, our analysis only requires $m$ is suitably large so that the
$\cH_K$-norm of $\tilde e_2(\cdot,t) = h_t - \hat f_t^{\textup{(NTK)}} $ is bounded by a positive constant rather than an infinitesimal number as
$m \to \infty$, that is,
$\norm{\tilde e_2(\cdot,t)}{\cH_K} \le 1$, which is revealed by the proof of Lemma~\ref{lemma:bounded-Linfty-vt-sum-et}.
\end{remark*}

Lemma~\ref{lemma:LRC-population-NN} below gives a sharp upper bound for the Rademacher complexity of a localized subset of the function class
$\cF(B,w)$. Based on Lemma~\ref{lemma:LRC-population-NN},
Theorem~\ref{theorem:bounded-NN-class}, and
using the local Rademacher complexity based analysis
\cite{bartlett2005},
%{\cite[Theorem 3.3]{bartlett2005}}
Theorem~\ref{theorem:LRC-population-NN-eigenvalue} presents
a sharp upper bound for the nonparametric regression risk,
$\Expect{P}{(f_t-f^*)^2}$, where $f_t$ is the function represented by
the two-layer NN (\ref{eq:two-layer-nn}) right after the
$t$-th step of GD using Algorithm~\ref{alg:GD}.

\begin{lemma}\label{lemma:LRC-population-NN}
For every $B,w > 0$ every $r > 0$,
\bal\label{eq:LRC-population-NN}
&\cfrakR
\pth{\set{f \in \cF(B,w) \colon \Expect{P}{f^2} \le r}}
\le \varphi_{B,w}(r),
\eal%
where
\bal\label{eq:varphi-LRC-population-NN}
\varphi_{B,w}(r) &\defeq
\min_{Q \colon Q \ge 0} \pth{({\sqrt r} + w) \sqrt{\frac{Q}{n}} +
B
\pth{\frac{\sum\limits_{q = Q+1}^{\infty}\lambda_q}{n}}^{1/2}} + w.
\eal
\end{lemma}

We define
\bal\label{eq:eps-K-eig-def}
\eps^{\textup{(eig)}}_{K} \defeq \min_{0 \le Q \le n} \pth{\frac{Q}{n}
+ \pth{\frac{\sum\limits_{q = Q+1}^{\infty}\lambda_q}{n}}^{1/2}},
\eal
which is also the sharp excess risk bound for kernel learning introduced in \cite[Corollary 6.7]{bartlett2005}, then we have the following theorem.
\begin{theorem}\label{theorem:LRC-population-NN-eigenvalue}
Suppose $w  \in (0,1)$,
$m$ satisfies (\ref{eq:m-cond-bounded-NN-class}),
and the neural network
$f_t = f(\bW(t),\cdot)$ is
trained by GD using Algorithm~\ref{alg:GD} with the learning rate
 $\eta \in (0,2/u_0^2)$ on the random initialization $\bW(0) \in \cW_0$, and $T \le \hat T$.
Then for every $t \in [T]$, with probability at least
$1 -  \exp\pth{-\Theta(n)}
- \exp\pth{-\Theta(n \hat\eps_n^2)}-\exp\pth{-n \eps_n^2}$
over the random noise $\bw$ and the random training features $\bS$,
\bal\label{eq:LRC-population-NN-bound-eigenvalue}
&\Expect{P}{(f_t-f^*)^2} - 2 \Expect{P_n}{(f_t-f^*)^2}
\lsim B_0^2 \eps^{\textup{(eig)}}_{K} +B_0^2 \eps_n^2 + B_0 w,
\eal
where $B_0 \defeq {(B_h+\mu_0) u_0}/{\sqrt 2} + 1$, $\eps^{\textup{(eig)}}_{K}$ is defined in (\ref{eq:eps-K-eig-def}).
Furthermore,
with probability at least
$1 -  \exp\pth{-\Theta(n)}
- \exp\pth{-\Theta(n \hat\eps_n^2)}-\exp\pth{-n \eps_n^2} - 2/n$
over the random noise $\bw$ and the random training features $\bS$,
\bal\label{eq:LRC-population-NN-bound-fixed-point-detail}
&\Expect{P}{(f_t-f^*)^2} - 2 \Expect{P_n}{(f_t-f^*)^2}
\lsim B_0^4 \eps_n^2+ B_0 w.
\eal
\end{theorem}

Theorem~\ref{theorem:empirical-loss-bound} below shows
that the empirical loss $\Expect{P_n}{(f_t-f^*)^2}$ is bounded by
$\Theta(1/(\eta t))$ with high probability over $\bw$.
Such upper bound for the empirical loss by Theorem~\ref{theorem:empirical-loss-bound}
will be plugged in the risk bound in
Theorem~\ref{theorem:LRC-population-NN-eigenvalue} to prove
Theorem~\ref{theorem:LRC-population-NN-fixed-point}. The proofs of
Theorem~\ref{theorem:LRC-population-NN-fixed-point} and its
corollary are presented in the next subsection.

\begin{theorem}\label{theorem:empirical-loss-bound}
 Suppose the neural network trained after the $t$-th step of gradient descent, $f_t = f(\bW(t),\cdot)$, satisfies $\bu(t) = f_t(\bS) - \by = \bv(t) + \be(t)$
with $\bv(t) \in \cV_t$, $\be(t) \in \cE_{t,\tau}$ and $T \le \hat T$. If
\bal\label{eq:eta-tau-cond-empirical-loss-convergence}
\eta \in (0,2/u_0^2), \quad \tau \le \frac{1}{{\eta T}},
\eal
then for every $t \in [T]$, with probability at least
$1-\exp\pth{- \Theta(n\hat\eps_n^2)}$ over the random noise $\bw$, we have
\bal\label{eq:empirical-loss-bound}
\Expect{P_n}{(f_t-f^*)^2} &\le
\frac{3}{\eta t} \pth{\frac{\mu_0^2}{2e}+\frac {1}{\eta}+2}.
\eal
\end{theorem}

\subsection{Proofs for the Main Result,
Theorem~\ref{theorem:LRC-population-NN-fixed-point} and
Corollary~\ref{corollary:minimax-nonparametric-regression}}
\label{sec:proofs-main-results}

\begin{proof}
[\textbf{\textup{Proof of
Theorem~\ref{theorem:LRC-population-NN-fixed-point}}}]
We use Theorem~\ref{theorem:LRC-population-NN-eigenvalue} and
Theorem~\ref{theorem:empirical-loss-bound} to prove this
theorem.

First of all, it follows by
Theorem~\ref{theorem:empirical-loss-bound} that with probability at least
$1-\exp\pth{- \Theta(n\hat\eps_n^2)}$,
\bals
\Expect{P_n}{(f_t-f^*)^2} &\le
\frac{3}{\eta t} \pth{\frac{\mu_0^2}{2e}+\frac {1}{\eta}+2}.
\eals
Plugging such bound for $\Expect{P_n}{(f_t-f^*)^2}$ in
(\ref{eq:LRC-population-NN-bound-fixed-point-detail})
of Theorem~\ref{theorem:LRC-population-NN-eigenvalue}
leads to
\bal\label{eq:LRC-population-NN-fixed-point-seg1}
&\Expect{P}{(f_t-f^*)^2} - \frac{6}{\eta t} \pth{\frac{\mu_0^2}{2e}+\frac {1}{\eta}+2}
\lsim B_0^4 \eps_n^2+ B_0 w,
\eal
where $B_0 = {(B_h+\mu_0) u_0}/{\sqrt 2} + 1$.
Due to the definition of $\hat T$ and $\hat \eps_n^2$, we have
\bal\label{eq:LRC-population-NN-fixed-point-seg4}
\hat \eps_n^2 \le \frac {1}{\eta \hat T} \le \frac {2}{\eta (\hat T+1)}
\le 2 \hat \eps_n^2.
\eal
Since $n \gsim 1/\lambda_1$, Lemma~\ref{lemma:hat-eps-eps-relation} suggests that with probability at least $1-4\exp(-\Theta(n\eps_n^2))$ over $\bS$,
$\hat \eps_n^2 \asymp \eps_n^2$. Since $T \asymp \hat T$, for every
$t \in [c_t T,T]$, we have
\bal\label{eq:LRC-population-NN-fixed-point-seg2}
\frac{1}{\eta t} \asymp \frac{1}{\eta T} \asymp \frac{1}{\eta \hat T}
\asymp \hat \eps_n^2 \asymp \eps_n^2.
\eal
We also have $\Prob{\cW_0} \ge 1-2/n$. Let $w = \eps_n^2$, we now verify that $w \in (0,1)$. Due to the definition of the fixed point, $w > 0$. Since
$\sum\limits_{i \ge 1} \lambda_i = \int_{\cX} K(\bx,\bx) \diff \mu(\bx) = 1/2$, we have
\bals
0< w = \sigma_0 \sqrt{\frac 1n \sum\limits_{i \ge 1} \min\set{\lambda_i,\eps_n^2}}
\le \sigma_0 \sqrt{\frac 1n \sum\limits_{i \ge 1} \lambda_i}
\le \sigma_0 \sqrt{\frac {u_0^2}{2n}} < 1
\eals
with $n \gsim \sigma_0^2 u_0^2/2$.
(\ref{eq:LRC-population-NN-fixed-point}) then follows from
(\ref{eq:LRC-population-NN-fixed-point-seg1}) with $w = \eps_n^2$,
(\ref{eq:LRC-population-NN-fixed-point-seg2}) and the union bound.
We note that $c_{\bu}$ and $u_0$ are bounded by positive constants, so
the condition on $m$ in (\ref{eq:m-cond-bounded-NN-class})
in Theorem~\ref{theorem:bounded-NN-class}, together with $w = \eps_n^2$ and
(\ref{eq:LRC-population-NN-fixed-point-seg2}) leads to
the condition on $m$ in (\ref{eq:m-cond-LRC-population-NN-fixed-point}).
Furthermore, $\hat T \asymp  \eps_n^{-2}$
follows from (\ref{eq:LRC-population-NN-fixed-point-seg2})
and $\eta = \Theta(1)$.

Finally, by the definition of $\eps_n^2$ we have $\eps_n^2 = \sigma_0 \sqrt{\frac 1n \sum\limits_{i=1}^{\infty} \min\set{\lambda_i,\eps_n^2}}
\lsim \sigma_0/{\sqrt n}$. Condition (\ref{eq:m-cond-LRC-population-NN-fixed-point}) on $m$, $m \gsim d/{\eps_n^{20}}$, ensures that $m$
satisfies the conditions on $m$ in Theorem~\ref{theorem:good-random-initialization}. As a result, $\Prob{\bW(0) \in \cW_0} \ge 1 - 2/n$.
\end{proof}

\begin{proof}
[\textbf{\textup{Proof of
Corollary~\ref{corollary:minimax-nonparametric-regression}}}]

We apply Theorem~\ref{theorem:LRC-population-NN-fixed-point}
to prove this corollary. It follows from \cite[Corollary 3]{RaskuttiWY14-early-stopping-kernel-regression},
that $\eps_n^2 \asymp n^{-\frac{2\alpha}{2\alpha+1}}$. It then can be verified
 by direct calculations that the condition on $m$,
 (\ref{eq:m-cond-LRC-population-NN-fixed-point}) in
 Theorem~\ref{theorem:LRC-population-NN-fixed-point}, is satisfied with
 the given condition (\ref{eq:m-N-cond-LRC-population-NN-concrete}).
It then follows from (\ref{eq:LRC-population-NN-fixed-point}) in
 Theorem~\ref{theorem:LRC-population-NN-fixed-point} that
 $\Expect{P}{(f_t-f^*)^2} \lsim n^{-\frac{2\alpha}{2\alpha+1}}$.

\end{proof}

\subsection{Proofs for Results in Section~\ref{sec:detailed-roadmap-key-results}}
\label{sec:proofs-key-results}

\newtheorem{innercustomthm}{{\bf{Theorem}}}
\newenvironment{customthm}[1]
  {\renewcommand\theinnercustomthm{#1}\innercustomthm}
  {\endinnercustomthm}

%\begin{customthm}{\bf{C.10 (repeat)}}
%\end{customthm}

The proofs of
Lemma~\ref{lemma:LRC-population-NN}, Theorem~\ref{theorem:empirical-loss-bound} are deferred to Section~\ref{sec:lemmas-main-results} of
the appendix.
We have the following two theorems, Theorem~\ref{theorem:sup-hat-g} and Theorem~\ref{theorem:V_R}, regarding the uniform convergence to $K(\cdot,\cdot)$ and
the uniform convergence to $\frac{2R}{\sqrt {2\pi} \kappa}$ on the unit sphere $\unitsphere{d}$.
The proofs of Theorem~\ref{theorem:sup-hat-g} and Theorem~\ref{theorem:V_R} are deferred to Section~\ref{sec:proofs-uniform-convergence-sup-hat-g-V_R} of the appendix.

\begin{theorem}\label{theorem:sup-hat-g}
Let ${\bW(0)} = \set{{{\bbw_r(0)}}}_{r=1}^m$, where each ${{\bbw_r(0)}} \sim \cN(\bzero,\kappa^2 \bI_{d+1})$ for $r \in [m]$. Then for any $\delta \in (0,1)$, with probability at least $1-\delta$ over ${\bW(0)}$,
\bal\label{eq:sup-hat-h}
\sup_{\tbx \in \unitsphere{d},\tby \in \unitsphere{d}} \abth{ \tK(\tbx,\tby) - \tilde h({\bW(0)},\tbx,\tby) } \le C_1(m,d,\delta),
\eal%
where $\tK(\tbx,\tby) \defeq  \frac{ \tbx^{\top}\tby} {2\pi} \pth{\pi -\arccos (\tbx^{\top}\tby)}$, $\tilde h(\bw,\tbx,\tby) \defeq \tbx^{\top} \tby \indict{\bw^{\top} \tbx \ge 0} \indict{\bw^{\top} \tby \ge 0}$, $\tilde h({\bW(0)},\tbx,\tby)
\defeq \sum_{r=1}^m \tilde h(\bbw_r(0),\tbx,\tby)/m$ for all $\tbx,\tby \in \unitsphere{d}$ and $\bw \in \RR^{d+1}$,
\bal\label{eq:C1}
&C_1(m,d,\delta) \defeq \frac{1}{\sqrt m} \pth{6(1+2B\sqrt{d}) + \sqrt{2\log{\frac {(1+2m)^{2(d+1)}} \delta}}} + \frac{1}{m} \pth{6+\frac{14 {\log{\frac {(1+2m)^{2(d+1)}} \delta}} }{3}},
\eal%
and $B$ is an absolute positive constant in Lemma~\ref{lemma:uniform-marginal-bound}. In addition, when $m \ge n^{1/(2(d+1))}$, $m/\log m \ge d$, and $\delta \asymp 1/n$,
$C_1(m,d,\delta) \lsim \sqrt{\frac{d \log m}{m}} + \frac{d \log m}{m} \lsim \sqrt{\frac{d \log m}{m}}$.
\end{theorem}

%$\tau$ is an arbitrary constant such that $\tau \in (0,\frac 12)$,
\begin{theorem}\label{theorem:V_R}
Let ${\bW(0)} = \set{{{\bbw_r(0)}}}_{r=1}^m$, where each ${{\bbw_r(0)}} \sim \cN(\bzero,\kappa^2 \bI_{d+1})$ for $r \in [m]$. $B$ is an absolute positive constant in Lemma~\ref{lemma:uniform-marginal-bound}. Suppose $\eta \lsim 1$, $m \gsim 1$. Then for any $\delta \in (0,1)$, with probability at least $1-\delta$ over ${\bW(0)}$,
\bal\label{eq:sup-hat-V_R}
&\sup_{\tbx \in \unitsphere{d}}\abth{\tilde v_R({\bW(0)},\tbx)-\frac{2R}{\sqrt {2\pi} \kappa}} \le C_2(m,d,\delta),
\eal%
where $\tilde v_R(\bw,\tbx) \defeq \indict{\abth{\bw^{\top}\tbx} \le R}$,  $\tilde v_R({\bW(0)},\tbx) \defeq \sum_{r=1}^m \tilde v_R(\bbw_r(0),\tbx) /m$ for all $\tbx \in \unitsphere{d}$ and $\bw \in \RR^{d+1}$,
\bal\label{eq:C2}
%\resizebox{1\hsize}{!}{$$}.
C_2(m,d,\delta) \defeq 3 \sqrt{\frac{d}{\kappa}} m^{-\frac 15} T^{\frac 12}
+ \sqrt{\frac{2{\log{\frac {(1+2{\sqrt m})^{d+1}} \delta}}}{m}} + \frac{7 {\log{\frac {(1+2{\sqrt m})^{d+1}} \delta}} }{3m}.
\eal
In addition, when $m \gsim n^{2/(d+1)}$, $m/\log m \ge d$, and $\delta \asymp 1/n$,
$C_2(m,d,\delta) \lsim  \sqrt{d} m^{-\frac 15} T^{\frac 12}$.
\end{theorem}

\begin{proof}
[\textbf{\textup{Proof of
Theorem~\ref{theorem:empirical-loss-convergence}}}]

First, when $m \gsim  {(\eta c_{\bu})^5 u_0^{10} T^{\frac {15}{2}} d^{\frac 52}}/{\tau^5}$ with a proper constant, it can be verified that $\bE_{m,\eta,\tau} \le {\tau {\sqrt n}}/{T}$ where $\bE_{m,\eta,\tau}$ is defined by
(\ref{eq:empirical-loss-Et-bound-Em}) of
Lemma~\ref{lemma:empirical-loss-convergence-contraction}.
Also, Theorem~\ref{theorem:sup-hat-g}
and Theorem~\ref{theorem:V_R} hold when (\ref{eq:m-cond-empirical-loss-convergence})
holds.
We then use mathematical induction to prove this theorem. We will first prove that $\bu(t) = \bv(t) + \be(t)$ where $\bv(t) \in \cV_t$,
$\be(t) \in \cE_{t,\tau}$, and $\ltwonorm{\bu(t)} \le c_{\bu} \sqrt n$ for for all $t \in [0,T]$.

When $t = 0$, we have
\bal\label{eq:empirical-loss-convergence-seg1}
\bu(0) = - \by &= \bv(0) + \be(0),
\eal
where $\bv(0) \defeq -f^*(\bS) = -\pth{\bI-\eta \bK_n}^0 f^*(\bS)$,
$\be(0) = -\bw = \bbe_1(0) + \bbe_2(0)$ with
$\bbe_1(0) = -\big(\bI-\eta \bK_n \big) ^0 \bw$
and $\bbe_2(0) = \bzero$. Therefore,
$\bv(0) \in \cV_{0}$ and $\be(0) \in \cE_{0,\tau}$.
Also, it follows from the proof of Lemma~\ref{lemma:yt-y-bound}
that $\ltwonorm{\bu(0)} \le  c_{\bu}{\sqrt n}$ with probability at least
$1 -  \exp\pth{-\Theta(n)}$
over the random noise $\bw$.

Suppose that for all $t_1 \in[0,t]$ with $t \in [0,T-1]$, $\bu(t_1) = \bv(t_1) + \be(t_1)$ where $\bv(t_1) \in \cV_{t_1}$, and
$\be(t_1) = \bbe_1(t_1) + \bbe_2(t_1)$ with
$\bv(t_1) \in \cV_{t_1}$ and $\be(t_1) \in \cE_{t_1,\tau}$ for all $t_1 \in[0,t]$. Then it follows from Lemma~\ref{lemma:empirical-loss-convergence-contraction} that the recursion
$\bu(t'+1)  = \pth{\bI- \eta \bK_n }\bu(t')
 +\bE(t'+1)$ holds for all $t' \in [0,t]$.
 As a result, we have
\bal\label{eq:empirical-loss-convergence-seg5}
\bu(t+1)  &= \pth{\bI- \eta \bK_n }\bu(t) +\bE(t+1) \nonumber \\
 & = -\pth{\bI-\eta \bK_n}^{t+1} f^*(\bS)
 -\pth{\bI-\eta \bK_n }^{t+1} \bw +\sum_{t'=1}^{t+1}
 \pth{\bI-\eta \bK_n}^{t+1-t'} \bE(t')
\nonumber \\
&=\bv(t+1) + \be(t+1),
\eal
where $\bv(t+1)$ and $\be(t+1)$ are defined as
\bal\label{eq:empirical-loss-convergence-vt-et-def}
\bv(t+1) \defeq-\pth{\bI-\eta \bK_n}^{t+1} f^*(\bS)\in \cV_{t+1},
\eal
\bal\label{eq:empirical-loss-convergence-et-pre}
&\be(t+1) \defeq \underbrace{-\pth{\bI-\eta \bK_n }^{t+1} \bw}_{\bbe_1(t+1)}
+ \underbrace{ \sum_{t'=1}^{t+1}
 \pth{\bI-\eta \bK_n}^{t+1-t'} \bE(t') }_{\bbe_2(t+1)}.
\eal
We now prove the upper bound for $\bbe_2(t+1)$.
With $\eta \in (0, 2/u_0^2)$, we have $\ltwonorm{\bI - \eta \bK_n} \in (0,1)$.
It follows that
\bal\label{eq:empirical-loss-convergence-et-bound}
&\ltwonorm{\bbe_2(t+1)} \le \sum_{t'=1}^{t+1} \ltwonorm{\bI-\eta \bK_n}^{t+1-t'}\ltwonorm{\bE(t')}
\le  \tau {\sqrt n},
%\nonumber \\
%&\le \pth{c_{\bu}-c_{\bu} + \tau} {\sqrt n},
\eal
where the last inequality follows from the fact
that $\ltwonorm{\bE(t)} \le \bE_{m,\eta,\tau} \le {\tau {\sqrt n}}/{T}$ for all $t \in [T]$. It follows that $\be(t+1) \in \cE_{t+1,\tau}$.
Also, it follows from Lemma~\ref{lemma:yt-y-bound}
that
\bals
\ltwonorm{\bu(t+1)} &\le \ltwonorm{\bv(t+1)} + \ltwonorm{\bbe_1(t+1)}
+\ltwonorm{\bbe_2(t+1)} \le\pth{\frac{\mu_0}{ \sqrt{2e\eta } } + \sigma_0+\tau+1} {\sqrt n} \le \ c_{\bu}{\sqrt n}.
\eals
The above inequality completes the induction step, which also completes the proof. It is noted that
$\ltwonorm{\bbw_r(t) - \bbw_r(0)} \le R$ holds for all $t \in [T]$ by Lemma~\ref{lemma:weight-vector-movement}.

\end{proof}

\begin{proof}
[\textbf{\textup{Proof of
Theorem~\ref{theorem:bounded-NN-class}}}]
Let $\tbx = \bth{\bx^{\top},1}^{\top}$.
In this proof we abbreviate $f_t$ as $f$ and $\bW(t)$ as $\bW$.
It follows from Theorem~\ref{theorem:empirical-loss-convergence}
and its proof that conditioned on an event $\Omega$ with probability at
least $1 -  \exp\pth{-\Theta(n)}$,
% subset $\Omega_1 \subseteq (\unitsphere{d-1})^n \times (\unitsphere{d-1})^N$ with
% $\Prob{\Omega_1} \ge 1-\Theta\pth{{nN}/{n^{c_d\eps^2_0/8}}} -
% \pth{1+2N}^{2d}\exp(-n^{c_x})$ and a subset $\Omega_2 \subseteq \RR^n$
% such that $\Prob{\Omega_2} \ge 1 -  \exp\pth{-\Theta(n)}$ that
% when the random training data $\bS = \set{\bbx_j}_{j=1}^n$ and $\bQ$ satisfy %$(\bS,\bQ) \in \Omega_1$ and the random noise $\bw \in \Omega_2$, the
% neural network trained on the $\bS$ with $\bQ$ and $\bw$ enjoys the properties %specified by Theorem~\ref{theorem:empirical-loss-convergence}. In particular,
$f \in \cFnn(\bS,\bW(0),T)$ with
$\bW(0) \in \cW_0$. Moreover, $f = f(\bW,\cdot)$ with $\bW = \set{\bbw_r}_{r=1}^m \in \cW(\bS,\bW(0),T)$, and $\vect{\bW} = \vect{\bW_{\bS}} = \vect{\bW(0)} - \sum_{t'=0}^{t-1} \eta/n \cdot \bZ_{\bS}(t') \bu(t')$ for some $t \in [T]$, where $\bu(t') \in \RR^n, \bu(t') = \bv(t') + \be(t')$ with $\bv(t') \in \cV_{t'}$ and $\be(t') \in \cE_{t',\tau}$ for all $t' \in [0,t-1]$. It also follows from Theorem~\ref{theorem:empirical-loss-convergence}
that conditioned on $\Omega$, $\ltwonorm{\bbw_r(t) - \bbw_r(0)} \le R$ for all $t \in [T]$.

$\bbw_r$ is expressed as
\bal\label{eq:bounded-Linfty-function-class-wr}
\bbw_r = \bbw_{\bS,r}(t) &= \bbw_r(0) - \sum_{t'=0}^{t-1} \frac{\eta}{n} \bth{\bZ_{\bS}(t')}_{[(r-1)(d+1)+1:r(d+1)]} \bu(t'),
\eal%
where the notation $\bbw_{\bS,r}$ emphasizes that $\bbw_r$ depends on the training data $\bS$.
We define the event
\bals%\label{eq:bounded-Linfty-function-class-Er}
&E_r(R) \defeq \set{ \abth{ \bbw_r(0)^\top \tbx} \le R },
\quad
\bar E_r(R) \defeq \set{ \abth{ \bbw_r(0)^\top \tbx} > R },
\quad r\in [m].
\eals%
We now approximate $ f(\bW,\bx)$ by $g(\bx) \defeq \frac{1}{\sqrt m} \sum_{r=1}^m a_r
\indict{\bbw_r(0)^{\top} \tbx \ge 0} \bbw_r^\top \tbx$.
 We have
\bal\label{eq:bounded-Linfty-function-class-seg1}
&\abth{f(\bW,\bx) - g(\bx)}=\frac{1}{\sqrt m} \abth{\sum\limits_{r=1}^m a_r \relu{\bbw_r^\top \tbx} - \sum_{r=1}^m a_r \indict{\bbw_r(0)^{\top} \tbx \ge 0} \bbw_r^\top \tbx}   \nonumber \\
&\le \frac{1}{\sqrt m}  \sum_{r=1}^m  \abth{ a_r \pth{ \indict{E_r(R) } + \indict{\bar E_r(R) }}  \pth{ \relu{\bbw_r^\top \tbx} - \indict{\bbw_r(0)^{\top} \tbx \ge 0} \bbw_r^\top \tbx } } \nonumber \\
&=\frac{1}{\sqrt m} \sum_{r=1}^m   \indict{E_r(R)} \abth{ \relu{\bbw_r^\top \tbx} - \indict{\bbw_r(0)^{\top} \tbx \ge 0} \bbw_r^\top \tbx  }\nonumber \\
&=\frac{1}{\sqrt m} \sum_{r=1}^m   \indict{E_r(R)}\abth{ \relu{\bbw_r^\top \tbx} - \relu{\bbw_r(0)^\top \tbx}   - \indict{\bbw_r(0)^{\top} \tbx \ge 0} (\bbw_r-\bbw_r(0))^\top \tbx } \nonumber \\
&\le \frac{2R u_0}{\sqrt m} \sum_{r=1}^m   \indict{E_r(R)},
\eal%
where first inequality follows from $\indict{\bar E_r(R)} \pth{ \relu{\bbw_r^\top \tbx} - \indict{\bbw_r(0)^{\top} \tbx \ge 0} \bbw_r^\top \tbx } = 0$.
Plugging $R = \frac{\eta c_{\bu} u_0 T }{\sqrt m}$ in (\ref{eq:bounded-Linfty-function-class-seg1}), since $\bW(0) \in \cW_0$, we have
\bal\label{eq:bounded-Linfty-function-class-seg2}
&\sup_{\bx \in \cX} \abth{f(\bW,\bx) - g(\bx)} \le  2\eta c_{\bu} u_0^2 T \cdot \frac 1m \sum_{r=1}^m
 \indict{E_r(R)}  \le
2 \eta c_{\bu} u_0^2 T \pth{\frac{2R}{\sqrt {2\pi} \kappa} + C_2(m/2,d,1/n)}.
\eal
Using (\ref{eq:bounded-Linfty-function-class-wr}), $g(\bx)$
is expressed as
\bal\label{eq:bounded-Linfty-function-class-seg3}
&g(\bx) = \frac{1}{\sqrt m} \sum_{r=1}^m a_r \sigma(\bbw_r(0)^\top \tbx) {-} \sum_{t'=0}^{t-1} \frac{1}{\sqrt m}
\sum\limits_{r=1}^m  \indict{\bbw_r(0)^{\top} \tbx \ge 0}\pth{ \frac{\eta}{n} \bth{\bZ_{\bS}(t')}_{[(r-1)(d+1)+1:r(d+1)]}
 \bu(t')
 }^\top \tbx  \nonumber \\
&\stackrel{\circled{1}}{=} -\sum_{t'=0}^{t-1}
\underbrace{\frac{\eta}{nm}
\sum\limits_{r=1}^m  \indict{\bbw_r(0)^{\top} \tbx \ge 0}
\sum\limits_{j=1}^n\indict{\bbw_r(t')^{\top} \tbbx_j \ge 0} \bu_j(t')\tbbx_j^{\top} \tbx}_{\defeq G_{t'}(\bx)},
\eal
where $\circled{1}$ follows from the fact that $\frac{1}{\sqrt m} \sum_{r=1}^m a_r  \sigma(\bbw_r(0)^\top \tbx) = f(\bW(0),\bx) = 0$ due to the particular initialization of the two-layer NN
(\ref{eq:two-layer-nn}). For each $G_{t'}$ in the RHS of
 (\ref{eq:bounded-Linfty-function-class-seg3}), we have
\bal\label{eq:bounded-Linfty-function-class-Gt}
G_{t'}(\bx) &\stackrel{\circled{2}}{=} \frac{\eta}{nm}
\sum\limits_{r=1}^m  \indict{\bbw_r(0)^{\top} \tbx \ge 0}
\sum\limits_{j=1}^n \pth{d_{t',r,j}+\indict{\bbw_r(0)^{\top} \tbbx_j \ge 0}} \bu_j(t')\tbbx_j^{\top} \tbx
\nonumber \\
&\stackrel{\circled{3}}{=}
\frac{\eta}{n}
\sum\limits_{j=1}^n K(\bx,\bbx_j)\bu_j(t') +
\underbrace{\frac{\eta}{n} \sum\limits_{j=1}^n q_j \bu_j(t')}_{\defeq E_1(\bx)}
 + \underbrace{\frac{\eta}{nm}
\sum\limits_{r=1}^m  \indict{\bbw_r(0)^{\top} \tbx \ge 0}
 \sum\limits_{j=1}^nd_{t',r,j} \bu_j(t')\tbbx_j^{\top} \tbx}_{\defeq E_2(\bx)}.
\eal
where $d_{t',r,j} \defeq \indict{\bbw_r(t')^{\top} \tbbx_j \ge 0}
- \indict{\bbw_r(0)^{\top} \tbbx_j \ge 0}$
in $\circled{2}$, and
$q_j \defeq \hat h(\bW(0),\bbx_j,\bx) - K(\bbx_j,\bx)$
for all $j \in [n]$ in $\circled{3}$.
We now analyze each term on the RHS of (\ref{eq:bounded-Linfty-function-class-Gt}).
Let $h(\cdot,t') \colon \cX \to \RR$ be defined by
$h(\bx,t') \defeq -\frac{\eta}{n} \sum\limits_{j=1}^n K(\bx,\bbx_j) \bu_j(t')$,
then $h(\cdot,t') \in \cH_K$ for each $t' \in [0,t-1]$. We further define
\bal\label{eq:bounded-Linfty-function-class-h}
h_t(\cdot) \defeq \sum_{t'=0}^{t-1} h(\cdot,t') \in \cH_K,
\eal
Since $\bW(0) \in \cW_0$, $q_j \le u_0^2 C_1(m/2,d,1/n)$ for all $j' \in [n]$ with
$C_1(m/2,d,1/n)$ defined in (\ref{eq:C1}). Moreover,
$\bu(t') \le c_{\bu} \sqrt n$ with high probability,  so that we have
\bal\label{eq:bounded-Linfty-function-class-E1-bound}
\supnorm{ E_1} = \supnorm{\frac{\eta}{n} \sum\limits_{j=1}^n q_j \bu_j(t')} &\le \frac{\eta}{n} \ltwonorm{\bu(t')}  \sqrt{n} u_0^2C_1(m/2,d,1/n)
\le \eta c_{\bu}u_0^2C_1(m/2,d,1/n) .
\eal
We now bound the last term on the RHS of (\ref{eq:bounded-Linfty-function-class-Gt}).
Define $\bX' \in \RR^{d \times n}$ with its $j$-column being ${\bX'}^{(j)}
= \frac{1}{m} \sum_{r=1}^m \indict{\bbw_r(0)^{\top} \tbx \ge 0}
d_{t',r,j}\tbbx_j$ for all $j \in [n]$, then
$E_2(\bx)= \frac{\eta}{n}\pth{\bX' \bu(t')}^{\top}\bx$.

We need to derive the upper bound for $\ltwonorm{\bX'}$. Because $\ltwonorm{\bbw_r(t') - \bbw_r(0)} \le R$, it follows that  $\indict{\bbw_r(t')^{\top}
\tbbx_j \ge 0} = \indict{\bbw_r(0)^{\top} \tbbx_j \ge 0}$ when $\abth{\bbw_r(0)^{\top} \tbbx_j} > R$ for all $j \in [n]$. Therefore,
\bals
\abth{d_{t',r,j'}} = \abth{ \indict{\bbw_r(t')^{\top} \tbbx_j \ge 0}
- \indict{\bbw_r(0)^{\top} \tbbx_j \ge 0} } \le
\indict{\abth{\bbw_r(0)^{\top} \tbbx_j } \le R},
\eals
and it follows that
\bal\label{eq:bounded-Linfty-function-class-E2-1}
\frac{\abth{ \sum\limits_{r=1}^m \indict{\bbw_r(0)^{\top} \tbbx_i \ge 0}
d_{t',r,j}  } }{m}  &\le \frac{ \sum\limits_{r=1}^m \abth{d_{t',r,j}}  }{m}
\le \frac{ \sum\limits_{r=1}^m \indict{\abth{\bbw_r(0)^{\top} \tbbx_j } \le R}  }{m} = \hat v_R(\bW(0),\bbx_j) \nonumber \\
&\le \frac{2R}{\sqrt {2\pi} \kappa} + C_2(m/2,d,1/n),
\eal
where $\hat v_R$ is defined by (\ref{eq:v-hat-v}), and the last inequality
follows from Theorem~\ref{theorem:V_R}.

It follows from (\ref{eq:bounded-Linfty-function-class-E2-1}) that
$\ltwonorm{\bX'} \le {\sqrt n} u_0 \pth{\frac{2R}{\sqrt {2\pi} \kappa}
+ C_2(m/2,d,1/n)}$,
and we have
\bal\label{eq:bounded-Linfty-function-class-E2-bound}
\supnorm{ E_2(\bx)} &\le \frac{\eta}{n}\ltwonorm{\bX'}
\ltwonorm{\bu(t')}
\ltwonorm{\bx} \le \eta  c_{\bu} u_0^2 \pth{\frac{2R}{\sqrt {2\pi} \kappa}
+ C_2(m/2,d,1/n)}.
\eal
Combining (\ref{eq:bounded-Linfty-function-class-Gt}),
(\ref{eq:bounded-Linfty-function-class-E1-bound}), and
(\ref{eq:bounded-Linfty-function-class-E2-bound}), for
any $t' \in [0,t-1]$,
\bal\label{eq:bounded-Linfty-function-class-Gt-ht-bound}
&\sup_{\bx \in \cX} \abth{G_{t'}(\bx)-h(\bx,t')}
\le \supnorm{E_1} + \supnorm{E_2}
\le \eta c_{\bu} u_0^2 \pth{ C_1(m/2,d,1/n)
+\frac{2R}{\sqrt {2\pi} \kappa}
+ C_2(m/2,d,1/n) }.
\eal
Define $e_t(\bx') = f(\bW,\bx') - h_t(\bx')$ for $\bx' \in \cX$, and $e_t(\bx') = 0$ for $\bx \in \cX \setminus \cX$. It then follows from (\ref{eq:bounded-Linfty-function-class-seg2}),
(\ref{eq:bounded-Linfty-function-class-seg3}), and
(\ref{eq:bounded-Linfty-function-class-Gt-ht-bound})
that
\bal\label{eq:bounded-Linfty-function-class-f-h-bound}
&\supnorm{e_t} \le \sup_{\bx \in \cX} \abth{f(\bW,\bx) - g(\bx)}
+\sup_{\bx \in \cX} \abth{g(\bx)-h_t(\bx)} \nonumber \\
&\le  \sup_{\bx \in \cX} \abth{f(\bW,\bx) - g(\bx)} + \sum\limits_{t'=0}^{t-1}
\sup_{\bx \in \cX} \abth{G_{t'}(\bx)-h(\bx,t')} \nonumber \\
& \stackrel{\circled{4}}{\le}
2 \eta c_{\bu}u_0^2T \pth{\frac{2R}{\sqrt {2\pi} \kappa} + C_2(m/2,d,1/n)}
+ \eta c_{\bu} u_0^2T \pth{ C_1(m/2,d,1/n)
+\frac{2R}{\sqrt {2\pi} \kappa}
+ C_2(m/2,d,1/n) }
\nonumber \\
&\le \eta c_{\bu} u_0^2T \pth{
C_1(m/2,d,1/n)
+3\pth{\frac{2R}{\sqrt {2\pi} \kappa}
+ C_2(m/2,d,1/n)} }
\defeq \Delta_{m,n,\eta,T},
\eal
where $\circled{4}$ follows from
(\ref{eq:bounded-Linfty-function-class-seg2}) and
(\ref{eq:bounded-Linfty-function-class-Gt-ht-bound}).
We now give an estimate for $\Delta_{m,n,\eta,T}$.
Since $\bW(0) \in \cW_0$, it follows from
Theorem~\ref{theorem:good-random-initialization} that
%we have $\sqrt{\frac{d \log m}{m}} \le \frac{\sqrt d}{m^{1/4}}$.
%As a result,
\bals
\Delta_{m,n,\eta,T}
&\lsim
\sqrt{d} m^{-\frac 15} T^{\frac 32}.
\eals
By direct calculations, for any $w > 0$, when
$m \gsim {T^{\frac {15}{2}} d^{\frac 52}}/{w^5}$,
we have
$\Delta_{m,n,\eta,T} \le w$.

It follows from Lemma~\ref{lemma:bounded-Linfty-vt-sum-et} that
with probability at least $1- \exp\pth{-\Theta(n \hat\eps_n^2)}$ over the random noise $\bw$,
$\norm{h_t}{\cH_K} \le B_h$,
where $B_h$ is defined in (\ref{eq:B_h}),
and $\tau$ is required to satisfy $\tau \le \min\set{\Theta(1/(\eta u_0 T)),1}$.
Theorem~\ref{theorem:empirical-loss-convergence} requires that
$m \gsim  {(\eta c_{\bu})^5 u_0^{10} T^{\frac {15}{2}} d^{\frac 52}}/{\tau^5}$. As a result,
we have $m \gsim \max\{{(\eta c_{\bu})^5 u_0^{10} T^{\frac {15}{2}} d^{\frac 52}}, \\
\allowdisplaybreaks \eta^{10} c_{\bu}^5
u_0^{15} T^{\frac{25}{2}} d^{\frac 52}\}$.
%It also follows from
%the  Cauchy-Schwarz inequality that
%$\supnorm{h_t} \le {B_h}/{\sqrt 2}$. This together with
%(\ref{eq:bounded-Linfty-function-class-f-h-bound})
%proves
%(\ref{eq:bounded-Linfty-function-class}).
\end{proof}

\begin{proof}
[\textbf{\textup{Proof of
Theorem~\ref{theorem:LRC-population-NN-eigenvalue}}}]
%We first remark that the conditions on $m$,
%(\ref{eq:m-cond-LRC-population-NN-eigenvalue}), is required by
%Theorem~\ref{theorem:empirical-loss-convergence} and Theorem~\ref{theorem:bounded-NN-class}.

It follows from Theorem~\ref{theorem:empirical-loss-convergence}
and Theorem~\ref{theorem:bounded-NN-class} that for every $t \in [T]$, conditioned on an event $\Omega$ with probability at least $1 -  \exp\pth{-\Theta(n)}
- \exp\pth{-\Theta(n \hat\eps_n^2)}$
over the random noise $\bw$, we have $\bW(t) \in \cW(\bS,\bW(0),T)$,
and $f(\bW(t),\cdot) = f_t \in \cFnn(\bS,\bW(0),T)$.
Moreover, conditioned on the event $\Omega$, $f_t = h_t + e_t$ where $h_t \in \cH_{K}(B_h)$ and
$e_t  \in L^{\infty}$ with $\supnorm{e_t } \le w$.

We then derive the sharp upper bound for $\Expect{P}{(f-f^*)^2} $ by
applying Theorem~\ref{theorem:LRC-population} to the function class
$\cF = \set{F=\pth{f- f^*}^2 \colon f \in
\cF(B_h,w)   }$.
Let $B_0 \defeq {(B_h+\mu_0) u_0}/{\sqrt 2} + 1
\ge {(B_h+\mu_0) u_0}/{\sqrt 2} + w$, then
$\supnorm{F} \le B^2_0$ with $F \in \cF$, so that
$\Expect{P}{F^2} \le B^2_0\Expect{P}{F}$.
Let $T(F) = B^2_0\Expect{P}{F}$ for $F \in \cF$. Then
$\Var{F} \le \Expect{P}{F^2} \le T(F) = B^2_0\Expect{P}{F}$.

We have
\bal\label{eq:LRC-population-NN-seg1}
\cfrakR \pth{\set{F \in \cF \colon T(F) \le r}} &= \cfrakR
\pth{ \set{(f-f^*)^2 \colon  f \in \cF(B_h,w),\Expect{P}{(f-f^*)^2}
\le \frac r{B^2_0}}} \nonumber \\
&\stackrel{\circled{1}}{\le} 2B_0 \cfrakR \pth{\set{f-f^* \colon
  f \in \cF(B_h,w), \Expect{P}{(f-f^*)^2} \le \frac{r}{B_0^2}}}\nonumber \\
&\stackrel{\circled{2}}{\le}  4B_0 \cfrakR \pth{ \set{f \in \cF(B_h,w) \colon \Expect{P}{f^2} \le \frac{r}{4B_0^2}} },
\eal
where $\circled{1}$ is due to the contraction property of
Rademacher complexity in Theorem~\ref{theorem:RC-contraction}.
Since $f^* \in \cF(B_h,w)$,
$f \in \cF(B_h,w)$, we have $\frac{f-f^*}{2} \in \cF(B_h,w)$ due to the fact that
$\cF(B_h,w)$ is  symmetric and convex, and it follows that $\circled{2}$
holds.

It follows from (\ref{eq:LRC-population-NN-seg1})
and Lemma~\ref{lemma:LRC-population-NN} that
\bal\label{eq:LRC-population-NN-seg2}
B^2_0 \cfrakR \pth{\set{F \in \cF \colon T(F) \le r}}
&\le 4 B_0^3 \cfrakR \pth{ \set{f \colon f \in \cF(B_h,w) , \Expect{P}{f^2} \le
\frac{r}{4B_0^2}}} \nonumber \\
&\le 4 B_0^3 \varphi_{B_h,w}\pth{\frac{r}{4B_0^2}} \defeq \psi(r).
\eal
$\psi$ defined as the RHS of (\ref{eq:LRC-population-NN-seg2}) is a sub-root function since it is nonnegative, nondecreasing and
$\frac{\psi(r)}{\sqrt r}$ is nonincreasing. Let $r^*$ be the fixed point of $\psi$, and $0 \le r \le r^*$. It follows from {\cite[Lemma 3.2]{bartlett2005}} that
$0 \le r \le \psi(r) =  4 B_0^3 \varphi_{B_h,w}\pth{\frac{r}{4B_0^2}}$.
Therefore, by the definition of $\varphi_{B_h,w}$ in (\ref{eq:varphi-LRC-population-NN}),
for every $0 \le Q \le n$, we have
\bal\label{eq:LRC-population-NN-seg3}
\frac{r}{4 B_0^3} \le \pth{ \frac{\sqrt r}{2B_0} + w} \sqrt{\frac{Q}{n}} +
B_h
\pth{\frac{\sum\limits_{q = Q+1}^{\infty}\lambda_q}{n}}^{1/2}+w.
\eal
Solving the quadratic inequality (\ref{eq:LRC-population-NN-seg3}) for $r$, we have
\bal\label{eq:LRC-population-NN-seg4}
r \le \frac{8B_0^4 Q}{n} + 8B_0^3
\pth{ w \pth{\sqrt{\frac{Q}{n}}+1}
+ B_h \pth{\frac{\sum\limits_{q = Q+1}^{\infty}\lambda_q}{n}}^{1/2}
}.
\eal
(\ref{eq:LRC-population-NN-seg4}) holds for every $0 \le Q \le n$,
so we have
\bal\label{eq:LRC-population-NN-seg5}
r \le 8 B_0^3 \min_{0 \le Q \le n} \pth{\frac{ B_0Q}{n} +w \pth{\sqrt{\frac{Q}{n}}+1}
+ B_h
\pth{\frac{\sum\limits_{q = Q+1}^{\infty}\lambda_q}{n}}^{1/2}}.
\eal
It then follows from (\ref{eq:LRC-population-NN-seg2}) and
Theorem~\ref{theorem:LRC-population} that with probability at least
$1-\exp(-x)$
over the random training features $\bS$,
\bal\label{eq:LRC-population-NN-risk-E1-bound}
&\Expect{P}{(f_t-f^*)^2} - \frac{K_0}{K_0-1} \Expect{P_n}{(f_t-f^*)^2}-\frac{x\pth{11B_0^2+26B_0^2 K_0}}{n} \le \frac{704K_0}{B_0^2} r^*,
\eal
or
\bal\label{eq:LRC-population-NN-risk-E1-bound-simple}
&\Expect{P}{(f_t-f^*)^2} - 2 \Expect{P_n}{(f_t-f^*)^2} \lsim
\frac{1}{B_0^2} r^* +\ \frac {B_0^2x}n,
\eal
with $K_0 = 2$ in (\ref{eq:LRC-population-NN-risk-E1-bound}).
It follows from (\ref{eq:LRC-population-NN-seg5})
and (\ref{eq:LRC-population-NN-risk-E1-bound-simple}) that
\bal\label{eq:LRC-population-NN-risk-eigenvalue-pre}
&\Expect{P}{(f_t-f^*)^2} - 2 \Expect{P_n}{(f_t-f^*)^2} \lsim B_0^2 \min_{0 \le Q \le n} \pth{\frac{ Q}{n} + \pth{\frac{\sum\limits_{q = Q+1}^{\infty}\lambda_q}{n}}^{1/2}} +
\frac {B_0^2x}n+B_0w.
\eal
Let $x = n\eps_n^2$ in the above inequality, and we note that
the above argument requires
Theorem~\ref{theorem:bounded-NN-class} which holds with probability
at least
$1 -  \exp\pth{-\Theta(n)}
- \exp\pth{-\Theta(n \hat\eps_n^2)}$ over the random noise $\bw$.
Then
(\ref{eq:LRC-population-NN-bound-eigenvalue}) is
proved combined with the facts that $\Prob{\cW_0} \ge 1 - 2/n$.

We now prove (\ref{eq:LRC-population-NN-bound-fixed-point-detail}).
First, it follows from the definition of $ \varphi_{B_h,w}$
in (\ref{eq:varphi-LRC-population-NN}) that
\bals
&\psi(r)=4 B_0^3 \varphi_{B_h,w}\pth{\frac{r}{4B_0^2}}
=4 B_0^3  \min_{Q \colon Q \ge 0} \pth{\pth{\frac{\sqrt r}{2B_0} + w}
 \sqrt{\frac{Q}{n}} +
B_h
\pth{\frac{\sum\limits_{q = Q+1}^{\infty}\lambda_q}{n}}^{1/2}} +
4 B_0^3 w
\nonumber \\
&\le 4 B_0^3 B_h\min_{Q \colon Q \ge 0} \pth{\sqrt{\frac{Qr}{n}} +
\pth{\frac{\sum\limits_{q = Q+1}^{\infty}\lambda_q}{n}}^{1/2}}
+8 B_0^3w
\le \frac{ 4{\sqrt 2} B_0^3 B_h}{\sigma_0} \cdot \sigma_0 R_K(\sqrt r) +
8 B_0^3w\defeq \psi_1(r),
\eals
where the last inequality follows from the Cauchy-Schwarz inequality and the definition of the kernel complexity.
It can be verified that $ \psi_1(r)$ is a sub-root function.
Let the fixed point of $ \psi_1(r)$ be $r_1^*$.
Because the fixed point of $\sigma_0 R_K(\sqrt r)$ as a function of
$r$ is $\eps^2_n$, it follows from Lemma~\ref{lemma:sub-root-fix-point-properties}
that
\bal\label{eq:LRC-population-NN-bound-fixed-point-detail-seg1}
r_1^*\le \max\set{\frac{32 B_0^6 B^2_h}{\sigma_0^2},1} \eps^2_n+ 16 B_0^3w.
\eal
It then follows from Theorem~\ref{theorem:LRC-population}
with $K_0 = 2$ that with probability at least $1-\exp(-x)$,
\bals
&\Expect{P}{(f_t-f^*)^2} - 2 \Expect{P_n}{(f_t-f^*)^2}
\lsim \frac{1}{B_0^2} r_1^*+ \frac { B_0^2x}n .
\eals
Letting $x = n\eps_n^2$, then
plugging the upper bound for $r_1^*$,
(\ref{eq:LRC-population-NN-bound-fixed-point-detail-seg1}), in the  above inequality leads to
\bal\label{eq:LRC-population-NN-bound-fixed-point-detail-seg2}
\Expect{P}{(f_t-f^*)^2} - 2 \Expect{P_n}{(f_t-f^*)^2}
\lsim
 B_0^4 \eps_n^2+ B_0 w.
\eal
Again, we note that the above argument requires Theorem~\ref{theorem:bounded-NN-class} which holds with probability
at least
$1 -  \exp\pth{-\Theta(n)}
- \exp\pth{-\Theta(n \hat\eps_n^2)}$ over the random noise $\bw$.
Then
(\ref{eq:LRC-population-NN-bound-fixed-point-detail}) is
proved by (\ref{eq:LRC-population-NN-bound-fixed-point-detail-seg2}).

\end{proof}

%\section{Simulation Study}
%\label{sec:simulation}
%
%We present simulation results for GD in this section.
%We randomly sample $n$ points $\set{\bbx_i}_{i=1}^n$  as a i.i.d. sample of random
%variables distributed uniformly on the unit sphere in
%$\RR^{50}$.  $n$ ranges within $[100,1000]$ with a step size of $100$.
%We set the target function to $f^*(\bx) = \bs^{\top} \bx$
%where $\bs \sim \Unif{\unitsphere{d-1}}$ is randomly sampled.
%We also uniformly and independenly sample $1000$ points on the unit sphere in $\RR^{50}$ as the test data.
%We train the two-layer NN (\ref{eq:two-layer-nn}) using either GD by
%Algoirthm~\ref{alg:GD}
%or GD by Algoirthm~\ref{alg:GD}
%with $m \asymp n^2$ on a NVIDIA A100 GPU card with a learning rate $\eta = 0.1$, and report the test loss
%in Figure~\ref{fig:simulation}. It can be observed that
%GD always enjoys lower test loss than GD across different training data size,
%supporting the sharper bounds by GD we present in the main paper.
%
%\begin{figure}[!hbtp]
%\begin{center}
%\includegraphics[width=0.8\textwidth]{test_loss.pdf}
%\end{center}
%\caption{Illustration of the test loss by GD and GD}
%\label{fig:simulation}
%\end{figure}

\section{Conclusion}

We study nonparametric regression
by training an over-parameterized two-layer neural network
where the target function is in the RKHS
associated with the NTK of the  neural network.
We show that, if the neural network is trained by gradient descent (GD) with early stopping, a sharp rate of the risk with the order of
$\Theta(\eps_n^2)$ can be obtained without distributional assumptions about the bounded covariate, with
$\eps_n$ being the critical population rate or the critical radius of the NTK. A novel proof strategy is employed to achieve this result, and we compare our results to the current state-of-the-art with a detailed roadmap of our technical approach.

%\newpage

\begin{appendices}

\section{Mathematical Tools}
\label{sec::math-tools}
The appendix of this paper is organized as follows.
We present the basic mathematical results employed in our proofs in Section~\ref{sec::math-tools}, and then present the detailed proofs
in Section~\ref{sec:proofs}. More results about the eigenvalue decay rates are presented in Section~\ref{sec:more-results-edr}, and simulation results are presented in
Section~\ref{sec:simulation}.

\begin{theorem}[{\cite[Theorem 2.1]{bartlett2005}}]
\label{theorem:Talagrand-inequality}
Let $\unitsphere{d-1},P$ be a probability space, $\set{\bbx_i}_{i=1}^n$ be independent random variables distributed according to $P$. Let $\cF$ be a class of functions that map $\unitsphere{d-1}$ into $[a, b]$. Assume that there is some $r > 0$
such that for every $f \in \cF$,$\Var{f(\bbx_i)} \le r$. Then, for every $x > 0$, with probability at least $1-e^{-x}$,
\bal\label{eq:Talagrand-inequality}
\resizebox{0.99\hsize}{!}{$
\sup_{f \in \cF} \big( \E_{P} [f(\bx)] - \E_{P_n } [f(\bx)] \big) \le \inf_{\alpha > 0} \bigg( 2(1+\alpha) \E_{\set{\bbx_i}_{i=1}^n,\set{\sigma_i}_{i=1}^n} [R_n \cF] + \sqrt{\frac{2rx}{n}} + (b-a) \pth{\frac{1}{3}+\frac{1}{\alpha}} \frac{x}{n} \bigg)$},
\eal%
%%&\phantom{\quad \quad}
and with probability at least $1-2e^{-x}$,
\bal\label{eq:Talagrand-inequality-empirical}
\resizebox{0.995\hsize}{!}{$
\sup_{f \in \cF} \big( \E_{P} [f(\bx)] - \E_{P_n } [f(\bx)] \big)
\le \inf_{\alpha \in (0,1)} \bigg(\frac{2(1+\alpha)}{1-\alpha} \E_{\set{\sigma_i}_{i=1}^n} [R_n \cF] + \sqrt{\frac{2rx}{n}}
 + (b-a) \pth{ \frac{1}{3}+\frac{1}{\alpha} + \frac{1+\alpha}{2\alpha(1-\alpha)} } \frac{x}{n} \bigg)$}.
\eal%
%&\phantom{\quad \quad}
$P_n$ is the empirical distribution over $\set{\bbx_i}_{i=1}^n$ with
$\Expect{P_n}{f(\bx)} = \frac{1}{n} \sum\limits_{i=1}^n f(\bbx_i)$. Moreover, the same results hold for $\sup_{f \in \cF} \big( \E_{P_n } [f(\bx)] - \E_{P} [f(\bx)] \big)$.
\end{theorem}

In addition, we have the contraction property for Rademacher complexity, which is due to Ledoux
and Talagrand~\cite{Ledoux-Talagrand-Probability-Banach}.

\begin{theorem}\label{theorem:RC-contraction}
Let $\phi$ be a contraction,that is, $\abth{\phi(x) - \phi(y)} \le \mu \abth{x-y}$ for $\mu > 0$. Then, for every function class $\cF$,
\bal\label{eq:RC-contraction}
&\Expect{\set{\sigma_i}_{i=1}^n} {R_n \phi \circ \cF} \le \mu \Expect{\set{\sigma_i}_{i=1}^n} {R_n \cF},
\eal%
where $\phi \circ \cF$ is the function class defined by $\phi \circ \cF = \set{\phi \circ f \colon f \in \cF}$.
\end{theorem}

\begin{definition}[Sub-root function,{\cite[Definition 3.1]{bartlett2005}}]
\label{def:sub-root-function}
A function $\psi \colon [0,\infty) \to [0,\infty)$ is sub-root if it is nonnegative,
nondecreasing and if $\frac{\psi(r)}{\sqrt r}$ is nonincreasing for $r >0$.
\end{definition}

\begin{theorem}[{\cite[Theorem 3.3]{bartlett2005}}]
\label{theorem:LRC-population}
Let $\cF$ be a class of functions with ranges in $[a, b]$ and assume
that there are some functional $T \colon \cF \to \RR+$ and some constant $\bar B$ such that for every $f \in \cF$ , $\Var{f} \le T(f) \le \bar B P(f)$. Let $\psi$ be a sub-root function and let $r^*$ be the fixed point of $\psi$.
Assume that $\psi$ satisfies that, for any $r \ge r^*$,
$\psi(r) \ge \bar B \cfrakR(\set{f \in \cF \colon T (f) \le r})$. Fix $x > 0$,
then for any $K_0 > 1$, with probability at least $1-e^{-x}$,
\bals
\forall f \in \cF, \quad \Expect{P}{f} \le \frac{K_0}{K_0-1} \Expect{P_n}{f} + \frac{704K_0}{\bar B} r^*
+ \frac{x\pth{11(b-a)+26 \bar B K_0}}{n}.
\eals
Also, with probability at least $1-e^{-x}$,
\bals
\forall f \in \cF, \quad \Expect{P_n}{f} \le \frac{K_0+1}{K_0} \Expect{P}{f}  + \frac{704K_0}{\bar B} r^*
+ \frac{x\pth{11(b-a)+26 \bar B K_0}}{n}.
\eals
\end{theorem}

\section{Detailed Proofs}
\label{sec:proofs}

In Section~\ref{sec:lemmas-main-results}, we present
the proofs of
Lemma~\ref{lemma:LRC-population-NN}, Theorem~\ref{theorem:empirical-loss-bound}, and the lemmas required for the proofs in Section~\ref{sec:proofs-key-results}.
Proofs of Theorem~\ref{theorem:sup-hat-g} and Theorem~\ref{theorem:V_R} are presented in Section~\ref{sec:proofs-uniform-convergence-sup-hat-g-V_R}.

%\subsection{Proofs of Theorem~\ref{theorem:empirical-loss-convergence}, Theorem~\ref{theorem:bounded-NN-class}, Theorem~\ref{theorem:LRC-population-NN-eigenvalue},
%Theorem~\ref{theorem:empirical-loss-bound}}

\subsection{Proof of
Lemma~\ref{lemma:LRC-population-NN} and the Lemmas Required for the Proofs in Section~\ref{sec:proofs-key-results}}
\label{sec:lemmas-main-results}

\begin{proof}
[\textbf{\textup{Proof of
Lemma~\ref{lemma:LRC-population-NN}}}]
According to the definition of $\cF(B,w)$ in (\ref{eq:def-cF-ext-general}), for any $f \in  \cF(B,w)$,
we have $f = h + e$ with $h \in \cH_K(B)$,
$e \in L^{\infty}$, $\supnorm{e} \le w$.
We first decompose the
Rademacher complexity of the function class
$\big\{f \in \cF(B,w) \colon \Expect{P}{f^2} \le r\big\}$ into two terms as follows:

\bal\label{eq:lemma-LRC-population-NN-decomp}
&\cfrakR \pth{\set{f \colon f \in \cF(B,w) , \Expect{P}{f^2} \le r}} \nonumber \\
&\le \frac 1n \Expect{}
{\sup_{f \in \cF(B,w) \colon \Expect{P}{f^2} \le r }
{ \sum\limits_{i=1}^n {\sigma_i}{h(\bbx_i)}}
} + \frac 1n \Expect{}
{\sup_{f \in \cF(B,w) \colon \Expect{P}{f^2} \le r } {{ \sum\limits_{i=1}^n {\sigma_i}{e(\bbx_i)}}
} } \nonumber \\
&= \underbrace{\frac 1n \Expect{}
{\sup_{f \in \cF(B,w) \colon \Expect{P}{f^2} \le r }
{\iprod{h}
{\sum\limits_{i=1}^n {\sigma_i}{K(\cdot,\bbx_i)}}_{\cH_K}}
}}_{\defeq \cR_1} +
\underbrace{\frac 1n \Expect{}
{\sup_{f \in \cF(B,w) \colon \Expect{P}{f^2} \le r } {{ \sum\limits_{i=1}^n {\sigma_i}{e(\bbx_i)}}
} }}_{\defeq \cR_2}
\eal

We now analyze the upper bounds for $\cR_1, \cR_2$ on the RHS
of (\ref{eq:lemma-LRC-population-NN-decomp}).

\textbf{Derivation for the upper bound for $\cR_1$.}

% It follows that
% \bal\label{eq:lemma-LRC-population-NN-Bh-restriction}
% \norm{e}{L^2}^2 &= \int_{\Omega_{\eps_0,\bQ}} e^2(\bx) \diff \mu(\bx)
% + \int_{\unitsphere{d-1} \setminus \Omega_{\eps_0,\bQ}} e^2(\bx) \diff \mu(\bx)
% \nonumber \\
% &\le w^2
% + 2N \exp(-d\eps^2_0/8 )
% \pth{\eta T/4+w}^2
%  \nonumber \\
%  &\le w^2
% + \frac 2{n^{c_d \eps_0^2/8-\log_n N}}
% \pth{\eta T/4+w}^2\defeq r_e^2.
% \eal
% We define the functions $\bar f, \bar h, \bar e$ as the restrictions of
% $f,h,e$ on $\Omega_{\eps_0,\bQ}$:
% \bals
% \bar f \vert_{\Omega_{\eps_0,\bQ}} = f\vert_{\Omega_{\eps_0,\bQ}},
% \bar f \vert_{\unitsphere{d-1} \setminus \Omega_{\eps_0,\bQ}} = 0, \\
% \bar h \vert_{\Omega_{\eps_0,\bQ}} = h\vert_{\Omega_{\eps_0,\bQ}},
% \bar h \vert_{\unitsphere{d-1} \setminus \Omega_{\eps_0,\bQ}} = 0, \\
% \bar e \vert_{\Omega_{\eps_0,\bQ}} = e\vert_{\Omega_{\eps_0,\bQ}},
% \bar e \vert_{\unitsphere{d-1} \setminus \Omega_{\eps_0,\bQ}} = 0,
% \eals
% and we have $\bar f = \bar h + \bar e$.

When $\Expect{P}{f^2} \le r$, it follows from the triangle inequality
that $\norm{h}{L^2} \le \norm{f}{L^2} + \norm{e}{L^2} \le {\sqrt r} + w \defeq r_h$.
% and $\Prob{\unitsphere{d-1} \setminus \Omega_{\eps_0,\bQ}} \le 2N \exp(-d\eps^2_0/8 )$ that
% \bal\label{eq:lemma-LRC-population-NN-h-L2-bound}
% &\norm{h}{L^2}^2 =
% \int_{\Omega_{\eps_0,\bQ}} h^2(\bx) \diff \mu(\bx)
% + \int_{\unitsphere{d-1} \setminus \Omega_{\eps_0,\bQ}} h^2(\bx) \diff \mu(\bx)
% \nonumber \\
% &\le \int_{\Omega_{\eps_0,\bQ}} 2
% \pth{f^2(\bx) + e^2(\bx)} \diff \mu(\bx)
% + \int_{\unitsphere{d-1} \setminus \Omega_{\eps_0,\bQ}} h^2(\bx) \diff \mu(\bx)
% \nonumber \\
% &\le 2r^2 + 2w^2  + N \exp(-d\eps^2_0/8 )B_h,
% \eal
% so that
% $\norm{h}{L^2} \le {\sqrt 2}(r+w)+n^{-\pth{c_d \eps_0^2/16-(\log_n N)/2}}
% \sqrt {B_h} = r_h$.
We now consider $h \in \cH_{K}(B)$ with $\norm{h}{L^2} \le r_h$ in the remaining part of this proof.
%We have
%\bal\label{eq:lemma-LRC-population-NN-seg1}
%\sum\limits_{i=1}^n {\sigma_i}{f(\bbx_i)} &=
%\sum\limits_{i=1}^n {\sigma_i}\pth{h(\bbx_i) + e(\bbx_i)}
%\nonumber \\
%&=
%\iprod{h}
%{\sum\limits_{i=1}^n {\sigma_i}{K(\cdot,\bbx_i)}}_{\cH_K} +
%\sum\limits_{i=1}^n {\sigma_i}e(\bbx_i).
%\eal
Because $\set{v_q}_{q \ge 1}$ is an orthonormal basis of $\cH_K$, for any $0 \le Q \le n$, we have
\bal\label{eq:lemma-LRC-population-NN-seg2}
\iprod{h}
{\sum\limits_{i=1}^n {\sigma_i}{K(\cdot,\bbx_i)}}_{\cH_K}
&=\iprod{\sum\limits_{q=1}^Q \sqrt{\lambda_q}
\iprod{h}{v_q}_{\cH_K} v_q }
{\sum\limits_{q=1}^Q \frac{1}{\sqrt{\lambda_q}}
\iprod{\sum\limits_{i=1}^n{\sigma_i}{K(\cdot,\bbx_i)}}{v_q}_{\cH_K}v_q}_{\cH_K}
\nonumber \\
&\phantom{=}+\iprod{h}
{\sum\limits_{q > Q} \iprod{\sum\limits_{i=1}^n {\sigma_i}{K(\cdot,\bbx_i)}}{v_q}_{\cH_K}v_q}_{\cH_K}.
\eal
Due to the fact that $h \in \cH_K$,
$h = \sum\limits_{q =1}^{\infty}  \bbeta^{(h)}_q v_q
=\sum\limits_{q =1}^{\infty}  \sqrt{\lambda_q} \bbeta^{(h)}_q e_q $
with $v_q = \sqrt{\lambda_q} e_q$. Therefore,
$\norm{h}{L^2}^2 = \sum\limits_{q=1}^{\infty} \lambda_q {\bbeta^{(h)}_q}^2$, and
\bal\label{eq:lemma-LRC-population-NN-seg3}
\norm{\sum\limits_{q=1}^Q \sqrt{\lambda_q}\iprod{h}{v_q}_{\cH_K}v_q}{\cH_K}
&= \norm{\sum\limits_{q=1}^Q \sqrt{\lambda_q} \bbeta^{(h)}_q v_q}{{\cH_K}}
= \sqrt{\sum\limits_{q=1}^Q \lambda_q  {\bbeta^{(h)}_q}^2}
\le
\norm{h}{L^2} \le r_h.
\eal
According to the Mercer's Theorem, because the kernel $K$ is continuous symmetric positive definite, it has the decomposition $K(\cdot,\bbx_i) =
\sum\limits_{j=1}^{\infty} \lambda_j e_j(\cdot)
e_j(\bbx_i)$, so that we have
 \bal\label{eq:lemma-LRC-population-NN-seg4}
\iprod{\sum\limits_{i=1}^n{\sigma_i}{K(\cdot,\bbx_i)}}{v_q}_{\cH_K}
&=\iprod{\sum\limits_{i=1}^n{\sigma_i} \sum\limits_{j=1}^{\infty} \lambda_j e_j e_j(\bbx_i) }{v_q}_{\cH_K} \nonumber \\
&=\iprod{\sum\limits_{i=1}^n{\sigma_i} \sum\limits_{j=1}^{\infty}
\sqrt{\lambda_j}e_j(\bbx_i) \cdot v_j  }{v_q}_{\cH_K}
=\sum\limits_{i=1}^n{\sigma_i} \sqrt{\lambda_q}e_q(\bbx_i).
\eal
Combining (\ref{eq:lemma-LRC-population-NN-seg2}),
(\ref{eq:lemma-LRC-population-NN-seg3}), and
(\ref{eq:lemma-LRC-population-NN-seg4}), we have
\bal\label{eq:lemma-LRC-population-NN-seg5}
\iprod{h}
{\sum\limits_{i=1}^n {\sigma_i}{K(\cdot,\bbx_i)}}_{\cH_K}
&\stackrel{\circled{1}}{\le}
\norm{\sum\limits_{q=1}^Q \sqrt{\lambda_q}\iprod{h}{v_q}_{\cH_K}v_q}{\cH_K}
\cdot
\norm{\sum\limits_{q=1}^Q \frac{1}{\sqrt{\lambda_q}}
\iprod{\sum\limits_{i=1}^n{\sigma_i}{K(\cdot,\bbx_i)}}{v_q}_{\cH_K}v_q}{\cH_K}
\nonumber \\
&\phantom{\le}+ \norm{h}{\cH_K} \cdot
\norm{\sum\limits_{q = Q+1}^{\infty} \iprod{\sum\limits_{i=1}^n{\sigma_i}{K(\cdot,\bbx_i)}}{v_q}_{\cH_K}v_q}{\cH_K}
\nonumber \\
&\le \norm{h}{L^2}
\norm{\sum\limits_{q=1}^Q \sum\limits_{i=1}^n{\sigma_i} e_q(\bbx_i)v_q}{\cH_K}
+ B
\norm{\sum\limits_{q = Q+1}^{\infty} \sum\limits_{i=1}^n{\sigma_i} \sqrt{\lambda_q}e_q(\bbx_i) v_q}{\cH_K} \nonumber \\
&\le r_h \sqrt{\sum\limits_{q=1}^Q \pth{\sum\limits_{i=1}^n{\sigma_i} e_q(\bbx_i)}^2}
+ B
\sqrt{\sum\limits_{q = Q+1}^{\infty} \pth{\sum\limits_{i=1}^n {\sigma_i} \sqrt{\lambda_q}e_q(\bbx_i)}^2},
\eal
where $\circled{1}$ is due to (\ref{eq:lemma-LRC-population-NN-seg3}) and the Cauchy-Schwarz inequality.
Moreover, by the Jensen's inequality we have
\bal\label{eq:lemma-LRC-population-NN-seg6}
\Expect{}{\sqrt{\sum\limits_{q=1}^Q \pth{\sum\limits_{i=1}^n{\sigma_i} e_q(\bbx_i)}^2}}
&\le \sqrt{ \Expect{}{\sum\limits_{q=1}^Q \pth{\sum\limits_{i=1}^n{\sigma_i} e_q(\bbx_i)}^2 } }\le \sqrt{
\Expect{}{\sum\limits_{q=1}^Q \sum\limits_{i=1}^n e_q^2(\bbx_i) }}
=\sqrt{nQ}.
\eal
and similarly,
\bal\label{eq:lemma-LRC-population-NN-seg7}
\Expect{}{\sqrt{\sum\limits_{q = Q+1}^{\infty} \pth{\sum\limits_{i=1}^n{\sigma_i\sqrt{\lambda_q}} e_q(\bbx_i)}^2}}
&\le \sqrt{
\Expect{}{\sum\limits_{q = Q+1}^{\infty} \lambda_q \sum\limits_{i=1}^n e_q^2(\bbx_i) }}
=\sqrt{n\sum\limits_{q = Q+1}^{\infty}\lambda_q}.
\eal
Since (\ref{eq:lemma-LRC-population-NN-seg5})-(\ref{eq:lemma-LRC-population-NN-seg7}) hold for all $Q \ge 0$,
it follows that
\bal\label{eq:lemma-LRC-population-NN-seg8}
\Expect{}{\sup_{h \in \cH_K(B), \norm{h}{L^2} \le r_h } {\frac{1}{n} \sum\limits_{i=1}^n {\sigma_i}{h(\bbx_i)}} }
\le \min_{Q \colon Q \ge 0} \pth{ r_h \sqrt{nQ} +
B
\sqrt{n\sum\limits_{q = Q+1}^{\infty}\lambda_q}}.
\eal
It follows from (\ref{eq:lemma-LRC-population-NN-decomp})
and
(\ref{eq:lemma-LRC-population-NN-seg8}) that
\bal\label{eq:lemma-LRC-population-NN-R1}
\cR_1 &\le \frac 1n \Expect{}{\sup_{h \in \cH_K(B), \norm{h}{L^2} \le r_h } { \sum\limits_{i=1}^n {\sigma_i}{h(\bbx_i)}} }
\le \min_{Q \colon Q \ge 0} \pth{r_h \sqrt{\frac{Q}{n}} +
B
\pth{\frac{\sum\limits_{q = Q+1}^{\infty}\lambda_q}{n}}^{1/2}}.
\eal

\textbf{Derivation for the upper bound for $\cR_2$.}

Because  $\abth{1/n \sum_{i=1}^n \sigma_i e(\bbx_i) }\le w$
when $\supnorm{e} \le w$, we have

\bal\label{eq:lemma-LRC-population-NN-R2}
\cR_2 \le \frac 1n \Expect{}{\sup_{e \in L^{\infty} \colon \supnorm{e} \le w } {{ \sum\limits_{i=1}^n {\sigma_i}{e(\bbx_i)}} } }
\le w.
\eal

It follows from (\ref{eq:lemma-LRC-population-NN-R1})
and (\ref{eq:lemma-LRC-population-NN-R2}) that
\bals
&\cfrakR \pth{\set{f \colon f \in \cF(B,w), \Expect{P}{f^2} \le r}}
\le \min_{Q \colon Q \ge 0} \pth{r_h \sqrt{\frac{Q}{n}} +
B
\pth{\frac{\sum\limits_{q = Q+1}^{\infty}\lambda_q}{n}}^{1/2}}
+ w.
\eals

Plugging $r_h$ in the RHS of the above inequality
 completes the proof.
\end{proof}

\begin{proof}
[\textbf{\textup{Proof of
Theorem~\ref{theorem:empirical-loss-bound}}}]

% It follows from Theorem~\ref{theorem:empirical-loss-convergence}
% and its proof that conditioned on an event $\Omega$ with probability at
% least $1 -  \exp\pth{-\Theta(n)} - \exp\pth{-\Theta(n \hat\eps_n^2)}-\Theta\pth{{nN}/{n^{c_d\eps^2_0/8}}} - \pth{1+2N}^{2d}\exp(-n^{c_x})$,
% $f \in \cFnn(\bS,\bW(0),T)$ with
% $\bW(0) \in \cW_0$. Moreover, $f(\cdot) = f(\bW,\cdot)$ with $\bW = \set{\bbw_r}_{r=1}^m \in \cW(\bS,\bW(0),T)$, and $\vect{\bW} = \vect{\bW_{\bS}} = \vect{\bW(0)} - \sum_{t'=0}^{t-1} \eta/n \bM \bZ_{\bS}(t') \bu(t')$ for some $t \in [T]$, where $\bu(t') \in \RR^n, \bu(t') = \bv(t') + \be(t')$ with $\bv(t') \in \cV_{t'}$ and $\be(t') \in \cE_{t',\tau}$ for all $t' \in [0,t-1]$.
% Also, $\bbe(t') = \bbe_1(t') + \bbe_2(t')$ with
% $\bbe_1(t') = -\pth{\bI_n-\eta\bK_n}^{t'} \bw$
% and $\ltwonorm{\bbe_2(t')} \lsim {\sqrt n} \pth{\tau +{\eta T n^{c_x}}/{N}}$ for all $t' \in [0,t-1]$.

We have
\bal\label{eq:empirical-loss-bound-seg1}
f_t(\bS) = f^*(\bS) + \bw + \bv(t) + \be(t),
\eal
where $\bv(t) \in \cV_{t}$, $\be(t) \in \cE_{t,\tau}$,
$\bbe(t) = \bbe_1(t) + \bbe_2(t)$ with
$\bbe_1(t) = -\pth{\bI_n-\eta\bK_n}^{t} \bw$
and $\ltwonorm{\bbe_2(t)} \le {\sqrt n} \tau$.
We have $\eta \hlambda_1 \in (0,1)$ if $\eta \in (0,2/u_0^2)$.
It follows from (\ref{eq:empirical-loss-bound-seg1}) that
\bsal\label{eq:empirical-loss-bound-seg2}
\Expect{P_n}{(f_t-f^*)^2}
&=\frac 1n \ltwonorm{f_t(\bS) - f^*(\bS)}^2  =\frac 1n \ltwonorm{\bv(t)+\bw+\be(t)}^2 \nonumber \\
&=\frac 1n \ltwonorm{-\pth{\bI- \eta \bK_n }^t f^*(\bS)
+\pth{\bI_n -\pth{\bI_n-\eta\bK_n}^t }\bw +\bbe_2(t)}^2 \nonumber \\
&\stackrel{\circled{1}}{\le} \frac 3n \sum\limits_{i=1}^n
\pth{1 - \eta \hlambda_i }^{2t}
\bth{{\bU}^{\top} f^*(\bS)}_i^2 + \frac 3n \sum\limits_{i=1}^{n} \pth{1-
\pth{1-\eta \hlambda_i  }^t}^2
\bth{{\bU}^{\top} \bw}_i^2 + \frac 3n \ltwonorm{\bbe_2(t)}^2 \nonumber \\
&
\nonumber \\
&\stackrel{\circled{2}}{\le} \frac{3\mu_0^2}{ 2e\eta t } + \frac 3n \sum\limits_{i=1}^{n} \pth{1-
\pth{1-\eta \lambda_i  }^t}^2
\bth{{\bU}^{\top} \bw}_i^2
+ 3 \tau^2 \nonumber \\
&\le \frac{3}{\eta t} \pth{\frac{\mu_0^2}{2e} +  \frac {1}{\eta}}   +
3\cdot \underbrace{\frac 1n \sum\limits_{i=1}^{n} \pth{1-\pth{1-\eta \lambda_i  }^t}^2
\bth{{\bU}^{\top} \bw}_i^2}_{ \defeq E_{\eps}} = \frac{3}{\eta t} \pth{\frac{\mu_0^2}{2e} +  \frac {1}{\eta}}   +
3 E_{\eps}.
\esal
Here $\circled{1}$ follows from the Cauchy-Schwarz inequality,
$\circled{2}$ follows from (\ref{eq:yt-y-bound-seg1}) in the proof of
Lemma~\ref{lemma:yt-y-bound}.
We then derive the upper bound for $E_{\eps}$ on the RHS of
(\ref{eq:empirical-loss-bound-seg2}), which is similar to the strategy used in \cite{RaskuttiWY14-early-stopping-kernel-regression} and also in the proof of Lemma~\ref{lemma:bounded-Linfty-vt-sum-et}. We define the diagonal matrix
$\bR \in \RR^{n \times n}$ with $\bR_{ii} =
\pth{1-\pth{1-\eta \lambda_i  }^t}^2$.
Then we have
$E_{\eps} = {\tr{\bU \bR \bU^{\top} \bw \bw^{\top} }}/n$.
It follows from~\cite{quadratic-tail-bound-Wright1973}
that
\bal\label{eq:empirical-loss-bound-E-1}
\Prob{E_{\eps} -
\Expect{}{E_{\eps}} \ge u}
\le \exp\pth{-c \min\set{nu/\ltwonorm{\bR},n^2u^2/\fnorm{\bR}^2}}
\eal
for all $u > 0$, and $c$ is a  positive constant depending on $\sigma$. With
$\eta_t = \eta t$ for all $t \ge 0$, we have
\bal\label{eq:empirical-loss-bound-E-2}
\Expect{}{E_{\eps}}
&\le \frac {\sigma_0^2}n \sum\limits_{i=1}^n
\pth{1-\pth{1-\eta \hlambda_i }^t}^2
\stackrel{\circled{1}}{\le}
\frac {\sigma_0^2}n \sum\limits_{i=1}^n
\min\set{1,\eta_t^2 \hlambda_i^2}
\nonumber \\
&\le
\frac {{\sigma_0^2}\eta_t}n \sum\limits_{i=1}^n
\min\set{\frac{1}{\eta_t},\eta_t \hlambda_i^2}
\stackrel{\circled{2}}{\le}
\frac {{\sigma_0^2}\eta_t}n \sum\limits_{i=1}^n
\min\set{\frac{1}{\eta_t}, \hlambda_i} = {{\sigma_0^2}\eta_t} \hat R_K^2(\sqrt{{1}/{\eta_t}}) \le
\frac{1}{\eta_t}.
\eal
Here $\circled{1}$ follows from the fact that
$(1-\eta \hlambda_i )^t \ge \max\set{0,1-t\eta \hlambda_i}$,
and $\circled{2}$ follows from
$\min\set{a,b} \le \sqrt{ab}$ for any nonnegative numbers $a,b$.
Because $t \le T \le \hat T$, we have
$R_K(\sqrt{{1}/{\eta_t}}) \le 1/(\sigma \eta_t)$, so the last inequality holds.
 Moreover, we have the upper bounds for $\ltwonorm{\bR}$ and $\fnorm{\bR}$
as follows. First, we have
\bal\label{eq:empirical-loss-bound-E-3}
\ltwonorm{\bR} &\le \max_{i \in [n] }
\pth{1-\pth{1-\eta \hlambda_i }^t}^2 \le \min\set{1,\eta_t^2 \hlambda_i^2}
\le 1.
\eal
We also have
\bal\label{eq:empirical-loss-bound-E-4}
\frac 1n \fnorm{\bR}^2 &=  \frac 1n
\sum\limits_{i=1}^n
\pth{1-\pth{1-\eta \hlambda_i }^t}^4
\le \frac {\eta_t}n \sum\limits_{i=1}^n
\min\set{\frac{1}{\eta_t},\eta_t^{3} \hlambda_i^4} \nonumber \\
&\stackrel{\circled{3}}{\le} \frac {\eta_t}n \sum\limits_{i=1}^n
\min\set{\hlambda_i,\frac{1}{\eta_t}}
=\eta_t\hat R_K^2(\sqrt{{1}/{\eta_t}})\le
\frac {1}{\sigma_0^2 \eta_t}.
\eal
If ${1}/{\eta_t} \le \eta_t^{3} (\hlambda_i)^4$, then
$\min\set{{1}/{\eta_t},\eta_t^{3} (\hlambda_i)^4} = {1}/{\eta_t}$. Otherwise,
we have $\eta_t^{4} \hlambda_i^4 < 1$, so that
$\eta_t \hlambda_i < 1$ and it follows that
$\min\set{{1}/{\eta_t},\eta_t^{3} (\hlambda_i)^4}
\le \eta_t^{3} \hlambda_i^4 \le \hlambda_i$. As a result,
$\circled{3}$ holds.

Combining (\ref{eq:empirical-loss-bound-E-1})-(\ref{eq:empirical-loss-bound-E-4}), we have
\bals
\Prob{E_{\eps} -
\Expect{}{E_{\eps}} \ge u}
&\le \exp\pth{-c n\min\set{u, u^2\sigma_0^2 \eta_t}}.
\eals
Let $u = 1/\eta_t$ in the above inequality, we have
\bals
\exp\pth{-c n\min\set{u, u^2\sigma_0^2 \eta_t}}
= \exp\pth{-c' n/\eta_t} \le
\exp\pth{-c'n\hat \eps_n^2},
\eals
where $c'  = c\min\set{1,\sigma_0^2}$, and the last inequality is due
to the fact that $1/\eta_t \ge \hat\eps_n^2$ since
$t \le T \le \hat T$.
It follows that with probability at least $1-
\exp\pth{- \Theta(n\hat\eps_n^2)}$,
\bal\label{eq:empirical-loss-bound-E-5}
E_{\eps}\le u+\frac{1}{\eta_t} = \frac{2}{\eta_t}.
\eal
(\ref{eq:empirical-loss-bound}) then follows from (\ref{eq:empirical-loss-bound-seg2})
and (\ref{eq:empirical-loss-bound-E-5})
with probability at least $1-\exp\pth{-c'n\hat \eps_n^2}$.

\end{proof}

\begin{lemma}\label{lemma:yt-y-bound}
Let $t \in [0,T]$, $\bv = -\pth{\bI-\eta \bK_n}^{t} f^*(\bS)$,
 $\be = -\pth{\bI-\eta \bK_n }^{t} \bw$, and $\eta \in (0,2/u_0^2)$.
Then with probability at least
$1 -  \exp\pth{-\Theta(n)}$
over the random noise $\bw$,
\bal\label{eq:yt-y-bound}
\ltwonorm{\bv} + \ltwonorm{\be} \le \pth{\max\set{{\mu_0}/{ \sqrt{2e\eta }}, \mu_0u_0/{\sqrt 2} } + \sigma_0+1} {\sqrt n}.
\eal
\end{lemma}
\begin{proof}
%Define matrix $\hat \bS \in \RR^{n \times n}$ as a diagonal matrix with
%\bal\label{eq:yt-y-bound-S}
%{\hat \bS }_{ii} = 1 - \eta \hat \lambda_i ,
%\forall i \in [n].
%\eal
When $t \ge 1$, we have
\bal\label{eq:yt-y-bound-seg1}
\ltwonorm{\bv}^2 &=\sum\limits_{i=1}^{n}
\pth{1-\eta \hlambda_i }^{2t}
\bth{{\bU}^{\top} f^*(\bS)}_i^2  \stackrel{\circled{1}}{\le}
\sum\limits_{i=1}^{n}
\frac{1}{2e\eta \hlambda_i  t}
\bth{{\bU}^{\top} f^*(\bS)}_i^2
\stackrel{\circled{2}}{\le}
\frac{n\mu_0^2}{ 2e\eta t }.
\eal

Here $\circled{1}$ follows from Lemma~\ref{lemma:auxiliary-lemma-1}, $\circled{2}$ follows from Lemma~\ref{lemma:bounded-Ut-f-in-RKHS}. Moreover, it follows from the concentration inequality about quadratic forms of sub-Gaussian random variables in \cite{quadratic-tail-bound-Wright1973} that
$\Pr\{\ltwonorm{\bw}^2 -
\Expect{}{\ltwonorm{\bw}^2} > n\}
\le \exp\pth{-\Theta(n)}$,
so that $\ltwonorm{\be} \le \ltwonorm{\bw} \le
\sqrt{\Expect{}{\ltwonorm{\bw}^2}}  + {\sqrt n}= \sqrt{n} (\sigma_0+1)$ with probability at least $1-\exp\pth{-\Theta(n)}$.
As a result, (\ref{eq:yt-y-bound}) follows from this inequality and (\ref{eq:yt-y-bound-seg1}) for $t \ge 1$.
When $t = 0$, $\ltwonorm{\bv} \le \mu_0u_0/{\sqrt 2} \cdot {\sqrt n}$, so that (\ref{eq:yt-y-bound}) still holds.

\end{proof}

\begin{lemma}
\label{lemma:empirical-loss-convergence-contraction}
Let $0<\eta<1$, $0 \le t \le T-1$ for $T \ge 1$, and suppose that $\ltwonorm{\hat \by(t') - \by} \le
 c_{\bu}{\sqrt{n}}  $ holds for all $0 \le t' \le t$ and
 the random initialization $\bW(0) \in \cW_0$. Then
\bal\label{eq:empirical-loss-convergence-contraction}
\hat \by(t+1) - \by  &= \pth{\bI- \eta \bK_n }\pth{\hat \by(t) - \by} +\bE(t+1),
\eal
where $\ltwonorm {\bE(t+1)} \le \bE_{m,\eta,\tau}$,
and $\bE_{m,\eta,\tau}$ is defined by
\bal\label{eq:empirical-loss-Et-bound-Em}
\resizebox{0.99\hsize}{!}{$
\bE_{m,\eta,\tau} \defeq \eta c_{\bu} u_0^2 {\sqrt n}
\pth{ 4 \pth{\frac{2R}{\sqrt {2\pi} \kappa}+
C_2(m/2,d,1/n)} + C_1(m/2,d,1/n)}
\lsim \eta c_{\bu} u_0^2 {\sqrt {dn}} m^{-\frac 15} T^{\frac 12}$}.
\eal
\end{lemma}

\begin{proof}

Because $\ltwonorm{\hat \by(t') - \by} \le {\sqrt{n}} c_{\bu}$ holds for all $t' \in [0,t]$, by Lemma~\ref{lemma:weight-vector-movement}, we have
\bal\label{eq:empirical-loss-convergence-pre1}
\norm{\bbw_r(t') - \bbw_r(0)}{2} & \le  R, \quad \forall \, 0 \le t' \le t+1.
\eal%
We define two sets of indices
\bals
E_{i,R} \defeq \set{r \in [m] \colon \abth{\bw_{r}(0)^{\top}\bbx_i} > R }, \quad \bar E_{i,R} \defeq [m] \setminus E_{i,R},
\eals%
then we have
\bal\label{eq:empirical-loss-convergence-contraction-seg1}
\hat \by_i(t+1) - \hat \by_i(t) &= \frac{1}{\sqrt m} \sum_{r=1}^m a_r \pth{ \relu{\bbw_{\bS,r}^{\top}(t+1) \tbbx_i}  -  \relu{\bbw_{\bS,r}^{\top}(t) \tbbx_i} } \nonumber \\
&=\underbrace{ \frac{1}{\sqrt m} \sum\limits_{r \in E_{i,R}} a_r \pth{ \relu{\bbw_{\bS,r}^{\top}(t+1) \tbbx_i}  -  \relu{\bbw_{\bS,r}^{\top}(t) \tbbx_i} }}_{\defeq \bD^{(1)}_i} \nonumber \\
&\phantom{=}{+} \underbrace{ \frac{1}{\sqrt m} \sum\limits_{r \in \bar E_{i,R}} a_r \pth{ \relu{\bbw_{\bS,r}^{\top}(t+1) \tbbx_i}  -  \relu{\bbw_{\bS,r}^{\top}(t) \tbbx_i} }}_{\defeq \bE^{(1)}_i} =\bD^{(1)}_i + \bE^{(1)}_i,
\eal%
and $\bD^{(1)}, \bE^{(1)} \in \RR^n$ are  vectors with their $i$-th element being $\bD^{(1)}_i$ and $\bE^{(1)}_i$ defined on the RHS of
 (\ref{eq:empirical-loss-convergence-contraction-seg1}).
Now we derive the upper bound for $\bE^{(1)}_i$. For all $i \in [n]$ we have
\bal\label{eq:empirical-loss-convergence-contraction-seg2}
\abth{\bE^{(1)}_i} &=  \abth{\frac{1}{\sqrt m}\sum\limits_{r \in \bar E_{i,R}} a_r \pth{ \relu{ \bbw_{\bS,r}(t+1)^\top \tbbx_i }  -  \relu{ \bbw_{\bS,r}(t)^\top \tbbx_i } } }  \nonumber \\
&\le \frac{1}{\sqrt m}\sum\limits_{r \in \bar E_{i,R}} \abth{ \bbw_{\bS,r}(t+1)^\top \tbbx_i  - \bbw_{\bS,r}(t)^\top \tbbx_i } \le \frac{u_0}{\sqrt m}\sum\limits_{r \in \bar E_{i,R}} \ltwonorm{\bbw_{\bS,r}(t+1) - \bbw_{\bS,r}(t) }  \nonumber \\
&\stackrel{\circled{1}}{=} \frac{u_0}{\sqrt m} \sum\limits_{r \in \bar E_{i,R}} \ltwonorm{\frac{\eta}{n} \bth{\bZ_{\bS}(t)}_{[(r-1)(d+1)+1:r(d+1)]} \pth{\hat \by(t) - \by}
}
\stackrel{\circled{2}}{\le} \frac{ c_{\bu} u_0^2}{\sqrt m}  \sum\limits_{r \in \bar E_{i,R}}  \frac{\eta }{\sqrt m}
= \eta c_{\bu} u_0^2 \cdot \frac{\abth{\bar E_{i,R}}}{m} .
%\stackrel{\circled{2}}{\le} \frac{c_{\bu} u_0}{\sqrt m}  \sum\limits_{r \in \bar E_{i,R}}  \frac{\eta }{\sqrt m}
%= {\eta} c_{\bu}\Theta(u_0)\cdot \frac{\abth{\bar E_{i,R}}}{m} .
\eal%
Here $\circled{1}, \circled{2}$ follow from (\ref{eq:weight-vector-movement-seg1-pre})
and (\ref{eq:weight-vector-movement-seg1}) in the proof of
Lemma~\ref{lemma:weight-vector-movement}.

Since $\bW(0) \in \cW_0 $, we have
\bal\label{eq:empirical-loss-convergence-contraction-seg3}
&\sup_{\bx \in \cX}\abth{\hat v_R(\bW(0),\bx)} \le \frac{2R}{\sqrt {2\pi} \kappa} + C_2(m/2,d,1/n),
\eal%
where $\hat v_R(\bW(0),\bx) =  \frac 1m \sum\limits_{r=1}^m \indict{\abth{\bbw_r(0)^{\top} \tbx} \le R }$, so that $ \hat v_R(\bW(0),\bbx_i) = \abth{\bar E_{i,R}}/m$.
It follows from (\ref{eq:empirical-loss-convergence-contraction-seg2}) and
(\ref{eq:empirical-loss-convergence-contraction-seg3}) that
$\abth{\bE^{(1)}_i} \le \eta c_{\bu} u_0^2  \pth{ \frac{2R}{\sqrt {2\pi} \kappa}+ C_2(m/2,d,1/n)}$, so that
$\ltwonorm{\bE^{(1)}}$ can be bounded by
\bal\label{eq:empirical-loss-convergence-contraction-E1-bound}
\ltwonorm{\bE^{(1)}} & \le
 \eta c_{\bu} u_0^2  {\sqrt n} \pth{ \frac{2R}{\sqrt {2\pi} \kappa}+ C_2(m/2,d,1/n)} .
\eal
$\bD^{(1)}_i$ on the RHS of  (\ref{eq:empirical-loss-convergence-contraction-seg1})
is expressed by
\bal\label{eq:empirical-loss-convergence-contraction-seg5}
&\bD^{(1)}_i = \frac{1}{\sqrt m} \sum\limits_{r \in E_{i,R}} a_r \pth{ \relu{\bbw_{\bS,r}^{\top}(t+1) \tbbx_i}  -  \relu{\bbw_{\bS,r}^{\top}(t) \tbbx_i} } \nonumber \\
&=  \frac{1}{\sqrt m} \sum\limits_{r \in E_{i,R}} a_r \indict{\bbw_{\bS,r}(t)^{\top} \tbbx_i \ge 0} \pth{ \bbw_{\bS,r}(t+1)  -  \bbw_{\bS,r}(t) }^\top \tbbx_i   \nonumber \\
&=  \frac{1}{\sqrt m} \sum\limits_{r=1}^m a_r \indict{\bbw_{\bS,r}(t)^{\top} \tbbx_i \ge 0} \pth{ -\frac{\eta}{n} \bth{\bZ_{\bS}(t)}_{[(r-1)(d+1)+1:r(d+1)]}
\pth{\hat \by(t) - \by}
 }^\top \tbbx_i \nonumber \\
&\phantom{=}{+}  \frac{1}{\sqrt m} \sum\limits_{r \in \bar E_{i,R} } a_r \indict{\bbw_{\bS,r}(t)^{\top} \tbbx_i \ge 0} \pth{\frac{\eta}{n}
\bth{\bZ_{\bS}(t)}_{[(r-1)(d+1)+1:r(d+1)]}
\pth{\hat \by(t) - \by}
 }^\top \tbbx_i \nonumber \\
&=\underbrace{-\frac{\eta}{n} \bth{\bH(t)}_i \pth{\hat \by(t) - \by}}_{\defeq \bD^{(2)}_i} {+} \underbrace{\frac{1}{\sqrt m} \sum\limits_{r \in \bar E_{i,R} } a_r \indict{\bbw_{\bS,r}(t)^{\top} \tbbx_i \ge 0} \pth{\frac{\eta}{n} \bth{\bZ_{\bS}(t)}_{[(r-1)(d+1)+1:r(d+1)]}
\pth{\hat \by(t) - \by}
 }^\top \tbbx_i }_{\defeq \bE^{(2)}_i}   \nonumber \\
&=  \bD^{(2)}_i + \bE^{(2)}_i,
\eal%
where $\bH(t) \in \RR^{n \times n}$ is a matrix specified by
\bals
\bH_{pq}(t) = \frac{\tbbx_p^\top \tbbx_q}{m} \sum_{r=1}^{m} \indict{\bbw_{\bS,r}(t)^\top \tbbx_p \ge 0} \indict{\bbw_{\bS,r}(t)^\top \tbbx_q \ge 0}, \quad
\forall \,p \in [n], q \in [n].
\eals

Let $\bD^{(2)}, \bE^{(2)} \in \RR^n$ be a vector with their $i$-the element being $\bD^{(2)}_i$ and
$\bE^{(2)}_i$ defined on the RHS of
(\ref{eq:empirical-loss-convergence-contraction-seg5}). $\bE^{(2)}$ can be expressed by $\bE^{(2)} = \frac{\eta}{n} \tilde \bE^{(2)}   \pth{\hat \by(t) - \by}$ with $\tilde \bE^{(2)} \in \RR^{n \times n}$ and

\bals
\tilde \bE^{(2)}_{pq} = \frac{1}{m} \sum\limits_{r \in \bar E_{i,R}} \indict{\bbw_{\bS,r}(t)^{\top} \tbbx_p \ge 0} \indict{\bbw_{\bS,r}(t)^{\top} \tbbx_q \ge 0} \tbbx_{q}^\top \tbbx_p
\le \frac {u_0^2}m \sum\limits_{r \in \bar E_{i,R}} 1 = u_0^2 \cdot \frac{\abth{\bar E_{i,R}}}{m}
\eals
for all $p \in [n], q \in [n]$.  The spectral norm of $\tilde \bE^{(2)}$ is bounded by
\bal\label{eq:empirical-loss-convergence-contraction-seg6}
\ltwonorm{\tilde \bE^{(2)}} \le \fnorm{\tilde \bE^{(2)}} \le n u_0^2 \frac{\abth{\bar E_{i,R}}}{m}
\stackrel{\circled{1}}{\le} n u_0^2 \pth{ \frac{2R}{\sqrt {2\pi} \kappa}+
C_2(m/2,d,1/n)},
\eal%
where $\circled{1}$ follows from (\ref{eq:empirical-loss-convergence-contraction-seg3}).
It follows from (\ref{eq:empirical-loss-convergence-contraction-seg6}) that $\ltwonorm{\bE^{(2)}}$ can be bounded by
\bal\label{eq:empirical-loss-convergence-contraction-E2-bound}
\ltwonorm{\bE^{(2)}} &\le \frac{\eta}{n} \ltwonorm{\tilde\bE^{(2)}}
\ltwonorm{\by(t)-\by} \le \eta c_{\bu} u_0^2 {\sqrt n} \pth{\frac{2R}{\sqrt {2\pi} \kappa}+
C_2(m/2,d,1/n)}.
\eal
$\bD^{(2)}_i$ on the RHS of (\ref{eq:empirical-loss-convergence-contraction-seg5}) is expressed by
\bal\label{eq:empirical-loss-convergence-contraction-seg7}
&\bD^{(2)} = -\frac{\eta}{n} \bH(t)\pth{\hat \by(t) - \by} \nonumber
\\ &=\underbrace{-\frac{\eta}{n} \bK  \pth{\hat \by(t) - \by}}_{\defeq \bD^{(3)}} +  \underbrace{\frac{\eta}{n}  \pth{\bK- \bH(0)}    \pth{\hat \by(t) - \by}}_{\defeq \bE^{(3)}} +  \underbrace{\frac{\eta}{n} \pth{\bH(0) - \bH(t)}  \pth{\hat \by(t) - \by}}_{\defeq \bE^{(4)}}=\bD^{(3)} + \bE^{(3)} + \bE^{(4)}.
\eal
On the RHS of  (\ref{eq:empirical-loss-convergence-contraction-seg7}), $\bD^{(3)},\bE^{(3)},\bE^{(4)} \in \RR^n$ are vectors which are analyzed as follows. We have
\bal\label{eq:empirical-loss-convergence-contraction-seg8}
\ltwonorm{\bK- \bH(0)} &\le \norm{\bK- \bH(0)}{F}
\le n u_0^2 C_1(m/2,d,1/n),
\eal%
where the last inequality holds due to $\bW(0) \in \cW_0$. %with probability $1-1/n$ over $\bW(0)$ by Theorem~\ref{theorem:good-random-initialization}.

In order to bound  $\bE^{(4)}$, we first estimate the upper bound for $\abth{ \bH_{ij}(t) - \bH_{ij}(0) }$ for all $i,j \in [n]$.
We note that
\bal\label{eq:empirical-loss-convergence-contraction-seg9-pre}
&\indict{\indict{\bbw_{\bS,r}(t)^\top \tbbx_i \ge 0} \neq \indict{\bw_{r}(0)^\top \tbbx_i \ge 0} } \le \indict{\abth{\bw_{r}(0)^{\top} \bbx_i} \le R} + \indict{\ltwonorm{\bw_{
\bS,r}(t) - \bbw_r(0)} > R}.
\eal%
It follows from (\ref{eq:empirical-loss-convergence-contraction-seg9-pre}) that
\bal\label{eq:empirical-loss-convergence-contraction-seg9}
\abth{ \bH_{ij}(t) - \bH_{ij}(0) } &= \abth{ \frac{\tbbx_i^\top \tbbx_j}{m} \sum_{r=1}^{m} \pth{ \indict{\bbw_{\bS,r}(t)^\top \tbbx_i \ge 0} \indict{\bbw_{\bS,r}(t)^\top  \tbbx_j \ge 0} - \indict{\bw_{r}(0)^\top \tbbx_i \ge 0} \indict{\bw_{r}(0)^\top \tbbx_j \ge 0} }} \nonumber \\
&\le \frac{u_0^2}m \sum_{r=1}^{m} \pth{\indict{\indict{\bbw_{\bS,r}(t)^\top \tbbx_i \ge 0} \neq \indict{\bbw_r(0)^\top \tbbx_i \ge 0}}
+ \indict{\indict{\bbw_{\bS,r}(t)^\top \bbx_j \ge 0} \neq \indict{\bbw_r(0)^\top \tbbx_j \ge 0}}} \nonumber \\
&\le \frac{u_0^2}m \sum_{r=1}^{m} \pth{ \indict{\abth{\bbw_r(0)^{\top} \tbbx_i} \le R}  +\indict{\abth{\bbw_r(0)^{\top} \tbbx_j} \le R}  +2  \indict{\ltwonorm{\bw_{
\bS,r}(t) - \bbw_r(0)} > R}   } \nonumber \\
&\le 2u_0^2v_R(\bW(0),\bbx_i) \stackrel{\circled{1}}{\le}
u_0^2 \pth{\frac{4R}{\sqrt {2\pi} \kappa}+ 2C_2(m/2,d,1/n)},
\eal%
where $\circled{1}$ follows from
(\ref{eq:empirical-loss-convergence-contraction-seg3}).

It follows from
(\ref{eq:empirical-loss-convergence-contraction-seg8}) and
 (\ref{eq:empirical-loss-convergence-contraction-seg9})
 that $\ltwonorm{\bE^{(3)}},\ltwonorm{\bE^{(4)}} $ are bounded by
\bal
\ltwonorm{\bE^{(3)}} &\le\frac{\eta}{n}
\ltwonorm{\bK- \bH(0)} \ltwonorm{\hat \by(t) - \by} \le \eta c_{\bu} u_0^2 {\sqrt n} C_1(m/2,d,1/n),
\label{eq:empirical-loss-convergence-contraction-E3-bound} \\
\ltwonorm{\bE^{(4)}} &\le\frac{\eta}{n}
\ltwonorm{\bH(0) - \bH(t)} \ltwonorm{\hat \by(t) - \by}
\le\eta c_{\bu} u_0^2 {\sqrt n} \pth{\frac{4R}{\sqrt {2\pi} \kappa}+ 2C_2(m/2,d,1/n)} . \label{eq:empirical-loss-convergence-contraction-E4-bound}
\eal

It follows from (\ref{eq:empirical-loss-convergence-contraction-seg5})
and (\ref{eq:empirical-loss-convergence-contraction-seg7}) that
\bal\label{eq:empirical-loss-convergence-contraction-seg12}
\bD^{(1)}_i &=  \bD^{(3)}_i +  \bE^{(2)}_i+\bE^{(3)}_i + \bE^{(4)}_i.
\eal
It then follows from (\ref{eq:empirical-loss-convergence-contraction-seg1}) that
\bal\label{eq:empirical-loss-convergence-contraction-seg13}
\hat \by_i(t+1) - \hat \by_i(t) &= \bD^{(1)}_i + \bE^{(1)}_i
=\bD^{(3)}_i + \underbrace{\bE^{(1)}_i + \bE^{(2)}_i+\bE^{(3)}_i
+ \bE^{(4)}_i}_{\defeq \bE_i}
=-\frac{\eta}{n}\bK\pth{\hat \by(t) - \by}+ \bE_i,
\eal
where $\bE \in \RR^n$ with its $i$-th element being $\bE_i$, and $\bE = \bE^{(1)}
+\bE^{(2)}+\bE^{(3)} + \bE^{(4)}$. It then follows from
(\ref{eq:empirical-loss-convergence-contraction-E1-bound}),
(\ref{eq:empirical-loss-convergence-contraction-E2-bound}),
(\ref{eq:empirical-loss-convergence-contraction-E3-bound}),
and (\ref{eq:empirical-loss-convergence-contraction-E4-bound})
that
\bal\label{eq:empirical-loss-convergence-contraction-E-bound}
&\ltwonorm {\bE} \le \eta c_{\bu} u_0^2 {\sqrt n}
\pth{ 4\pth{\frac{2R}{\sqrt {2\pi} \kappa}+
C_2(m/2,d,1/n)} + C_1(m/2,d,1/n)}.
\eal
Finally, (\ref{eq:empirical-loss-convergence-contraction-seg13})
can be rewritten as
\bals%\label{eq:empirical-loss-convergence-contraction-seg14}
\hat \by(t+1) - \by
&=\pth{\bI-\frac{\eta}{n} \bK}\pth{\hat \by(t) - \by} + \bE(t+1),
\eals
which proves (\ref{eq:empirical-loss-convergence-contraction})
with the upper bound for $\ltwonorm {\bE} $ in
(\ref{eq:empirical-loss-convergence-contraction-E-bound}).

\end{proof}

\begin{lemma}\label{lemma:weight-vector-movement}
Suppose that $t \in [0\relcolon T-1]$ for $T \ge 1$, and $\ltwonorm{\hat \by(t') - \by} \le {\sqrt n} c_{\bu} $ holds for all $0 \le t' \le t$. Then
\bal\label{eq:R}
\ltwonorm{\bbw_{\bS,r}(t') - \bbw_r(0)} \le R, \quad \forall\, 0 \le t' \le t+1.
\eal
\end{lemma}
\begin{proof}
Let $\bth{\bZ_{\bS}(t)}_{[(r-1)(d+1)+1:r(d+1)]}$ denote the submatrix of $\bZ_{\bS}(t)$ formed by the the rows of $\bZ_{\bQ}(t)$ with row indices in $[(r-1)(d+1)+1:r(d+1)]$.
By the GD update rule we have for every $t'' \in [0 \colon T-1]$ that
\bal\label{eq:weight-vector-movement-seg1-pre}
&\bbw_{\bS,r}(t''+1) - \bbw_{\bS,r}(t'') = -\frac{\eta}{n} \bth{\bZ_{\bS}(t'')}_{[(r-1)(d+1)+1:r(d+1)]} \pth{\hat \by(t'') - \by},
\eal%
We have $\ltwonorm{\bth{\bZ_{\bS}(t'')}_{[(r-1)(d+1)+1:r(d+1)]}} \le u_0 \sqrt{n /m}$.
It then follows from (\ref{eq:weight-vector-movement-seg1-pre}) that
\bal\label{eq:weight-vector-movement-seg1}
\ltwonorm{\bbw_{\bS,r}(t''+1) - \bbw_{\bS,r}(t'')}
&\le \frac{\eta}{n}
\ltwonorm{\bth{\bZ_{\bS}(t'')}_{[(r-1)(d+1)+1:r(d+1)]}}\ltwonorm{\hat \by(t'')-\by}
\le  \frac{\eta c_{\bu}u_0}{\sqrt m}, \, \forall t'' \in [0\relcolon t].
\eal%
Note that (\ref{eq:R}) trivially holds for $t'=0$. For $t' \in [1,t+1]$, it follows from
 (\ref{eq:weight-vector-movement-seg1}) that
\bal\label{eq:weight-vector-movement-proof}
\ltwonorm{ \bbw_{\bS,r}(t') - \bbw_r(0) }
& \le \sum_{t''=0}^{t'-1} \ltwonorm{\bbw_{\bS,r}(t''+1) - \bbw_{\bS,r}(t'')}  \le \frac{\eta c_{\bu} u_0 T }{\sqrt m} =R,
\eal%
which completes the proof.
\end{proof}

\begin{lemma}\label{lemma:bounded-Linfty-vt-sum-et}
Let
$h_t(\cdot) = \sum_{t'=0}^{t-1} h(\cdot,t')$ for $t \in [T]$, $T \le
\hat T$ where
\bals
h(\cdot,t') &= v(\cdot,t') + \hat e(\cdot,t'), \\
v(\cdot,t')  &= -\frac{\eta}{n} \sum_{j=1}^n
K(\cdot, \bbx_j) \bv_j(t') , \\
\hat e(\cdot,t') &= -\frac{\eta}{n}
\sum\limits_{j=1}^n  K(\cdot, \bbx_j) \be_j(t'),
\eals
where $\bv(t') \in \cV_{t'}$,
$\be(t') \in \cE_{t',\tau}$ for all $0 \le t' \le t-1$.
Suppose that $\tau \le \min\set{\Theta(1/(\eta u_0 T)),1}$,
then with probability at least $1 - \exp\pth{-\Theta(n \hat \eps_n^2)}$
over the random noise $\bw$,
\bal\label{eq:bounded-h}
\norm{h_t}{\cH_K} \le B_h = \mu_0 +1+ {\sqrt 2},
\eal
and $B_h$ is also defined in (\ref{eq:B_h}).
\end{lemma}
\begin{proof}
%[\textbf{\textup{Proof of
%Lemma~\ref{lemma:bounded-Linfty-vt-sum-et}}}]
We have $\by = f^*(\bS) + \bw$,
$\bv(t) = -\pth{\bI- \eta \bK_n }^t f^*(\bS)$,
$\be(t) = \bbe_1(t) + \bbe_2(t)$ with
$\bbe_1(t) = -\pth{\bI-\eta\bK_n}^t \bw$,
$\ltwonorm{\bbe_2(t)} \le {\sqrt n} \tau$.
We define
\bal\label{eq:bounded-Linfty-vt-sum-et-hat-e1-hat-e2}
\hat e_1(\cdot,t') \defeq- \frac{\eta}{n}
\sum\limits_{j=1}^n  K(\bbx_j,\bx) \bth{\bbe_1(t')}_j,
\quad
\hat e_2(\cdot,t') \defeq- \frac{\eta}{n}
\sum\limits_{j=1}^n  K(\bbx_j,\bx) \bth{\bbe_2(t')}_j,
\eal

Let $\bSigma$ be the diagonal matrix
containing eigenvalues of $\bK_n$, we then have
\bal\label{eq:bounded-Linfty-vt-sum-seg1}
\sum_{t'=0}^{t-1} v(\bx,t') &=\frac{\eta}{n} \sum\limits_{j=1}^n  \sum_{t'=0}^{t-1}
\bth{\pth{\bI- \eta \bK_n }^{t'} f^*(\bS)}_j K(\bbx_j,\bx) \nonumber \\
&=\frac{\eta}{n} \sum\limits_{j=1}^n \sum_{t'=0}^{t-1}
\bth{\bU \pth{\bI-\eta \bSigma }^{t'} {\bU}^{\top} f^*(\bS)}_j K(\bbx_j,\bx).
\eal
It follows from (\ref{eq:bounded-Linfty-vt-sum-seg1}) that
\bal\label{eq:bounded-Linfty-vt-sum-seg2}
\norm{\sum_{t'=0}^{t-1} v(\cdot,t')}{\cH_K}^2
&= \frac{\eta^2}{n^2} f^*(\bS)^{\top}
\bU \sum_{t'=0}^{t-1} \pth{\bI-\eta \bSigma}^{t'} {\bU}^{\top}
\bK \bU \sum_{t'=0}^{t-1} \pth{\bI-\eta \bSigma}^{t'}
{\bU}^{\top} f^*(\bS) \nonumber \\
&= \frac 1n \ltwonorm{\eta\pth{\bK_n}^{1/2} \bU \sum_{t'=0}^{t-1} \pth{\bI-\eta \bSigma}^{t'} {\bU}^{\top} f^*(\bS)}^2 \nonumber \\
&\le \frac 1n \sum\limits_{i=1}^{n} \frac{\pth{1-
\pth{1-\eta \hlambda_i }^t}^2}
{\hlambda_i}\bth{{\bU}^{\top} f^*(\bS)}_i^2
\le \mu_0^2,
\eal
where the last inequality
 follows from Lemma~\ref{lemma:bounded-Ut-f-in-RKHS}.

Similarly, we have
\bal\label{eq:bounded-Linfty-hat-et-sum-1}
&\norm{\sum_{t'=0}^{t-1} \hat e_1(\cdot,t')}{\cH_K}^2
\le \frac 1n \sum\limits_{i=1}^{n} \frac{\pth{1-
\pth{1-\eta \hlambda_i }^t}^2}
{\hlambda_i}\bth{{\bU}^{\top} \bw}_i^2.
\eal

It then follows from the argument in the proof of~\cite[Lemma 9]{RaskuttiWY14-early-stopping-kernel-regression}
that the RHS of (\ref{eq:bounded-Linfty-hat-et-sum-1}) is bounded with high probability. We define a diagonal matrix $\bR \in \RR^{n \times n}$
with $\bR_{ii} = \big(1-(1-\eta \hlambda_i )^t\big)^2/\hlambda_i$ for $i \in [n]$. Then the RHS of (\ref{eq:bounded-Linfty-hat-et-sum-1}) is
$1/n \cdot \tr{\bU \bR \bU^{\top} \bw \bw^{\top} }$.
%$\tilde \bw^{\top} \bR \tilde \bw$ with $\tilde \bw \defeq {\bU}^{\top} \bw$ where the elements of $\tilde \bw \in \RR^n$ are  i.i.d. Gaussian random variables distributed according to $\cN(0,\sigma_0^2)$, since $\bU$ is an orthogonal matrix.
 It follows from~\cite{quadratic-tail-bound-Wright1973}
that
\bal\label{eq:bounded-Linfty-hat-et-sum-2}
\Prob{\frac {\tr{\bU \bR \bU^{\top} \bw \bw^{\top} }}n  -
\Expect{}{\frac {\tr{\bU \bR \bU^{\top} \bw \bw^{\top} }}n } \ge u}
\le \exp\pth{-c \min\set{nu/\ltwonorm{\bR},n^2u^2/\fnorm{\bR}^2}}
\eal
for all $u > 0$, and $c$ is a  positive constant depending on $\sigma_0$. Recall that
$\eta_t = \eta t$ for all $t \ge 0$, we have
\bal\label{eq:bounded-Linfty-hat-et-sum-3}
\Expect{}{\frac {\tr{\bU \bR \bU^{\top} \bw \bw^{\top} }}n }
&\stackrel{\circled{1}}{\le} \frac {\sigma_0^2}n \sum\limits_{i=1}^n
\frac{\pth{1-\pth{1-\eta \hlambda_i }^t}^2}{\hlambda_i}
\stackrel{\circled{2}}{\le}
\frac {\sigma_0^2}n \sum\limits_{i=1}^n
\min\set{\frac{1}{\hlambda_i},\eta_t^2 \hlambda_i}
\nonumber \\
&\le
\frac {{\sigma_0^2}\eta_t}n \sum\limits_{i=1}^n
\min\set{\frac{1}{\eta_t\hlambda_i},\eta_t \hlambda_i}
\stackrel{\circled{3}}{\le}
\frac {{\sigma_0^2}\eta_t}n \sum\limits_{i=1}^n
\min\set{1,\eta_t \hlambda_i} \nonumber \\
&= \frac {{\sigma_0^2}\eta_t^2}n \sum\limits_{i=1}^n
\min\set{\eta_t^{-1},\hlambda_i} = {{\sigma_0^2}\eta_t^2} \hat R_K^2(\sqrt{{1}/{\eta_t}}) \le
1.
\eal
Here $\circled{1}$ follows from the Von Neumann’s trace inequality and the fact that $\Expect{}{w_i^2} \le \sigma_0^2$ for all $i \in [n]$.
$\circled{2}$ follows from the fact that
$(1-\eta \hlambda_i )^t \ge \max\set{0,1-t\eta \hlambda_i}$,
and $\circled{3}$ follows from
$\min\set{a,b} \le \sqrt{ab}$ for any nonnegative numbers $a,b$.
Because $t \le T \le \hat T$, we have
$\hat R_K(\sqrt{{1}/{\eta_t}}) \le 1/(\sigma_0 \eta_t)$, so the last inequality holds.

Moreover, we have the upper bounds for $\ltwonorm{\bR}$ and $\fnorm{\bR}$
as follows. First, we have
\bal\label{eq:bounded-Linfty-hat-et-sum-4}
\ltwonorm{\bR} &\le \max_{i \in [n] }\frac{\pth{1-\pth{1-\eta \hlambda_i }^t}^2}{\hlambda_i} \le \min\set{\frac{1}{\hlambda_i},\eta_t^2 \hlambda_i}
\le \eta_t.
\eal

We also have
\bal\label{eq:bounded-Linfty-hat-et-sum-5}
\frac 1n \fnorm{\bR}^2 &=  \frac 1n
\sum\limits_{i=1}^n
\frac{\pth{1-\pth{1-\eta \hlambda_i }^t}^4}{\hlambda_i^2}
\le \frac {\eta_t^3}n \sum\limits_{i=1}^n
\min\set{\frac{1}{\eta_t^3 \hlambda_i^2},\eta_t \hlambda_i^2} \nonumber \\
&\stackrel{\circled{3}}{\le} \frac {\eta_t^3}n \sum\limits_{i=1}^n
\min\set{\hlambda_i,\frac{1}{\eta_t}}
=\eta_t^3 \hat R_K^2(\sqrt{{1}/{\eta_t}})\le
\frac {\eta_t}{\sigma_0^2},
\eal
where $\circled{3}$ follows from
\bals
\min\set{\frac{1}{\eta_t^3 \hlambda_i^2},\eta_t \hlambda_i^2}
= \hlambda_i
\min\set{\frac{1}{\eta_t^3 \hlambda_i^3},\eta_t \hlambda_i}
\le  \hlambda_i.
\eals
Combining (\ref{eq:bounded-Linfty-hat-et-sum-2})-(\ref{eq:bounded-Linfty-hat-et-sum-5}) with $u=1$ in
(\ref{eq:bounded-Linfty-hat-et-sum-2}), we have
\bals
\Prob{1/n \cdot \tilde \bw^{\top} \bR \tilde \bw -
\Expect{}{1/n \cdot \tilde \bw^{\top} \bR \tilde \bw} \ge 1}
&\le \exp\pth{-c \min\set{n/\eta_t,n \sigma_0^2/\eta_t}} \\
&\le \exp\pth{-nc'/\eta_t} \le
\exp\pth{-c'n\hat\eps_n^2}
\eals
where $\tilde \bw = \bU^{\top}\bw$, $c'  = c\min\set{1,\sigma_0^2}$, and the last inequality is due
to the fact that $1/\eta_t \ge \hat\eps_n^2$ since
$t \le T \le \hat T$.
It then follows from (\ref{eq:bounded-Linfty-hat-et-sum-1}) that with probability at least $1-
\exp\pth{- \Theta(n\hat\eps_n^2)}$,
$\norm{\sum_{t'=0}^{t-1} \hat e_1(\cdot,t')}{\cH_{K}}^2 \le 2$.

We now find the upper bound for $\norm{\sum_{t'=0}^{t-1} \hat e_2(\cdot,t')}{\cH_K}$. We have
\bals
\norm{\hat e_2(\cdot,t')}{\cH_K}^2
&\le \frac{\eta^2}{n^2} \bbe_2^{\top}(t')\bK\bbe_2(t')
\le \eta^2 \hlambda_1 \tau^2,
\eals
so that
\bal\label{eq:bounded-Linfty-hat-et-sum-6}
&\norm{\sum_{t'=0}^{t-1} \hat e_2(\cdot,t')}{\cH_K}
\le \sum_{t'=0}^{t-1} \norm{\hat e_2(\cdot,t')}{\cH_K}
\le  T \eta \sqrt{\hlambda_1} \tau \le 1,
\eal
if $\tau \lsim 1/(\eta u_0 T) $ since $\hat \lambda_1 \in (0, u_0^2/2)$.

Finally, we have
\bals
\norm{h_t}{\cH_K}  &\le \norm{\sum_{t'=0}^{t-1} \hat v(\cdot,t')}{\cH_K}
+\norm{\sum_{t'=0}^{t-1} \hat e_1(\cdot,t')}{\cH_K} + \norm{\sum_{t'=0}^{t-1} \hat e_2(\cdot,t')}{\cH_K} \le \mu_0 +1+ {\sqrt 2} = B_h.
\eals
\end{proof}

\begin{lemma}[In the proof of
{\cite[Lemma 8]{RaskuttiWY14-early-stopping-kernel-regression}}]
\label{lemma:bounded-Ut-f-in-RKHS}
For any $f \in \cH_{K}(\mu_0)$, we have
\bal\label{eq:b.ounded-Ut-f-in-RKHS}
\frac 1n \sum_{i=1}^n \frac{\bth{\bU^{\top}f(\bS')}_i^2}{\hat \lambda_i} \le \mu_0^2.
\eal

\end{lemma}

\begin{lemma}
\label{lemma:auxiliary-lemma-1}
For any positive real number $a \in (0,1)$ and natural number $t$,
we have
\bal\label{eq:auxiliary-lemma-1}
(1-a)^t \le e^{-ta} \le \frac{1}{eta}.
\eal
\end{lemma}
\begin{proof}
The result follows from the facts that
$\log(1-a) \le a$ for $a \in (0,1)$ and $\sup_{u \in \RR}
ue^{-u} \le 1/e$.
\end{proof}

%\begin{lemma}\label{lemma:gap-spectrum-Tk-Tn}
%(\cite[Proposition 10]{RosascoBV10-integral-operator})
%With probability $1-\delta$ over the training data $\bS$, for all $j \in [n]$,
%\bal\label{eq:gap-spectrum-Tk-Tn}
%&
%\abth{\lambda_j - \hlambda_j} \le \sqrt{\frac{2\log{\frac{2}{\delta}}}{n}}.
%\eal%
%\end{lemma}

\begin{lemma}\label{lemma:hat-eps-eps-relation}
Suppose $\sup_{\bx \in \cX} K(\bx,\bx) \le
\tau_0^2$ for some positive
number $\tau_0$ and
$\inf_{\bx \in \cX} K(\bx,\bx) \gsim 1$ where $K$ is a PD kernel defined over
$\cX \times \cX$, and
 $n \gsim 1/\lambda_1$. Let
 the critical population rate and the critical empirical radius of $K$
be $\eps_{n}$ and $\hat \eps_{n}$, respectively. Then with probability at least $1-2\exp(-\Theta(n\eps_n^2))$,
\bal
{\eps_n^2} \lsim \hat \eps_n^2. \label{eq:eps-n-hat-eps-n-bound}
\eal
Furthermore, with probability at least
$1-2\exp(-\Theta(n\eps_n^2))$,
\bal
{\hat \eps_n^2} \lsim \eps_n^2. \label{eq:hat-eps-n-eps-n-bound}
\eal
where $E_{K,n}$ is defined in (\ref{eq:E-K-n}).
\end{lemma}
\begin{remark*}
$K$ in Lemma~\ref{lemma:hat-eps-eps-relation} can be a general PDS
kernel not limited to the NTK in (\ref{eq:kernel-two-layer}).
Lemma~\ref{lemma:hat-eps-eps-relation} shows that with probability at least $1-4\exp(-\Theta(n\eps_n^2))$,
$\eps_n^2 \asymp \hat \eps_n^2$, which is also a fact used in kernel complexity or local Rademacher based analysis for kernel regression in the statistical learning literature.
\end{remark*}
\begin{proof}
We define the function classes
\bal\label{eq:F-t-hat-Ft-def}
\cF_{K,t} \defeq \set{f \in \cH_{K} \colon \norm{f}{\cH_{K}} \le 1, \norm{f}{L^2} \le t}, \quad
\hat \cF_{K,t} \defeq \set{f \in \cH_{K} \colon \norm{f}{\cH_{K}} \le 1, \norm{f}{n} \le t},
\eal
where
$\norm{f}{n}^2 \defeq 1/n \cdot \sum_{i=1}^n f^2(\bbx_i)$.
Let $\cR_{K}(t), \hat \cR_{K}(t)$ be the Rademacher complexity and empirical Rademacher complexity of $\cF_{K,t}$ and $\hat \cF_{K,t}$, that is,
\bal
\label{eq:F-t-RC}
\cR_{K}(t) \defeq \cfrakR\pth{\cF_{K,t}} =\Expect{\set{\bbx_i},\set{\sigma_i}}{\sup_{f \in \cF_{K,t}} \frac 1n \sum\limits_{i=1}^n \sigma_i f(\bbx_i)},
\hat \cR_{K}(t) \defeq \Expect{\bsigma}{\sup_{f \in \hat \cF_{K,t}} \frac 1n \sum\limits_{i=1}^n \sigma_i f(\bbx_i)},
\eal
and we will also write $\cR_{K}(t) =
\Expect{}{\sup_{f \in \cF_{K,t}} \frac 1n \sum\limits_{i=1}^n \sigma_i f(\bbx_i),
}$ for simplicity of notations.

It follows from Lemma~\ref{lemma:kernel-complexity-RC-relation} there are universal positive constants $c_{\ell}$ and $C_u$ with $0 < c_{\ell} < C_u$ such that when  $n \gsim 1/\lambda_1$ and $t^2 \gsim 1/n$, we have
\bal\label{eq:kernel-complexity-Rademacher-complexity-relation}
c_{\ell}R_{K}(t) \le \cR_{K}(t) \le C_u R_{K}(t),
\quad c_{\ell} \hat R_{K}(t) \le
\hat \cR_{K}(t) \le C_u \hat R_{K}(t).
\eal
When $f \in \cF_{K,t}$, $\supnorm{f} \le \tau_0$ since $\sup_{\bx \in \cX}
K(\bx,\bx) \le \tau_0^2$.
It follows from Lemma~\ref{lemma:gn-g-LRC-bound} that with probability at least $1-\exp(-n\eps_n^2)$,
\bal\label{eq:hat-eps-eps-relation-seg1}
\cF_{K,t} \subseteq
\set{f \in \cH_{K} \colon \norm{f}{\cH_{K}} \le 1, \norm{f}{n} \le \sqrt{c_2 t^2 + E_{K,n}} } \defeq \hat \cF_{K,\sqrt{c_2 t^2 +E_{K,n} }},
\eal
where $E_{K,n}$ is defined in (\ref{eq:E-K-n}).
Moreover, by the relation between Rademacher
complexity and its empirical version in {\cite[Lemma A.4]{bartlett2005}}, for every $x > 0$, with probability at least $1-\exp(-x)$,
\bal\label{eq:hat-eps-eps-relation-seg2}
\Expect{}{\sup_{f \in \hat \cF_{K,\sqrt{c_2 t^2 + E_{K,n}}}} \frac 1n \sum\limits_{i=1}^n \sigma_i f(\bbx_i)}\le 2 \Expect{\bsigma}{\sup_{f \in\hat \cF_{K,\sqrt{c_2 t^2 + E_{K,n}}}} \frac 1n \sum\limits_{i=1}^n \sigma_i f(\bbx_i)}
+ \frac{2 \tau_0 x}{n}.
\eal
As a result,
\bals
\cR_{K}(t)
&\stackrel{\circled{1}}{\le}
\Expect{}{\sup_{f \in \hat \cF_{K,\sqrt{c_2 t^2 + E_{K,n}}}} \frac 1n \sum\limits_{i=1}^n \sigma_i f(\bbx_i)}
\stackrel{\circled{2}}{\le} 2 \Expect{\bsigma}{\sup_{f \in
\hat \cF_{K,\sqrt{c_2 t^2 + E_{K,n}}}} \frac 1n \sum\limits_{i=1}^n \sigma_i f(\bbx_i)} + \frac{2\tau_0x}{n} \nonumber \\
&=2 \hat \cR_{K}(\sqrt{c_2 t^2 + E_{K,n}})+ \frac{2x}{n}.
\eals
Here $\circled{1}$ follows from (\ref{eq:hat-eps-eps-relation-seg1}), and
$\circled{2}$ follows from (\ref{eq:hat-eps-eps-relation-seg2}).
It follows from (\ref{eq:kernel-complexity-Rademacher-complexity-relation})
and the above inequality that
\bals
c_{\ell}/\sigma_0 \cdot \sigma_0 R_{K}(t) \le 2C_u /\sigma_0 \cdot \sigma_0 \hat R_{K}(\sqrt{c_2 t^2 + E_{K,n}})+ \frac{2\tau_0x}{n}, \forall t^2 \gsim 1/n.
\eals
We rewrite $R_{K}(t)$ as a function  of $r = t^2$ as $R_{K}(t) = F_{K}(r)$. Similarly, $\hat R_{K}(t) = \hat F_{K}(r)$ with $r = t^2$. Then we have
\bal\label{eq:hat-eps-eps-relation-seg3}
\sigma_0 F_{K}(r) \le 2C_u/c_{\ell} \cdot  \sigma_0 \hat F_{K}(c_2 r+ E_{K,n}) + \frac{2\sigma_0 \tau_0x}{n c_{\ell}} \defeq G(r),
\forall r \gsim 1/n.
\eal
It can be verified that $G(r)$ is a sub-root function, and let $r^*_G$ be the fixed point of $G$. Let $x \gsim c_{\ell}/(2\sigma_0\tau_0)$,
then $r^*_G \gsim 1/n$. Moreover, $\sigma_0 F_{K}(r)$
and $\sigma_0 \hat F_{K}(r)$ are sub-root functions, and they have fixed points $\eps_n^2$ and $\hat \eps_n^2$, respectively. Set $r = r^*_G \gsim 1/n$ in (\ref{eq:hat-eps-eps-relation-seg3}), we have
\bals
\sigma_0 F_{K}(r^*_G) \le r^*_G,
\eals
and it follows from the above inequality and {\cite[Lemma 3.2]{bartlett2005}} that $\eps_n^2 \le r^*_G$. Since $c_2 > 1$, it then follows from the properties about the fixed point of a sub-root function in Lemma~\ref{lemma:sub-root-fix-point-properties}
that
\bals
\eps_n^2 \le r^*_G \le \Theta(1) \cdot
\pth{c_2 \hat \eps_n^2 + \frac{2E_{K,n}}{c_2}}
+ \frac{4\sigma_0 \tau_0x}{n c_{\ell}}.
\eals
Let $x = c' c_{\ell} n\eps_n^2/(\sigma_0\tau_0)$ where $c' > 0$ is a positive constant. Since
$E_{K,n} \lsim  \eps_n^2$, by choosing properly small $c'$
and properly large $c_2$, we have
\bals
\eps_n^2 \lsim \hat \eps_n^2,
\eals
and (\ref{eq:eps-n-hat-eps-n-bound}) is proved. We remark
that $n\eps_n^2 \gsim 1$ due to the definition of the fixed point of the kernel complexity, so choosing $x = c' n\eps_n^2/c_{\ell}$
also satisfies $x \gsim c_{\ell}/(2\sigma_0\tau_0)$.

We now prove (\ref{eq:hat-eps-n-eps-n-bound}).
It follows from Lemma~\ref{lemma:gn-g-LRC-bound} again that
with probability at least $1-\exp(-n\eps_n^2)$,
\bal\label{eq:hat-eps-eps-relation-seg4}
\hat \cF_{K,t} \subseteq
\set{f \in \cH_{K} \colon \norm{f}{\cH_{K}} \le 1, \norm{f}{L^2} \le\sqrt{c_2 t^2 + E_{K,n}} }
= \cF_{K,\sqrt{c_2 t^2 + E_{K,n}}}.
\eal
It follows from {\cite[Lemma A.4]{bartlett2005}} again that for every $x > 0$, with probability at least $1-\exp(-x)$,
\bal\label{eq:hat-eps-eps-relation-seg5}
\Expect{\bsigma}{\sup_{f \in \cF_{K,\sqrt{c_2 t^2 + E_{K,n}}}} \frac 1n \sum\limits_{i=1}^n \sigma_i f(\bbx_i)}\le 2 \Expect{}{\sup_{f \in \cF_{K,\sqrt{c_2 t^2 + E_{K,n}}}} \frac 1n \sum\limits_{i=1}^n \sigma_i f(\bbx_i)} + \frac{5\tau_0x}{6n}.
\eal
As a result, we have
\bals
&c_{\ell} \hat R_{K}(t) \le \hat \cR_{K}(t)
\stackrel{\circled{1}}{\le}
\Expect{\bsigma}{\sup_{f \in \cF_{K,\sqrt{c_2 t^2 + E_{K,n}}}} \frac 1n \sum\limits_{i=1}^n \sigma_i f(\bbx_i)} \stackrel{\circled{2}}{\le} 2 \Expect{}{\sup_{f \in  \cF_{K,\sqrt{c_2 t^2 + E_{K,n}}}} \frac 1n \sum\limits_{i=1}^n \sigma_i f(\bbx_i)} + \frac{5\tau_0x}{6n} \nonumber \\
&=2 \cR_{K}(\sqrt{c_2 t^2 + E_{K,n}})+ \frac{5\tau_0x}{6n} \le
2C_u R_{K}(\sqrt{c_2 t^2 + E_{K,n}})+ \frac{5\tau_0x}{6n},
\quad \forall t^2 \gsim 1/n,
\eals
where $\circled{1}$ follows from (\ref{eq:hat-eps-eps-relation-seg4}), and
$\circled{2}$ follows from (\ref{eq:hat-eps-eps-relation-seg5}).
Using a similar argument for the proof of (\ref{eq:eps-n-hat-eps-n-bound}), we have
\bals
\hat \eps_n^2 \le  \Theta(1) \cdot   \pth{c_2 \eps_n^2+ \frac{2E_{K,n} }{c_2}}  + \frac{5\sigma_0 \tau_0x}{3nc_{\ell}},
\eals
and (\ref{eq:hat-eps-n-eps-n-bound}) is proved by the above inequality with $x = c'' c_{\ell}n\eps_n^2/(\sigma_0\tau_0)$ where
 $c''$ is a positive constant.

%Repeating the above arguments with $K$ replaced by $K$, we obtain (\ref{eq:bound-epsn-hat-epsn}).

\end{proof}

\begin{lemma}\label{lemma:gn-g-LRC-bound}
Let $K$ be a PD kernel defined over $\cX \times \cX$ with $\sup_{\bx \in \cX} K(\bx,\bx) \le
\tau_0^2$ for some positive number $\tau_0 \gsim 1$. Define
\bal\label{eq:E-K-n}
E_{K,n} \defeq \min\set{\Theta(\tau_0^2) \eps^{\textup{(eig)}}_{K} + \Theta(\tau_0^2)\eps_n^2,\Theta(\tau_0^4) \eps_n^2},
\eal
where $\eps^{\textup{(eig)}}_{K}$ is defined in (\ref{eq:eps-K-eig-def}).
Then with probability at least $1-\exp\pth{-n \eps_n^2}$,
\bal
\norm{g}{L^2}^2 &\le c_2 \norm{g}{n}^2 +E_{K,n},
\forall g \in \cH_K(1) \label{eq:g-le-gn-LRC-bound}.
\eal
Similarly, with probability at least $1-\exp\pth{-n \eps_n^2}$,
\bal
\norm{g}{n}^2 &\le c_2 \norm{g}{L^2}^2 + E_{K,n}, \forall g \in \cH_K(1). \label{eq:gn-le-g-LRC-bound-fixed-point}
\eal
Here $c_2 > 1$ is a  positive constant.
\end{lemma}

\begin{proof}
The results follow from Theorem~\ref{theorem:Talagrand-inequality} and repeating the argument similar to that in the proof of
Theorem~\ref{theorem:LRC-population-NN-eigenvalue} (with $B_0 = \Theta(\tau_0)+1$).
\end{proof}

\begin{lemma}\label{lemma:sub-root-fix-point-properties}
Suppose $\psi \colon [0,\infty) \to [0,\infty)$ is a sub-root function with the unique fixed point $r^*$. Then the following properties hold.

\begin{itemize}[leftmargin=8pt]
\item[(1)] Let $a \ge 0$, then $\psi(r) + a$ as a function of $r$ is also a sub-root function with fixed point $r^*_a$, and
$r^* \le r^*_a \le r^* + 2a$.
\item[(2)] Let $b \ge 1$, $c \ge 0$ then $\psi(br+c)$ as a function of $r$ is also a sub-root function with fixed point $r^*_b$, and
$r^*_b \le br^* +2c/b$.
\item[(3)] Let $b \ge 1$, then $\psi_b(r) = b\psi(r)$ is also a sub-root function with fixed point $r^*_b$, and
$r^*_b \le b^2r^* $.
\end{itemize}
\end{lemma}
\begin{proof}
(1). Let $\psi_a(r) = \psi(r) + a$. It can be verified that $\psi_a(r)$ is a sub-root function because its nonnegative, nondecreasing and $\psi_a(r)/\sqrt{r}$ is nonincreasing. It follows from
{\cite[Lemma 3.2]{bartlett2005}} that $\psi_a$ has unique fixed point denoted
by $r^*_a$. Because $r^* = \psi(r^*) \le \psi(r^*) + a = \psi_a(r^*)$, it follows from {\cite[Lemma 3.2]{bartlett2005}} that $r^* \le  r^*_a $. Furthermore,
since
\bals
\psi_a(r^*+2a) = \psi(r^*+2a) + a \le \psi(r^*) \sqrt{\frac{r^*+2a}{r^*}} + a \le  \sqrt{r^*(r^*+2a)} + a \le r^*+2a,
\eals
it follows from {\cite[Lemma 3.2]{bartlett2005}} again that
$r^*_a \le r^* + 2a$.

(2). Let $\psi_b(r) = \psi(br+c)$. It can be verified that $\psi_b(r)$ a sub-root function by checking the definition. Also, we have
$\psi(b(br^* +2c/b)+c)/\sqrt{b(br^* +2c/b)+c} \le \psi(r^*)/\sqrt{r^*}$.
It follows that
\bals
\psi_b\pth{br^* +\frac {2c}b} = \psi\pth{b\pth{br^* + \frac {2c}b}+c} &\le b \sqrt{\pth{r^*+\frac {3c}{b^2}}r^*} \le b \pth{r^* + \frac {3c}{2b^2}} \le br^* +\frac {2c}b.
\eals
 Then it follows from {\cite[Lemma 3.2]{bartlett2005}} that $r^*_b \le br^*+2c/b$.

(3). Let $\psi_b(r) = b\psi(r)$. It can be verified that $\psi_b(r)$ a sub-root function by checking the definition. Also, we have
$\psi(b^2r^*)/\sqrt{b^2r^*} \le \psi(r^*)/\sqrt{r^*}$, so
$\psi(b^2r^*) \le b r^*$ and $\psi_b(b^2 r^*) = b \psi(b^2r^*) \le b^2r^*$.
Then it follows from {\cite[Lemma 3.2]{bartlett2005}} that $r^*_b \le b^2r^*$.
\end{proof}

\begin{lemma}
\label{lemma:kernel-complexity-RC-relation}
Suppose $K$ is a PD kernel defined over $\cX \times \cX$ and $\sup_{\bx \in \cX} K(\bx,\bx) \le
\tau_0^2$ for a positive constant
 $\tau_0$ and $\inf_{\bx \in \cX} K(\bx,\bx) \gsim 1$. Then there are universal positive constants $c_{\ell}$ and $C_u$ with $0 < c_{\ell} < C_u$ such that when $n \gsim 1/\lambda_1$ and
$t^2 \gsim 1/n$, we have
\bal
\label{eq:kernel-complexity-RC-relation-population}
c_{\ell}R_K(t) \le \cR_K(t) \le C_u R_K(t).
\eal
Furthermore,  when $t^2 \gsim 1/n$, we have
\bal
\label{eq:kernel-complexity-RC-relation-empirical}
c_{\ell} \hat R_K(t) \le \hat \cR_K(t) \le C_u \hat R_K(t).
\eal
Herein $\cR_K(t)$ and $\hat \cR_K(t)$ are defined in
(\ref{eq:F-t-RC}).
\end{lemma}
\begin{remark*}
We note that (\ref{eq:kernel-complexity-RC-relation-empirical})
does not require the condition
$n \gsim 1/\lambda_1$.
It is also noted that \cite[Theorem 41]{Mendelson02-geometric-kernel-machine} presents the relation between
 $R_K(t)$ and $\cR_K(t)$ in (\ref{eq:kernel-complexity-RC-relation-population}),
 we herein further provide the relation
 between $\hat R_K(t)$ and $\hat \cR_K(t)$.
%for bounded kernel where $\sup_{\bx \in \cX} K(\bx,\bx)$ is bounded by a constant, it is remarked that Lemma~\ref{lemma:kernel-complexity-RC-relation}
% can handle general finite $\tau_0$ which may not be a constant.
\end{remark*}
\begin{proof}
We first prove
(\ref{eq:kernel-complexity-RC-relation-empirical}).
Let $\bsigma = \bth{\sigma_1,\ldots,\bsigma_n}$ be $n$ i.i.d.
Rademacher variables, for a function $\cF$ we define
\bal
R_{\cF}(\bsigma) &\defeq \frac{1}{\sqrt n}
\sup_{f \in \cF} \sum\limits_{i=1}^n \sigma_i f(\bbx_i),
\label{eq:R-F-sigma}
\eal
and $\hat \sigma^2_{\cF} \defeq \sup_{f \in \cF} \norm{f}{n}^2$.
We consider $R_{\hat \cF_{K,t}}(\bsigma)$ and
$\hat \sigma^2_{\hat \cF_{K,t}}$ in the sequel, where
$\hat \cF_{K,t}$ is the function class defined in (\ref{eq:F-t-hat-Ft-def}).
We note that $R_{\hat \cF_{K,t}}(\bsigma)  = \frac{1}{\sqrt n}
\sup_{f \in \hat \cF_{K,t}} \abth{\sum\limits_{i=1}^n \sigma_i f(\bbx_i)}$
due to the symmetry of the function class $\hat \cF_{K,t}$, that is,
$-\hat \cF_{K,t} = \hat \cF_{K,t}$. The proof of
(\ref{eq:kernel-complexity-RC-relation-empirical}) is then completed
by the following two steps.

\vspace{0.1in}\noindent \textit{{\bf Step 1}: Prove that if
$\hat \sigma^2_{\hat \cF_{K,t}} \gsim 1/n$, then
$\Prob{R_{\hat \cF_{K,t}}(\bsigma) \gsim m \Expect{\bsigma}{R_{\hat \cF_{K,t}}(\bsigma)}} \le \exp(-\Theta(m))$ for all $m \in \NN$.}

Let $Z = \sup_{f \in \hat \cF_{K,t}}
\pth{\sum\limits_{i=1}^n \sigma_i f(\bbx_i)} =  \sup_{f \in \hat \cF_{K,t}}
\abth{\sum\limits_{i=1}^n \sigma_i f(\bbx_i)}$, then
it follows from \cite[Theorem 4]{Massart2000-Talagrand-empirical-process} that with probability at least $1-\exp(-x)$ for every $x > 0$,
\bal
\label{eq:kernel-complexity-RC-relation-empirical-seg1}
R_{\hat \cF_{K,t}}(\bsigma) \le 2\Expect{\bsigma}{R_{\hat \cF_{K,t}}(\bsigma)}
+  C \pth{\hat \sigma_{\hat \cF_{K,t}} {\sqrt x} + \frac{x}{\sqrt n}},
\eal
where $C > 0$ is a positive constant.

By the definition of $\hat \sigma^2_{\hat \cF_{K,t}}$, for any $\tau > 0$,
there exists $f_{\tau} \in \hat \cF_{K,t} $ such that
$\norm{f_{\tau}}{n}^2 \ge \hat \sigma^2_{\hat \cF_{K,t}} - \tau$.
It follows from the Kahane-Khintchine inequality
\cite{Milman1986-asymptotic-fdim-normed-space} that
\bals
\Expect{\bsigma}{
\sup_{f \in \hat \cF_{K,t}} \abth{\sum\limits_{i=1}^n \sigma_i f(\bbx_i)}}
\ge
\Expect{\bsigma}{ \abth{\sum\limits_{i=1}^n \sigma_i f_{\tau}(\bbx_i)}}
\ge c' {\sqrt n}\norm{f_{\tau}}{n} \ge c'  {\sqrt n}\sqrt{\hat \sigma^2_{\hat \cF_{K,t}} -\tau},
\eals
where $c' > 0$ is a postive constant.
Letting $\tau \to 0$ in the above inequality, we have
\bal
\label{eq:kernel-complexity-RC-relation-empirical-seg2}
\Expect{\bsigma}{R_{\hat \cF_{K,t}}(\bsigma)} = \frac{1}{\sqrt n}\Expect{\bsigma}{
\sup_{f \in \hat \cF_{K,t}} \abth{\sum\limits_{i=1}^n \sigma_i f(\bbx_i)}}
\ge c' \sigma_{\hat \cF_{K,t}}.
\eal
Furthermore, since $\hat \sigma^2_{\hat \cF_{K,t}} \gsim 1/n$,
it follows from (\ref{eq:kernel-complexity-RC-relation-empirical-seg2})  that
\bal
\label{eq:kernel-complexity-RC-relation-empirical-seg3}
\frac{1}{\sqrt n} \lsim  \frac{\Expect{\bsigma}{R_{\hat \cF_{K,t}}(\bsigma)}}{c'}.
\eal
It then follows from
(\ref{eq:kernel-complexity-RC-relation-empirical-seg1})-
(\ref{eq:kernel-complexity-RC-relation-empirical-seg3}) that
with probability at least $1-\exp(-x)$,
\bals
R_{\hat \cF_{K,t}}(\bsigma) \le \pth{2+\frac{C{\sqrt x}+C \cdot \Theta(1)x}{c'}}\Expect{\bsigma}{R_{\hat \cF_{K,t}}(\bsigma)}.
\eals
It follows from the above inequality that
$\Prob{R_{\hat \cF_{K,t}}(\bsigma) \gsim m \Expect{\bsigma}{R_{\hat \cF_{K,t}}(\bsigma)}} \le \exp(-\Theta(m))$ for all $m \in \NN$.

\vspace{0.1in}\noindent \textit{{\bf Step 2}: Prove that if
$\Prob{R_{\hat \cF_{K,t}}(\bsigma) \gsim m\tau_0 \Expect{\bsigma}{R_{\hat \cF_{K,t}}(\bsigma)}} \le \exp(-\Theta(m))$ for all $m \in \NN$, then }
\bal
\label{eq:kernel-complexity-RC-relation-empirical-seg4}
\pth{\Expect{\bsigma}{R_{\hat \cF_{K,t}}^2(\bsigma)}}^{1/2} \lsim
\Expect{\bsigma}{R_{\hat \cF_{K,t}}(\bsigma)} \lsim
\pth{\Expect{\bsigma}{R_{\hat \cF_{K,t}}^2(\bsigma)}}^{1/2}.
\eal
(\ref{eq:kernel-complexity-RC-relation-empirical-seg4}) follows
by repeating the argument in the proof of
\cite[Lemma 44]{Mendelson02-geometric-kernel-machine} (with
$a = \Expect{\bsigma}{R_{\hat \cF_{K,t}}(\bsigma)}$ in that proof).

It follows from Lemma~\ref{lemma:norm-f-n-lower-bound}
that $\hat \sigma^2_{\hat \cF_{K,t}} \gsim 1/n$, and Steps 1-2 indicate that (\ref{eq:kernel-complexity-RC-relation-empirical-seg4}) holds.
Then
(\ref{eq:kernel-complexity-RC-relation-empirical}) is proved by
(\ref{eq:kernel-complexity-RC-relation-empirical-seg4}) and
(\ref{eq:kernel-complexity-RC-relation-R-hat-Ft-bound})
in Lemma~\ref{lemma:kernel-complexity-RC-relation-R-hat-Ft-bound},
along with the definition of the empirical kernel complexity $\hat R_K$
in (\ref{eq:kernel-LRC-empirical}).

Let $\sigma^2_{\cF} \defeq \sup_{f \in \cF_{K,t}} \norm{f}{L^2}^2$ where
$\cF_{K,t}$ is defined in (\ref{eq:F-t-hat-Ft-def}). Similar to
Lemma~\ref{lemma:norm-f-n-lower-bound}, we have
$\sigma^2_{\cF} \gsim 1/n$ when $n \gsim 1/\lambda_1$.
(\ref{eq:kernel-complexity-RC-relation-population})
is then proved by a similar argument, which also follows the proof of
\cite[Theorem 41]{Mendelson02-geometric-kernel-machine}.

\end{proof}

Suppose $K$ is a PD kernel defined over $\cX \times \cX$ and let the empirical gram matrix computed by $K$ on the training features $\bS$
be $\bK_n$ with the eigenvalues
$\hlambda_1 \ge \ldots \ge \hlambda_n \ge 0$. We need the following background in the RKHS spanned by
$\set{K(\cdot,\bbx_i)}_{i=1}^n$ for the proof of Lemma~\ref{lemma:norm-f-n-lower-bound} and
Lemma~\ref{lemma:kernel-complexity-RC-relation-R-hat-Ft-bound}.
We define
\bal
\cH_{K,\bS} = \set{\sum\limits_{i=1}^n \alpha_i K(\cdot,\bbx_i) \colon
\set{\alpha_i} \subseteq \RR^n }
\eal
as the RKHS spanned by $\set{K(\cdot,\bbx_i)}_{i=1}^n$.
Herein we introduce the operator $\hat T_n \colon \cH_{K,\bS}  \to \cH_{K,\bS} $ which is  defined by $\hat T_n g \defeq \frac 1n \sum_{i=1}^n K(\cdot,\bbx_i) g(\bbx_i)$ for every $g \in \cH_{K,\bS} $. It can be verified that the
eigenvalues of $\hat T_n$ coincide with the eigenvalues of $\bK_n$, that is,
the eigenvalues of $\hat T_n$ are
$\set{\hat \lambda_i}_{i=1}^n$. By the spectral theorem, all the normalized eigenfunctions of $\hat T_n$, denoted by $\set{{\Phi}^{(k)}}_{k = 1}^n$ with ${\Phi}^{(k)} = 1/{\sqrt{n \hat \lambda_k}} \cdot \sum_{j=1}^n
K(\cdot,\bbx_j) \bth{\bU^k}_j$, is an orthonormal basis of $\cH_{K,\bS}$. Since $\cH_{K,\bS} \subseteq \cH_{K}$, we can complete $\set{{\Phi}^{(k)}}_{k = 1}^n$ so that
$\set{{\Phi}^{(k)}}_{k \ge 1}$ is an orthonormal basis of the RKHS $\cH_{K}$.

\begin{lemma}
\label{lemma:norm-f-n-lower-bound}
Suppose $K$ is a PD kernel defined over $\cX \times \cX$ and
$\inf_{\bx \in \cX} K(\bx,\bx) \gsim 1$.
Then
$\sup_{f \in \hat \cF_{K,t}} \norm{f}{n}^2
\gsim 1/n$ when $t^2 \gsim 1/n$, where $\hat \cF_{K,t}$ is defined in (\ref{eq:F-t-hat-Ft-def}).
\end{lemma}
\begin{proof}
Let the empirical gram matrix computed by $K$ be $\bK_n$ with eigenvalues
$\hlambda_1 \ge \ldots \ge \hlambda_n \ge 0$.
Then $\hlambda_1 \ge \sum_{k=1}^n \hlambda_k/n \gsim 1/n$ since
$\sum_{k=1}^n \hlambda_k = \sum_{k=1}^n K(\bbx_i,\bbx_i)/n \gsim 1$.

For any $f \in \hat \cF_{K,t}$, let the projection of $f$ onto
$\cH_{K,\bS} $ be $\Proj_{\cH_{K,\bS}}(f) =
\sum\limits_{k=1}^n \beta_k{\Phi}^{(k)}$, $\bbeta =
\bth{\beta_1,\ldots,\beta_n} \in \RR^n$, then
$\ltwonorm{\bbeta} \le 1$. We have
$\iprod{\hat T_n f}{f} = \sum\limits_{k=1}^n
\beta_k^2 \hlambda_k = \norm{f}{n}^2$.
If $t^2 \ge \hat \lambda_1$, then
$\sup_{f \in \hat \cF_{K,t}} \norm{f}{n}^2 = \hat \lambda_1 \gsim 1/n$. Otherwise, we still have
$\sup_{f \in \hat \cF_{K,t}} \norm{f}{n}^2 = t^2 \gsim 1/n$.
\end{proof}

\begin{lemma}
\label{lemma:kernel-complexity-RC-relation-R-hat-Ft-bound}
Let $\bsigma = \bth{\sigma_1,\ldots,\sigma_n}$ be n i.i.d.
Rademacher variables, then
\bal
\label{eq:kernel-complexity-RC-relation-R-hat-Ft-bound}
\pth{\sum\limits_{k=1}^n \min\set{\hlambda_k,t^2}}^{1/2} \le \pth{\Expect{\bsigma}{R_{\hat \cF_{K,t}}^2(\bsigma)}}^{1/2}
\le  {\sqrt 2} \pth{\sum\limits_{k=1}^n \min\set{\hlambda_k,t^2}}^{1/2},
\eal
where $R_{\cdot}(\bsigma)$ is defined in
(\ref{eq:R-F-sigma}), and $\hat \cF_{K,t}$ is defined in (\ref{eq:F-t-hat-Ft-def}).
\end{lemma}
\begin{proof}
The proof follows a similar argument in the proof of
\cite[Lemma 42]{Mendelson02-geometric-kernel-machine}.
\end{proof}

\subsection{Proofs of Theorem~\ref{theorem:sup-hat-g} and Theorem~\ref{theorem:V_R}}
\label{sec:proofs-uniform-convergence-sup-hat-g-V_R}

We need the following definition of $\eps$-net for the proof of Theorem~\ref{theorem:sup-hat-g} and Theorem~\ref{theorem:V_R}.
\begin{definition}\label{definition:eps-net}
($\eps$-net) Let $(X, d)$ be a metric space and let $ \eps > 0$. A subset $N_{\eps}(X,d)$ is called an $\eps$-net of $X$ if for every point $x \in X$, there exists some point $y \in N_{\eps}(X,d)$ such that $d(x,y) \le \eps$. The minimal cardinality of an $\eps$-net of $X$, if finite, is
denoted by $N(X, d, \eps)$ and is called the covering number of $X$
at scale $\eps$.
\end{definition}

\begin{proof}
[\textup{\textbf{Proof of Theorem~\ref{theorem:sup-hat-g}}}]
%Let $\tbu = \bth{\bu^{\top},1}^{\top}$, $\tbv = \bth{\bv^{\top},1}^{\top}$, and $\tbu', \tbv'$ are constructed similarly from $\tbu', \tbv'$.
First, we have $\Expect{\bw \sim \cN(\bzero, \kappa^2 \bI_{d+1})}{\tilde h(\bw,\tbu,\tbv)} = \tK(\tbu,\tbv)$.
For any $\tbu \in \unitsphere{d}$, $\tbv \in \unitsphere{d}$, and $s > 0$, define function class
\bal\label{eq:sup-hat-g-seg1}
&\cH_{\tbu,\tbv,s} \coloneqq \set{\tilde h(\cdot,\tbu',\tbv') \colon \RR^{d+1} \to \RR \colon \tbu' \in \ball{}{\tbu}{s} \cap \unitsphere{d}, \tbv'  \in \ball{}{\tbv}{s} \cap \unitsphere{d}}.
\eal%
We first build an $s$-net for the unit sphere $\unitsphere{d}$. By~\cite[Lemma 5.2]{Vershynin2012-nonasymptotics-random-matrix}, there exists an $s$-net $N_{s}(\unitsphere{d}, \ltwonorm{\cdot})$ of $\unitsphere{d}$ such that $N(\unitsphere{d}, \ltwonorm{\cdot},s) \le \pth{1+\frac{2}{s}}^{d+1}$.

%Given fixed $\tbv \in N(\unitsphere{d}, \eps), \tbv \in N(\unitsphere{d}, \eps)$, with probability at least $1-\delta$ over the random initialization ${\bW(0)}$,
%\bal\label{eq:sup-hat-g-seg1-1}
%&\abth{\tilde h({\bW(0)},\tbv,\tbu) - \Expect{{\bW(0)}}{\tilde h({\bW(0)},\tbv,\tbu)}} \le \sqrt{\frac{\log \frac 2\delta}{2m}}.
%\eal%

In the sequel, a function in the class $\cH_{\tbu,\tbv,s}$ is also denoted as $\tilde h(\bw)$, omitting the presence of variables $\tbu'$ and $\tbv'$ when no confusion arises. Let $P_m$ be the empirical distribution over $\set{{{\bbw_r(0)}}}$ so that $\Expect{\bw \sim P_m}{\tilde h(\bw)} = \tilde h({\bW(0)},\tbu,\tbv)$. Given $\tbu \in N(\unitsphere{d}, s)$, we aim to estimate the upper bound for the supremum of empirical process $\Expect{ \bw \sim \cN(\bzero, \kappa^2 \bI_{d+1}) }{\tilde h(\bw)} - \Expect{\bw \sim P_m}{\tilde h(\bw)}$ when function $\tilde h$ ranges over the function class $\cH_{\tbu,\tbv,s}$. To this end, we apply Theorem~\ref{theorem:Talagrand-inequality} to the function class $\cH_{\tbu,\tbv,s}$ with ${\bW(0)} = \set{{{\bbw_r(0)}}}_{r=1}^m$. It can be verified that $\tilde h \in [-1,1]$ for any $\tilde h \in \cH_{\tbu,\tbv,s}$. It follows that we can set $a=-1,b=1,\alpha = \frac 12$ in Theorem~\ref{theorem:Talagrand-inequality}. Since $\Var{\tilde h} \le \Expect{\bw}{\tilde h(\bw,\tbu',\tbv)^2} \le 1$, with probability at least $1-\delta$, over the random initialization ${\bW(0)}$,
\bal\label{eq:sup-hat-g-seg2}
\sup_{\tbu' \in \ball{}{\tbu}{s} \cap \unitsphere{d},\tbv'  \in \ball{}{\tbv}{s} \cap \unitsphere{d}} \abth{ \tK(\tbu',\tbv') - \tilde h({\bW(0)},\tbu',\tbv') }
&\le 3 \cR(\cH_{\tbu,\tbv,s}) + \sqrt{\frac{2{\log{\frac 1\delta}}}{m}} + \frac{14 {\log{\frac 1 \delta}} }{3m},
\eal%
where $\cR(\cH_{\tbu,\tbv,s}) = \Expect{{\bW(0)}, \set{\sigma_r}_{r=1}^m}{ \sup_{\tilde h \in \cH_{\tbu,\tbv,s}} {\frac{1}{m} \sum\limits_{r=1}^m {\sigma_r}{ \tilde h({{\bbw_r(0)}}) }}}$ is the Rademacher complexity of the function class $\cH_{\tbu,\tbv,s}$, $\set{\sigma_r}_{r=1}^m$ are i.i.d. Rademacher random variables taking values of $\pm 1$ with equal probability. By Lemma~\ref{lemma:rademacher-bound-hat-h}, $\cR(\cH_{\tbu,\tbv,s}) \le 2 \pth{B\sqrt{ds}(s+1)
+ \sqrt{s} + s}$. Plugging such upper bound for $\cR(\cH_{\tbu,\tbv,s})$ in (\ref{eq:sup-hat-g-seg2}) and setting $s = \frac 1m$, we have
\bal\label{eq:sup-hat-g-seg3}
\sup_{\substack{\tbu' \in \ball{}{\tbu}{s} \cap \unitsphere{d}, \\ \tbv'\in \ball{}{\tbv}{s} \cap \unitsphere{d}}} \abth{ \tK(\tbu',\tbv') - \tilde h({\bW(0)},\tbu',\tbv') }
&\le 6 \pth{  \frac{B \sqrt{d} \pth{1+ \frac 1m}} {\sqrt m} +  \frac {1}{ \sqrt{m}} + \frac 1m } + \sqrt{\frac{2{\log{\frac 1\delta}}}{m}} + \frac{14 {\log{\frac 1\delta}} }{3m} \nonumber \\
&\le \frac{1}{\sqrt m} \pth{6(1+2B\sqrt{d}) + \sqrt{2\log{\frac 1\delta}}} + \frac{1}{m} \pth{6+\frac{14 {\log{\frac 1\delta}} }{3}}.
\eal%
By the union bound, with probability at least $1 - \pth{1+2m}^{2(d+1)} \delta$ over ${\bW(0)}$, (\ref{eq:sup-hat-g-seg3}) holds for arbitrary $\tbu,\tbv \in N(\unitsphere{d}, s)$. In this case, for any $\tbu' \in \unitsphere{d},  \tbv' \in \unitsphere{d}$, there exists $\tbu,\tbv \in N_{s}(\unitsphere{d}, \ltwonorm{\cdot})$ such that $\ltwonorm{\tbu' - \tbu} \le s, \ltwonorm{\tbv'-\tbv} \le s$, so that $\tbu' \in \ball{}{\tbu}{s} \cap \unitsphere{d}, \tbv'\in \ball{}{\tbv}{s} \cap \unitsphere{d}$, and (\ref{eq:sup-hat-g-seg3}) holds. Changing the notations $\tbu',\tbv'$ to $\tbu,\tbv$, (\ref{eq:sup-hat-h}) is proved.

\end{proof}

\begin{lemma}\label{lemma:rademacher-bound-hat-h}
Let $\cR(\cH_{\tbu,\tbv,s}) \defeq \Expect{{\bW(0)}, \set{\sigma_r}_{r=1}^m}{ \sup_{\tilde h \in \cH_{\tbu,\tbv,s}} {\frac{1}{m} \sum\limits_{r=1}^m {\sigma_r}{ \tilde h({{\bbw_r(0)}}) }}}$ be the Rademacher complexity of the function class $\cH_{\tbu,\tbv,s}$, and $B$ is a positive constant. Then

\bal\label{eq:rademacher-bound-hat-h}
&\cR(\cH_{\tbu,\tbv,s}) \le 2 \pth{B\sqrt{ds}(s+1)
+ \sqrt{s} + s}.
\eal%
\end{lemma}
\begin{proof}
%Let $\tbu = \bth{\bu^{\top},1}^{\top}$, $\bv = \bth{\tbv^{\top},1}^{\top}$, and $\tbu', \tbv'$ are constructed similarly from $\bu', \bv'$.
We have
\bal\label{eq:rademacher-bound-hat-h-seg1}
\cR(\cH_{\tbu,\tbv,s})
&= \Expect{{\bW(0)}, \set{\sigma_r}_{r=1}^m}{ \sup_{\tbu' \in \ball{}{\tbu}{s} \cap \unitsphere{d}, \tbv'\in \ball{}{\tbv}{s} \cap \unitsphere{d}} {\frac{1}{m} \sum\limits_{r=1}^m {\sigma_r}{ \tilde h({{\bbw_r(0)}}, \tbu' ,\tbv') }}}  \le \cR_1 + \cR_2 + \cR_3,
\eal%
where
\bal\label{eq:rademacher-bound-hat-h-seg2}
&\cR_1 = \Expect{{\bW(0)}, \set{\sigma_r}_{r=1}^m}{ \sup_{\tbu' \in \ball{}{\tbu}{s} \cap \unitsphere{d}, \tbv' \in \ball{}{\tbv}{s} \cap \unitsphere{d}} {\frac{1}{m} \sum\limits_{r=1}^m {\sigma_r}{ \pth{\tilde h({{\bbw_r(0)}}, \tbu' ,\tbv') - \tilde h({{\bbw_r(0)}}, \tbu ,\tbv') }}}}, \nonumber \\
&\cR_2 = \Expect{{\bW(0)}, \set{\sigma_r}_{r=1}^m}{ \sup_{\tbu' \in \ball{}{\tbu}{s} \cap \unitsphere{d}, \tbv' \in \ball{}{\tbv}{s} \cap \unitsphere{d}} {\frac{1}{m} \sum\limits_{r=1}^m {\sigma_r}{ \pth{\tilde h({{\bbw_r(0)}}, \tbu ,\tbv') - \tilde h({{\bbw_r(0)}}, \tbu ,\tbv) }}}}, \nonumber \\
&\cR_3 = \Expect{{\bW(0)}, \set{\sigma_r}_{r=1}^m}{ \sup_{\tbu' \in \ball{}{\tbu}{s} \cap \unitsphere{d}, \tbv' \in \ball{}{\tbv}{s} \cap \unitsphere{d}} {\frac{1}{m} \sum\limits_{r=1}^m {\sigma_r}{ \tilde h({{\bbw_r(0)}}, \tbu,\tbv)  }}}.
\eal%
Here (\ref{eq:rademacher-bound-hat-h-seg1}) follows from the subadditivity of supremum.
Now we bound $\cR_1$, $\cR_2$, and $\cR_3$ separately.
First, $\cR_3 = 0$ by the definition of the Rademacher variables.
For $\cR_1$, we first define $$Q \defeq \frac 1m \sum\limits_{r=1}^m \indict{\indict{\tbu'^{\top}{{\bbw_r(0)}} \ge 0} \neq \indict{\tbu^{\top} {{\bbw_r(0)}} \ge 0} },$$
 which is the average number of weights in ${\bW(0)}$ whose inner products with $\tbu$ and $\tbu'$ have different signs. Our observation is that, if $\abth{ \tbu^{\top} {{\bbw_r(0)}}} > s \norm{{{\bbw_r(0)}}}{2}$, then $\tbu^{\top} {{\bbw_r(0)}}$ has the same sign as $\tbu'^{\top} {{\bbw_r(0)}}$. To see this, by the Cauchy-Schwarz inequality,
$\abth{\tbu'^{\top} {{\bbw_r(0)}} - \tbu^{\top} {{\bbw_r(0)}}} \le \norm{\tbu'-\tbu}{2} \norm{{{\bbw_r(0)}}}{2} \le s \norm{{{\bbw_r(0)}}}{2}$,
then we have $\tbu^{\top} {{\bbw_r(0)}} > s \norm{{{\bbw_r(0)}}}{2} \Rightarrow \tbu'^{\top} {{\bbw_r(0)}} \ge \tbu^{\top} {{\bbw_r(0)}} - s \norm{{{\bbw_r(0)}}}{2} > 0$, and  $\tbu^{\top} {{\bbw_r(0)}} < -s \norm{{{\bbw_r(0)}}}{2} \Rightarrow \tbu'^{\top} {{\bbw_r(0)}} \le \tbu^{\top} {{\bbw_r(0)}} + s \norm{{{\bbw_r(0)}}}{2} < 0$.
As a result,
\bals
Q \le \frac 1m \sum\limits_{r=1}^m \indict{\abth{ \tbu^{\top} {{\bbw_r(0)}}} \le s \norm{{{\bbw_r(0)}}}{2}} \defeq \tilde Q,
\eals
and it follows that
\bal\label{eq:rademacher-bound-hat-h-seg5}
\Expect{{\bW(0)}}{\tilde Q} = \Prob{\abth{ \tbu^{\top} {{\bbw_r(0)}}} \le s \norm{{{\bbw_r(0)}}}{2}} =\Prob{ \frac{ \abth{\tbu^{\top}{{\bbw_r(0)}}} } { \norm{{{\bbw_r(0)}}}{2} }  \le s},
\eal%
where the last equality holds because each ${{\bbw_r(0)}}, r \in [m]$, follows a continuous Gaussian distribution. It follows from Lemma~\ref{lemma:uniform-marginal-bound} that $\Prob{ \frac{ \abth{\tbu^{\top}{{\bbw_r(0)}}} } { \norm{{{\bbw_r(0)}}}{2} }  \le s} \le B \sqrt{d} s$ for an absolute positive constant $B$. According to this inequality and (\ref{eq:rademacher-bound-hat-h-seg5}), it follows that $\Expect{{\bW(0)}}{\tilde Q} \le B\sqrt{d}s$. By the Markov's inequality, we have
$\Prob{\tilde Q \ge \sqrt{s}} \le B\sqrt{ds}$, where the probability is with respect to the probability measure space of ${\bW(0)}$. Let $A$ be the event that $\tilde Q \ge \sqrt{s}$. We denote by $\Omega_s$ the subset of the probability measure space of ${\bW(0)}$ such that $A$ happens, then $\Prob{\Omega_s} \le B\sqrt{ds}$. Now we aim to bound $\cR_1$ by estimating its bound on $\Omega_s$ and its complement. First, we have
\bal\label{eq:rademacher-bound-hat-h-seg8}
\cR_1 &= \Expect{{\bW(0)}, \set{\sigma_r}_{r=1}^m}{ \sup_{\tbu' \in \ball{}{\tbu}{s} \cap \unitsphere{d}, \tbv' \in \ball{}{\tbv}{s} \cap \unitsphere{d}} {\frac{1}{m} \sum\limits_{r=1}^m {\sigma_r}{ \pth{ \tilde h({{\bbw_r(0)}}, \tbu' ,\tbv') - \tilde h({{\bbw_r(0)}}, \tbu ,\tbv')} }}} \nonumber \\
&= \underbrace{\Expect{{\bW(0)} \in \Omega_s, \set{\sigma_r}_{r=1}^m}{ \sup_{\tbu' \in \ball{}{\tbu}{s} \cap \unitsphere{d}, \tbv' \in \ball{}{\tbv}{s} \cap \unitsphere{d}} {\frac{1}{m} \sum\limits_{r=1}^m {\sigma_r}{ \pth{ \tilde h({{\bbw_r(0)}}, \tbu' ,\tbv') - \tilde h({{\bbw_r(0)}}, \tbu ,\tbv')} }}}}_{\textup{$\cR_{11}$}}  \nonumber \\
&\phantom{=}+ \underbrace{\Expect{{\bW(0)} \notin \Omega_s, \set{\sigma_r}_{r=1}^m}{ \sup_{\tbu' \in \ball{}{\tbu}{s} \cap \unitsphere{d}, \tbv' \in \ball{}{\tbv}{s} \cap \unitsphere{d}} {\frac{1}{m} \sum\limits_{r=1}^m {\sigma_r}{ \pth{ \tilde h({{\bbw_r(0)}}, \tbu' ,\tbv') - \tilde h({{\bbw_r(0)}}, \tbu ,\tbv')} }}}}_{\textup{$\cR_{12}$}},
\eal%
where we used the convention that $\Expect{{\bW(0)} \in \Omega_s}{\cdot} = \Expect{{\bW(0)}}{\indict{{\bW(0)} \in \Omega_s} \times \cdot}$.
Now we estimate the upper bound for $\cR_{11}$ and $\cR_{12}$ separately. Let $I = \set{r \in [m] \colon \indict{\tbu'^{\top}{{\bbw_r(0)}} \ge 0} \neq \indict{\tbu^{\top} {{\bbw_r(0)}} \ge 0} }$. When ${\bW(0)} \notin \Omega_s$, we have $Q \le \tilde Q < \sqrt{s}$. In this case, it follows that $|I| \le m \sqrt{s}$.  Moreover, when $r \in I$, either $\indict{\tbu'^{\top}{{\bbw_r(0)}} \ge 0} = 0$ or $\indict{\tbu^{\top}{{\bbw_r(0)}} \ge 0} = 0$. As a result,
\bal\label{eq:rademacher-bound-hat-h-seg9}
&\abth{\tilde h({{\bbw_r(0)}}, \tbu' ,\tbv') - \tilde h({{\bbw_r(0)}}, \tbu ,\tbv')}  \nonumber \\
&= \abth{ \tbu'^{\top}\tbv'\indict{\tbu'^{\top}{{\bbw_r(0)}} \ge 0}\indict{\tbv'^{\top}{{\bbw_r(0)}} \ge 0} -  \tbu^{\top}\tbv'\indict{\tbu^{\top}{{\bbw_r(0)}} \ge 0}\indict{\tbv'^{\top}{{\bbw_r(0)}} \ge 0} } \le 1.
\eal%
When $r \in [m] \setminus I$, we have
\bal\label{eq:rademacher-bound-hat-h-seg10}
&\abth{\tilde h({{\bbw_r(0)}}, \tbu' ,\tbv') - \tilde h({{\bbw_r(0)}}, \tbu ,\tbv')}  \nonumber \\
&= \abth{ \tbu'^{\top}\tbv'\indict{\tbu'^{\top}{{\bbw_r(0)}} \ge 0}\indict{\tbv'^{\top}{{\bbw_r(0)}} \ge 0} -  \tbu^{\top}\tbv'\indict{\tbu^{\top}{{\bbw_r(0)}} \ge 0}\indict{\tbv'^{\top}{{\bbw_r(0)}} \ge 0} } \nonumber \\
&\stackrel{\circled{1}}{=}\abth{ \pth{\tbu'- \tbu }^{\top} \tbv'\indict{\tbu^{\top}{{\bbw_r(0)}} \ge 0}\indict{\tbv'^{\top}{{\bbw_r(0)}} \ge 0} }
\stackrel{\circled{2}}{\le} \norm{\tbu' - \tbu }{2} \norm{\tbv'}{2} \abth{\indict{\tbu^{\top}{{\bbw_r(0)}} \ge 0}} \abth{\indict{\tbv'^{\top}{{\bbw_r(0)}} \ge 0}}
\stackrel{\circled{3}}{\le}  s,
\eal%
where $\circled{1}$ follows from $\indict{\tbu'^{\top}{{\bbw_r(0)}} \ge 0} = \indict{\tbu^{\top} {{\bbw_r(0)}} \ge 0}$ because $r \notin I$.
$\circled{2}$ follows from the Cauchy-Schwarz inequality.  $\circled{3}$ follows from
$\tbu' \in \ball{}{\tbu}{s}$ and $\abth{\indict{\tbu^{\top}{{\bbw_r(0)}} \ge 0}},\abth{\indict{\tbv'^{\top}{{\bbw_r(0)}} \ge 0}} \in \set{0,1}$.
By (\ref{eq:rademacher-bound-hat-h-seg9}) and (\ref{eq:rademacher-bound-hat-h-seg10}), for any $\tbu' \in \ball{}{\tbu}{s} \cap \unitsphere{d}$ and
$\tbv' \in \ball{}{\tbv}{s} \cap \unitsphere{d}$, we have
\bal\label{eq:rademacher-bound-hat-h-seg11}
&\frac{1}{m} \sum\limits_{r=1}^m {\sigma_r}{ \pth{ \tilde h({{\bbw_r(0)}}, \tbu' ,\tbv') - \tilde h({{\bbw_r(0)}}, \tbu ,\tbv')} } \nonumber \\
&=\frac{1}{m} \sum\limits_{r \in I} {\sigma_r}{ \pth{ \tilde h({{\bbw_r(0)}}, \tbu' ,\tbv') - \tilde h({{\bbw_r(0)}}, \tbu ,\tbv')} } + \frac{1}{m} \sum\limits_{r \in [m] \setminus I} {\sigma_r}{ \pth{ \tilde h({{\bbw_r(0)}}, \tbu' ,\tbv') - \tilde h({{\bbw_r(0)}}, \tbu ,\tbv')} } \nonumber \\
&\le \frac{1}{m} \sum\limits_{r \in I} { \abth{ \tilde h({{\bbw_r(0)}}, \tbu' ,\tbv') - \tilde h({{\bbw_r(0)}}, \tbu ,\tbv')} } + \frac{1}{m} \sum\limits_{r \in [m] \setminus I} { \abth{ \tilde h({{\bbw_r(0)}}, \tbu' ,\tbv') - \tilde h({{\bbw_r(0)}}, \tbu ,\tbv')} } \nonumber \\
&\stackrel{\circled{1}}{\le} \frac{m \sqrt{s}}{m} + \frac{m-|I| }{m} s \le \sqrt{s} + s,
\eal%
where $\circled{1}$ uses the bounds in (\ref{eq:rademacher-bound-hat-h-seg9}) and (\ref{eq:rademacher-bound-hat-h-seg10}).

Using (\ref{eq:rademacher-bound-hat-h-seg11}), we now estimate the upper bound for $\cR_{12}$ by
\bal\label{eq:rademacher-bound-hat-h-seg12}
\cR_{12} &=\Expect{{\bW(0)} \notin \Omega_s, \set{\sigma_r}_{r=1}^m}{ \sup_{\tbu' \in \ball{}{\tbu}{s} \cap \unitsphere{d}, \tbv' \in \ball{}{\tbv}{s} \cap \unitsphere{d}} {\frac{1}{m} \sum\limits_{r=1}^m {\sigma_r}{ \pth{ \tilde h({{\bbw_r(0)}}, \tbu' ,\tbv') - \tilde h({{\bbw_r(0)}}, \tbu ,\tbv')} }}} \nonumber \\
&\le \Expect{{\bW(0)} \notin \Omega_s, \set{\sigma_r}_{r=1}^m}{ { \sqrt{s} + s } } \le \sqrt{s} + s.
\eal%
When ${\bW(0)} \in \Omega_s$, for any $\tbu' \in \ball{}{\tbu}{s} \cap \unitsphere{d}$ and
$\tbv' \in \ball{}{\tbv}{s} \cap \unitsphere{d}$, we have
\bal\label{eq:rademacher-bound-hat-h-seg13}
&\abth{\tilde h({{\bbw_r(0)}}, \tbu' ,\tbv') - \tilde h({{\bbw_r(0)}}, \tbu ,\tbv')} \nonumber \\
&\le \norm{\tbu' - \tbu}{2} \abth{\indict{\tbu'^{\top}{{\bbw_r(0)}} \ge 0}} + \norm{\tbu}{2} \abth{ \indict{\tbu'^{\top}{{\bbw_r(0)}} \ge 0} - \indict{\tbu^{\top}{{\bbw_r(0)}} \ge 0} } \le s+1.
\eal%
For $\cR_{11}$, it follows from (\ref{eq:rademacher-bound-hat-h-seg13})
that
\bal\label{eq:rademacher-bound-hat-h-seg14}
\cR_{11} &= \Expect{{\bW(0)} \in \Omega_s, \set{\sigma_r}_{r=1}^m}{ \sup_{\tbu' \in \ball{}{\tbu}{s} \cap \unitsphere{d}, \tbv' \in \ball{}{\tbv}{s} \cap \unitsphere{d}} {\frac{1}{m} \sum\limits_{r=1}^m {\sigma_r}{ \pth{ \tilde h({{\bbw_r(0)}}, \tbu' ,\tbv' - \tilde h({{\bbw_r(0)}}, \tbu ,\tbv')} }}} \nonumber \\
&\stackrel{\circled{1}}{\le} \Expect{{\bW(0)} \in \Omega_s, \set{\sigma_r}_{r=1}^m}{  s+1 }  = (s+1)\Prob{\Omega_s} \le B\sqrt{ds}(s+1)
\eal%
Combining (\ref{eq:rademacher-bound-hat-h-seg8}), (\ref{eq:rademacher-bound-hat-h-seg12}), and (\ref{eq:rademacher-bound-hat-h-seg14}), we have the upper bound for $\cR_1$ as
\bal\label{eq:rademacher-bound-hat-h-seg15}
\cR_1 &=  \cR_{11} + \cR_{12} \le B\sqrt{ds}(s+1) + \sqrt{s} + s.
\eal%
Applying the argument for $\cR_1$ to $\cR_2$, we have
$\cR_2 \le B\sqrt{ds}(s+1) + \sqrt{s} + s$.
Plugging such upper bound for $\cR_2$,
(\ref{eq:rademacher-bound-hat-h-seg15}), and $\cR_3=0$ in (\ref{eq:rademacher-bound-hat-h-seg1}), we have
\bal\label{eq:rademacher-bound-hat-h-seg16}
&\cR(\cH_{\tbu,\tbv,s}) \le \cR_1 + \cR_2 + \cR_3 \le 2 \pth{B\sqrt{ds}(s+1)
+ \sqrt{s} + s}.
\eal%

\end{proof}

\begin{lemma}\label{lemma:uniform-marginal-bound}
Let $\bw \sim \cN(\bzero, \kappa^2 \bI_{d+1})$ with $\kappa > 0$. Then for any $\eps \in (0,1)$ and fixed $\tbu \in \unitsphere{d}$, $\Prob{\frac{\abth{\tbu^{\top}\bw}}{\norm{\bw}{2}} \le \eps} \le B \sqrt{d} \eps$ where $B$ is an absolute positive constant.
\end{lemma}
\begin{remark*}
In fact, $B$ can be set to $\pi^{-1/2}$.
\end{remark*}

\begin{proof}
Let $z = \frac{\tbu^{\top}\bw}{\norm{\bw}{2}}$.
It can be verified that $z^2 \sim z_1$ where $z_1$ is a random variable following the Beta distribution $\textup{Beta}(\frac 12, \frac {d}{2})$. Therefore, the distribution of $z$ has the following continuous probability density function $p_z$ with respect to the Lebesgue measure,
\bal\label{eq:uniform-marginal-bound-seg1}
p_z(x) =(1-x^2)^{\frac{d-2}{2}} \indict{\abth{x} \le 1}/{B'},
\eal%
where $B' = \int_{-1}^1 (1-x^2)^{\frac{d-2}{2}} \diff x = B(1/2,d/2) = {\sqrt \pi} \Gamma(d/2)/\Gamma((d+1)/2)$ is the normalization factor. It can be verified by standard calculation that $1/B' \le {B\sqrt{d}}/{2}$ for an absolute positive constant $B$.
Since $1-x^2 \le 1$ over $x \in [-1,1]$, we have
\bal\label{eq:uniform-marginal-bound-seg2}
&\Prob{\frac{\abth{\tbu^{\top}\bw}}{\norm{\bw}{2}} \le \eps} = \Prob{-\eps \le z \le \eps} =  \frac 1{B'} \int_{-\eps}^{\eps} (1-x^2)^{\frac{d-2}{2}} \diff x \le B \sqrt{d} \eps,
\eal
where the last inequality is due to the fact that $1-x^2 \le 1$ for $x \in [-\eps,\eps]$ with $\eps \in (0,1)$.
\end{proof}

\begin{proof}
[\textup{\textbf{Proof of Theorem~\ref{theorem:V_R}}}]
%Let $\tbu = \bth{\bu^{\top},1}^{\top}$, $\tbu' = \bth{\bu'^{\top},1}^{\top}$.
We follow a similar proof strategy as that for Theorem~\ref{theorem:sup-hat-g}.

First, we have $\Expect{\bw \sim \cN(\bzero, \kappa^2 \bI_{d+1})}{\tilde v_R(\bw,\tbu)} = \Prob{\abth{\bw^{\top}\tbu} \le R}$. For any $\tbu \in \unitsphere{d}$ and $s > 0$, define function class
\bal\label{eq:sup-hat-v-seg1}
&\cV_{\tbu,s} \coloneqq \set{\tilde v_R(\cdot,\tbu') \colon \RR^{d+1} \to \RR \colon \tbu' \in \ball{}{\tbu}{s} \cap \unitsphere{d}}.
\eal%
We first build an $s$-net for the unit sphere $\unitsphere{d}$. It follows from \cite[Lemma 5.2]{Vershynin2012-nonasymptotics-random-matrix} that there exists an $s$-net $N_{s}(\unitsphere{d}, \ltwonorm{\cdot})$ of $\unitsphere{d}$ such that $N(\unitsphere{d},\ltwonorm{\cdot}, s) \le \pth{1+\frac{2}{s}}^{d+1}$.

%Given fixed $\tbv \in N(\unitsphere{d}, \eps), \tbv \in N(\unitsphere{d}, \eps)$, with probability at least $1-\delta$ over the random initialization ${\bW(0)}$,
%\bal\label{eq:sup-hat-v-seg1-1}
%&\abth{\tilde h({\bW(0)},\tbv,\tbu) - \Expect{{\bW(0)}}{\tilde h({\bW(0)},\tbv,\tbu)}} \le \sqrt{\frac{\log \frac 2\delta}{2m}}.
%\eal%
In the sequel, a function in the class $\cV_{\tbu}$ is also denoted as $\tilde v_R(\bw)$, omitting the presence of $\tbu'$ when no confusion arises. Let $P_m$ be the empirical distribution over $\set{{{\bbw_r(0)}}}$ and $\Expect{\bw \sim P_m}{\tilde v_R(\bw)} = \tilde v_R({\bW(0)},\cdot)$.

Given $\tbu \in N_{s}(\unitsphere{d}, \ltwonorm{\cdot})$, we aim to estimate the upper bound for the supremum of empirical process $\Expect{ \bw \sim \cN(\bzero, \kappa^2 \bI_{d+1}) }{\tilde v_R(\bw)} - \Expect{\bw \sim P_m}{\tilde v_R(\bw)}$ when function $\tilde v_R$ ranges over the function class $\cV_{\tbu,s}$. To this end, we apply Theorem~\ref{theorem:Talagrand-inequality} to the function class $\cV_{\tbu,s}$ with ${\bW(0)} = \set{{{\bbw_r(0)}}}_{r=1}^m$. It can be verified that $\tilde v_R \in [0,1]$ for any $\tilde v_R \in \cV_{\tbu,s}$. It follows that we can set $a=0,b=1$ in Theorem~\ref{theorem:Talagrand-inequality}.  Setting $\alpha = \frac 12$ and
$r = 1$ in Theorem~\ref{theorem:Talagrand-inequality} since $\Var{\tilde v_R} \le \Expect{\bw}{\tilde v_R(\bw,\tbu)^2} \le 1$, then with probability at least $1-\delta$ over the random initialization ${\bW(0)}$,
\bal\label{eq:sup-hat-v-seg2}
\sup_{\tbu' \in \ball{}{\tbu}{s} \cap \unitsphere{d}} \abth{\tilde v_R({\bW(0)},\tbu') - \Prob{\abth{\bw^{\top} \tbu'} \le R} }
&\le 3 \cR(\cV_{\tbu,s}) + \sqrt{\frac{2{\log{\frac 1\delta}}}{m}} + \frac{7 {\log{\frac 1 \delta}} }{3m}.
\eal%
where  $\cR(\cV_{\tbu,s}) = \Expect{{\bW(0)}, \set{\sigma_r}_{r=1}^m}{ \sup_{\tilde v_R \in \cV_{\tbu,s}} {\frac{1}{m} \sum\limits_{r=1}^m {\sigma_r}{ \tilde v_R({{\bbw_r(0)}}) }}}$ is the Rademacher complexity of the function class $\cV_{\tbu,s}$, $\set{\sigma_r}_{r=1}^m$ are i.i.d. Rademacher random variables taking values of $\pm 1$ with equal probability. We set $s = 1/{\sqrt m}$.
By Lemma~\ref{lemma:rademacher-bound-hat-v}, $\cR(\cV_{\tbu,s}) \le \sqrt{\frac{d}{\kappa}} m^{-\frac 15} T^{\frac 12}$ with $\eta \lsim 1$, $m \gsim 1$, and $d \ge 4$. Plugging such upper bound for $\cR(\cV_{\tbu,s})$ in (\ref{eq:sup-hat-v-seg2}), we have
\bal\label{eq:sup-hat-v-seg3}
&\sup_{\tbu' \in \ball{}{\tbu}{s} \cap \unitsphere{d}} \abth{ \tilde v_R({\bW(0)},\tbu') - \Prob{\abth{\bw^{\top} \tbu'} \le R} }
\le 3 \sqrt{\frac{d}{\kappa}} m^{-\frac 15} T^{\frac 12} + \sqrt{\frac{2{\log{\frac 1\delta}}}{m}} + \frac{7 {\log{\frac 1\delta}} }{3m}.
\eal%
By union bound, with probability at least $1 - \pth{1+2{\sqrt m}}^{d+1} \delta$ over ${\bW(0)}$, (\ref{eq:sup-hat-v-seg3}) holds for arbitrary $\tbu \in N(\unitsphere{d}, s)$. In this case, for any $\tbu' \in \unitsphere{d}$, there exists $\tbu \in N(\unitsphere{d}, s)$ such that $\norm{\tbu' - \tbu}{2} \le s$, so that $\tbu' \in \ball{}{\tbu}{s} \cap \unitsphere{d}$, and (\ref{eq:sup-hat-v-seg3}) holds.
Note that $\Prob{\abth{\bw^{\top} \tbu'} \le R} \le \frac{2R}{\sqrt {2\pi} \kappa}$ for any $\tbu' \in \unitsphere{d}$, changing the notation $\tbu'$ to $\tbu$ completes the proof.

\end{proof}

%$\tau$ is an arbitrary constant such that $\tau \in (0,\frac 12)$.
\begin{lemma}\label{lemma:rademacher-bound-hat-v}
 Let
\bals
\cR(\cV_{\tbu,s}) \defeq \Expect{{\bW(0)}, \set{\sigma_r}_{r=1}^m}{ \sup_{v_R
\in \cV_{\tbu,s}} {\frac{1}{m} \sum\limits_{r=1}^m {\sigma_r}{ v_R({{\bbw_r(0)}}) }}}
\eals
be the Rademacher complexity of the function class $\cV_{\tbu,s}$ defined in (\ref{eq:sup-hat-v-seg1}). Then
\bal\label{eq:rademacher-bound-hat-v}
&\cR(\cV_{\tbu,s}) \le (B\sqrt{d}+1)\sqrt{\frac{m^{\frac{1}{2(d+1)}} R}{\kappa}+s} + \frac{2}{m^{\frac 12}(d+1){\sqrt {\frac{m^{\frac{1}{2(d+1)}} R}{\kappa}+s}}},
%\le \sqrt{\frac{d}{\kappa}} m^{-\frac 15} T^{\frac 12},
\eal%
where $B$ is a positive constant.
\end{lemma}
\begin{proof}
%Let $\tbu = \bth{\bu^{\top},1}^{\top}$, $\tbu' = \bth{\bu'^{\top},1}^{\top}$.
We have
\bal\label{eq:rademacher-bound-hat-v-seg1}
\cR(\cV_{\tbu,s})
&= \Expect{{\bW(0)}, \set{\sigma_r}_{r=1}^m}{ \sup_{\tbu' \in \ball{}{\tbu}{s} \cap \unitsphere{d}} {\frac{1}{m} \sum\limits_{r=1}^m {\sigma_r}{ \tilde v_R({{\bbw_r(0)}}, \tbu') }}} \le \cR_1 + \cR_2,
\eal%
where
\bals%\label{eq:rademacher-bound-hat-v-seg2}
&\cR_1 = \Expect{{\bW(0)}, \set{\sigma_r}_{r=1}^m}{ \sup_{\tbu' \in \ball{}{\tbu}{s}\cap \unitsphere{d}} {\frac{1}{m} \sum\limits_{r=1}^m {\sigma_r}{ \tilde v_R({{\bbw_r(0)}}, \tbu) }}}, \nonumber \\
&\cR_2 = \Expect{{\bW(0)}, \set{\sigma_r}_{r=1}^m}{ \sup_{\tbu' \in \ball{}{\tbu}{s}\cap \unitsphere{d}} {\frac{1}{m} \sum\limits_{r=1}^m {\sigma_r}{ \pth{ \tilde v_R({{\bbw_r(0)}}, \tbu') - \tilde v_R({{\bbw_r(0)}}, \tbu)} }}}.
\eals
Here (\ref{eq:rademacher-bound-hat-v-seg1}) follows from the subadditivity of supremum.
Now we bound $\cR_1$ and $\cR_2$ separately. For $\cR_1$, we have
\bal\label{eq:rademacher-bound-hat-v-seg3}
\cR_1 &= \Expect{{\bW(0)}, \set{\sigma_r}_{r=1}^m}{ \sup_{\tbu' \in \ball{}{\tbu}{s}\cap \unitsphere{d}} {\frac{1}{m} \sum\limits_{r=1}^m {\sigma_r}{ \tilde v_R({{\bbw_r(0)}}, \tbu) }}}  = 0.
\eal
For $\cR_2$, we first define
$Q \defeq \frac 1m \sum\limits_{r=1}^m
\indict{\indict{\abth{\tbu'^{\top}{{\bbw_r(0)}}} \le R} \neq \indict{\abth{\tbu^{\top} {{\bbw_r(0)}}} \le R} }$,
which is the number of weights in ${\bW(0)}$ whose inner products with $\tbu$ and $\tbu'$ have different signs. Note that if $\abth{\abth{ \tbu^{\top} {{\bbw_r(0)}}} - R} > s \norm{{{\bbw_r(0)}}}{2}$, then $\indict{\abth{\tbu^{\top} {{\bbw_r(0)}}} \le R} = \indict{\abth{\tbu'^{\top} {{\bbw_r(0)}}} \le R}$. To see this, by the Cauchy-Schwarz inequality,
$\abth{\tbu'^{\top} {{\bbw_r(0)}} - \tbu^{\top} {{\bbw_r(0)}}} \le \norm{\tbu'-\tbu}{2} \norm{{{\bbw_r(0)}}}{2} \le s \norm{{{\bbw_r(0)}}}{2}$,
then we have $\abth{\tbu^{\top} {{\bbw_r(0)}}} - R > s \norm{{{\bbw_r(0)}}}{2} \Rightarrow \abth{\tbu'^{\top} {{\bbw_r(0)}}} - R \ge \abth{\tbu^{\top} {{\bbw_r(0)}}}  - s \norm{{{\bbw_r(0)}}}{2} -R> 0$, and  $\abth{\tbu^{\top} {{\bbw_r(0)}}} - R < -s \norm{{{\bbw_r(0)}}}{2} \Rightarrow \abth{\tbu'^{\top} {{\bbw_r(0)}}} - R \le \abth{\tbu^{\top} {{\bbw_r(0)}}} + s \norm{{{\bbw_r(0)}}}{2} - R < 0$.
As a result,
\bals
Q \le \frac 1m \sum\limits_{r=1}^m \indict{\abth{\abth{ \tbu^{\top} {{\bbw_r(0)}}} - R} \le s \norm{{{\bbw_r(0)}}}{2}} \defeq \tilde Q.
\eals
Due to the fact that $\indict{\abth{\abth{\tbu^{\top} {{\bbw_r(0)}}}- R} \le s \norm{{{\bbw_r(0)}}}{2}} \le \indict{\abth{ \tbu^{\top} {{\bbw_r(0)}}} \le R+s \norm{{{\bbw_r(0)}}}{2}}$, we have
\bal\label{eq:rademacher-bound-hat-v-seg5-pre2}
&\Expect{{\bW(0)}} { \indict{\abth{\abth{\bx^{\top} {{\bbw_r(0)}}}-R} \le s \norm{{{\bbw_r(0)}}}{2}} } \le \Expect{{{\bbw_r(0)}}} { \indict{\abth{ \bx^{\top} {{\bbw_r(0)}}} \le R+s \norm{{{\bbw_r(0)}}}{2}} } \nonumber \\
&\stackrel{\circled{1}}{\le} \Expect{
\ltwonorm{\bbw_r(0)} > \frac{\kappa}{m^{\frac{1}{2(d+1)}}} }{\indict{\abth{ \bx^{\top} {{\bbw_r(0)}}} \le R+s \norm{{{\bbw_r(0)}}}{2}}} +
\Expect{ \ltwonorm{\bbw_r(0)} \le
\frac{\kappa}{m^{\frac{1}{2(d+1)}}} }{\indict{\abth{ \bx^{\top} {{\bbw_r(0)}}} \le R+s \norm{{{\bbw_r(0)}}}{2}}} \nonumber \\
&\stackrel{\circled{2}}{\le} \Expect{ \ltwonorm{\bbw_r(0)} >
\frac{\kappa}{m^{\frac{1}{2(d+1)}}}  }{\indict{\abth{ \bx^{\top} {{\bbw_r(0)}}} \le \pth{\frac{m^{\frac{1}{2(d+1)}} R}{\kappa}+s}\norm{{{\bbw_r(0)}}}{2}
}} + \frac{2}{m^{\frac 12}(d+1)} \nonumber \\
&\le \Prob{ \abth{\bx^{\top}{{\bbw_r(0)}}}/ \ltwonorm{{{\bbw_r(0)}}}  \le \frac{m^{\frac{1}{2(d+1)}} R}{\kappa}+s }
+  \frac{2}{m^{\frac 12}(d+1)} \nonumber \\
&\stackrel{\circled{3}}{\le}  B \sqrt{d} \pth{\frac{m^{\frac{1}{2(d+1)}} R}{\kappa}+s} +  \frac{2}{m^{\frac 12}(d+1)},
\eal%
where we used the convention that $\Expect{{{\bbw_r(0)}} \in A}{\cdot} = \Expect{{{\bbw_r(0)}}}{\indict{A}
\times \cdot}$ in $\circled{1}$ with $A$ being an event.
$\circled{2}$ follows from Lemma~\ref{lemma:concentration-l2norm-Gaussian}. By Lemma~\ref{lemma:uniform-marginal-bound}, $\Prob{ \frac{ \abth{\bx^{\top}{{\bbw_r(0)}}} } { \norm{{{\bbw_r(0)}}}{2} }  \le \frac{m^{\frac{1}{4(d+1)}} R}{\kappa}+s} \le B \sqrt{d} \pth{\frac{m^{\frac{1}{4(d+1)}} R}{\kappa}+s}$ for an absolute constant $B$, so $\circled{3}$ holds.
According to (\ref{eq:rademacher-bound-hat-v-seg5-pre2}), we have
\bal\label{eq:rademacher-bound-hat-v-seg5}
\Expect{{\bW(0)}}{\tilde Q} &\le B \sqrt{d}
\pth{\frac{m^{\frac{1}{2(d+1)}} R}{\kappa}+s} + \frac{2}{m^{\frac 12}(d+1)}.
\eal%
Define $s' \defeq \frac{m^{\frac{1}{2(d+1)}} R}{\kappa}+s$. By Markov's inequality, we have
\bal\label{eq:rademacher-bound-hat-v-seg7}
&\Prob{\tilde Q \ge \sqrt{s'}} \le B \sqrt{ds'}
+ \frac{2}{m^{\frac 12}(d+1){\sqrt {s'}}},
\eal%
where the probability is with respect to the probability measure space of ${\bW(0)}$. Let $\Omega_s$ be the subset of the  probability measure space of ${\bW(0)}$ such that $Q \ge \sqrt{s'}$. Now we aim to bound $\cR_2$ by estimating its bound on $\Omega_s$ and its complement. First, we have
\bal\label{eq:rademacher-bound-hat-v-seg8}
\cR_2 &= \Expect{{\bW(0)}, \set{\sigma_r}_{r=1}^m}{ \sup_{\tbu' \in \ball{}{\tbu}{s} \cap \unitsphere{d}} {\frac{1}{m} \sum\limits_{r=1}^m {\sigma_r}{ \pth{ \tilde v_R({{\bbw_r(0)}}, \tbu') - \tilde v_R({{\bbw_r(0)}}, \tbu)} }}} \nonumber \\
&= \underbrace{\Expect{{\bW(0)} \colon \tilde Q \ge \sqrt{s'} , \set{\sigma_r}_{r=1}^m}{ \sup_{\tbu' \in \ball{}{\tbu}{s}\cap \unitsphere{d}} {\frac{1}{m} \sum\limits_{r=1}^m {\sigma_r}{ \pth{ \tilde v_R({{\bbw_r(0)}}, \tbu') - \tilde v_R({{\bbw_r(0)}}, \tbu)} }}}}_{\textup{$\cR_{21}$}}  \nonumber \\
&\phantom{=}+ \underbrace{\Expect{{\bW(0)} \colon \tilde Q < \sqrt{s'}, \set{\sigma_r}_{r=1}^m}{ \sup_{\tbu' \in \ball{}{\tbu}{s}\cap \unitsphere{d}} {\frac{1}{m} \sum\limits_{r=1}^m {\sigma_r}{ \pth{ \tilde v_R({{\bbw_r(0)}}, \tbu') - \tilde v_R({{\bbw_r(0)}}, \tbu)} }}}}_{\textup{$\cR_{22}$}},
\eal%
Now we estimate the upper bound for $\cR_{22}$ and $\cR_{21}$ separately.
Let
\bals
I = \set{r \in [m] \colon \indict{\abth{\tbu'^{\top}{{\bbw_r(0)}}} \le R} \neq \indict{\abth{\tbu^{\top} {{\bbw_r(0)}}} \le R} }.
\eals
When $Q \le \tilde Q \le \sqrt{s'}$, $|I| \le m \sqrt{s'}$.  Moreover, when $r \in I$, either $\indict{\abth{\tbu'^{\top}{{\bbw_r(0)}}} \le R} = 0$ or $\indict{\abth{\tbu^{\top}{{\bbw_r(0)}}} \le R} = 0$. As a result,
\bal\label{eq:rademacher-bound-hat-v-seg9}
\abth{\tilde v_R({{\bbw_r(0)}}, \tbu') - \tilde v_R({{\bbw_r(0)}}, \tbu)}
&= \abth{\indict{\abth{\tbu'^{\top}{{\bbw_r(0)}}} \le R} -  \indict{\abth{\tbu^{\top}{{\bbw_r(0)}}} \le R}} \le 1.
\eal
When $r \in [m] \setminus I$, we have
\bal\label{eq:rademacher-bound-hat-v-seg10}
\abth{\tilde v_R({{\bbw_r(0)}}, \tbu') - \tilde v_R({{\bbw_r(0)}}, \tbu)}  = \abth{\indict{\abth{\tbu'^{\top}{{\bbw_r(0)}}} \le R} -  \indict{\abth{\tbu^{\top}{{\bbw_r(0)}}} \le R}} =0.
\eal%
It follows that for any $\tbu' \in \ball{}{\tbu}{s}\cap \unitsphere{d}$,
\bal\label{eq:rademacher-bound-hat-v-seg11}
&\frac{1}{m} \sum\limits_{r=1}^m {\sigma_r}{ \pth{ \tilde v_R({{\bbw_r(0)}}, \tbu') - \tilde v_R({{\bbw_r(0)}}, \tbu)} } \nonumber \\
&=\frac{1}{m} \sum\limits_{r \in I} {\sigma_r}{ \pth{ \tilde v_R({{\bbw_r(0)}}, \tbu') - \tilde v_R({{\bbw_r(0)}}, \tbu)} } + \frac{1}{m} \sum\limits_{r \in [m] \setminus I} {\sigma_r}{ \pth{ \tilde v_R({{\bbw_r(0)}}, \tbu') - \tilde v_R({{\bbw_r(0)}}, \tbu)} } \nonumber \\
&\le \frac{1}{m} \sum\limits_{r \in I} { \abth{ \tilde v_R({{\bbw_r(0)}}, \tbu') - \tilde v_R({{\bbw_r(0)}}, \tbu)} } + \frac{1}{m} \sum\limits_{r \in [m] \setminus I} { \abth{ \tilde v_R({{\bbw_r(0)}}, \tbu') - \tilde v_R({{\bbw_r(0)}}, \tbu)} } \stackrel{\circled{1}}{\le}   \sqrt{s'},
\eal%
where $\circled{1}$ follows from (\ref{eq:rademacher-bound-hat-v-seg9}) and (\ref{eq:rademacher-bound-hat-v-seg10}).
Using (\ref{eq:rademacher-bound-hat-v-seg11}), we now estimate the upper bound for $\cR_{22}$ by
\bal\label{eq:rademacher-bound-hat-v-seg12}
\cR_{22} &=\Expect{{\bW(0)} \colon \tilde Q < \sqrt{s'}, \set{\sigma_r}_{r=1}^m}{ \sup_{\tbu' \in \ball{}{\tbu}{s}\cap \unitsphere{d}} {\frac{1}{m} \sum\limits_{r=1}^m {\sigma_r}{ \pth{ \tilde v_R({{\bbw_r(0)}}, \tbu') - \tilde v_R({{\bbw_r(0)}}, \tbu)} }}} \nonumber \\
&\le \Expect{{\bW(0)} \colon \tilde Q < \sqrt{s'}, \set{\sigma_r}_{r=1}^m}{\sqrt{s'}} \le \sqrt{s'}.
\eal%
When $\tilde Q \ge \sqrt{s'}$, we still have
$\abth{\tilde v_R({{\bbw_r(0)}}, \tbu') - \tilde v_R({{\bbw_r(0)}}, \tbu)} \le 1$ by (\ref{eq:rademacher-bound-hat-v-seg9}). For $\cR_{21}$, we have
\bal\label{eq:rademacher-bound-hat-v-seg14}
\cR_{21} &= \Expect{{\bW(0)} \colon \tilde Q \ge \sqrt{s'}, \set{\sigma_r}_{r=1}^m}{ \sup_{\tbu' \in \ball{}{\tbu}{s}\cap \unitsphere{d}} {\frac{1}{m} \sum\limits_{r=1}^m {\sigma_r}{ \pth{ \tilde v_R({{\bbw_r(0)}}, \tbu') - \tilde v_R({{\bbw_r(0)}}, \tbu)} }}} \nonumber \\
&\le \Expect{{\bW(0)} \colon Q \ge \sqrt{s'}, \set{\sigma_r}_{r=1}^m}{  1 }  = \Prob{\tilde Q\ge \sqrt{s'}} \le B \sqrt{ds'}
+ \frac{2}{m(d+1){\sqrt {s'}}}
\eal%
where the last inequality follows from (\ref{eq:rademacher-bound-hat-v-seg7}). Combining (\ref{eq:rademacher-bound-hat-v-seg8}), (\ref{eq:rademacher-bound-hat-v-seg12}), and (\ref{eq:rademacher-bound-hat-v-seg14}), we have the upper bound for $\cR_2$ as
\bal\label{eq:rademacher-bound-hat-v-seg15}
\cR_2 &=  \cR_{21} + \cR_{22} \le (B\sqrt{d}+1)\sqrt{s'} + \frac{2}{m^{\frac 12}(d+1){\sqrt {s'}}}.
\eal%
Plugging (\ref{eq:rademacher-bound-hat-v-seg3}) and (\ref{eq:rademacher-bound-hat-v-seg15}) in (\ref{eq:rademacher-bound-hat-v-seg1}), we have
$\cR(\cV_{\tbu,s}) \le \cR_1 + \cR_2 \le (B\sqrt{d}+1)\sqrt{s'} + \frac{2}{m^{\frac 12}(d+1){\sqrt {s'}}}$,
which completes the proof.

\end{proof}

%\begin{lemma}\label{lemma:concentration-l2norm-Gaussian}
%Let $\bw \in \RR^{d+1}$ be a Gaussian random vector distribute according to $\bw \sim \cN(\bzero,\kappa^2 \bI_{d+1})$. Then $\Prob{\ltwonorm{\bw} \ge \kappa \sqrt {d/2}} \ge 1 - \exp\pth{-d/16}$.
%\end{lemma}
%\begin{proof}
%Let $X = \frac{\norm{\bw}{2}^2}{\kappa^2}$, then $X$ follows the chi-square distribution with $d$ degrees of freedom, that is, $X \sim \chi^2(d+1)$. By~\cite[Lemma 1]{Laurent2000-quadratic-functional}, we have the following concentration inequalities for any $x > 0$, $\Prob{X-(d+1) \ge 2 \sqrt{(d+1)x} + 2x} \le \exp(-x), \Prob{(d+1)-X \ge 2 \sqrt{(d+1)x}} \le \exp(-x)$.
%Then this lemma follows from setting $x = (d+1)/16$ in the second inequality.
%\end{proof}

\begin{lemma}\label{lemma:concentration-l2norm-Gaussian}
Let $\bw \in \RR^{d+1}$ be a Gaussian random vector distribute according to $\bw \sim \cN(\bzero,\kappa^2 \bI_{d+1})$, and $m \ge \Theta(1)$ and $d \ge 4$. Then $\Prob{\ltwonorm{\bw} \le \kappa/m^{\frac{1}{2(d+1)}} } \le 2/(m^{\frac 12}(d+1))$.
\end{lemma}
\begin{proof}
Let $X = {\norm{\bw}{2}^2}/{\kappa^2}$, then $X$ follows the chi-square distribution with $d$ degrees of freedom, that is, $X \sim \chi^2(d+1)$, with the PDF
$f(x; d+1) = x^{(d+1)/2-1} e^{-x/2}/(2^{(d+1)/2} \Gamma((d+1)/2))$ for $x > 0$ and $f(x; d+1) = 0$ for all $x \le 0$.
Using the approximation to the Gamma function \cite{Gosper1978-dr}
$\Gamma(x) \asymp x^{x-0.5} \exp(-x) \sqrt{2\pi}$, we have
$\Prob{X \le {1}/{m^{\frac{1}{d+1}}}} \le 2/(m^{\frac 12}(d+1))$, which proves this lemma.
To the last inequality, we note that there exists an remainder function $r(x) \in [1/(12x+1),1/(12x)]$, and
$\Gamma(x) = x^{x-0.5} \exp(-x+r(x)) \sqrt{2\pi}$. For all $d \ge 4$, we have
\bals
\Prob{X \le {1}/{m^{\frac{1}{d+1}}}} &\le \frac{ \frac{1}{\pth{m^{\frac{1}{d+1}}}^{(d-1)/2}} \cdot \frac{1}{m^{\frac{1}{d+1}}} }{2^{ (d+1)/2} \cdot \pth{\frac{d+1}{2}}^{d/2} \exp\pth{-\frac{d+2}{2}
+ r\pth{\frac{d+1}2}}  \sqrt{2\pi} } \nonumber \\
& =\frac{1}{\sqrt{m}} \cdot \frac{ 1 }{2\sqrt{\pi} e^{-1} \cdot (d+1)^{d/2} \exp\pth{-\frac{d}{2} + r\pth{\frac{d+1}2}} } \nonumber \\
&\le \frac{2}{\sqrt{m}(d+1)} \cdot \frac{e}{2} \cdot \frac{\exp(d/2-1)}{ (d+1)^{d/2 - 1}  } < \frac{2}{\sqrt{m}(d+1)}.
\eals
\end{proof}

\section{More Result about the Eigenvalue Decay Rate}
\label{sec:more-results-edr}
\begin{proposition}
\label{proposition:EDR-unitsphere-Sd}
If $\cX = \unitsphere{d-1}$ and $P = \Unif{\unitsphere{d-1}}$, then
the polynomial EDR $\lambda_j \asymp j^{-(d+1)/d}$ holds
for all $j \ge 1$, where
$\set{\lambda_j}_{j \ge 1}$ are the eigenvalues of the integral operator associated with the NTK $K$ defined in (\ref{eq:kernel-two-layer}).
Furthermore, if the probability density function of the distribution $P$ satisfies  $p(\bx) \lsim ({1+\ltwonorm{\bx}^2})^{-(d+3)}$, then the
same polynomial EDR still holds.
\end{proposition}
\begin{proof}
The proof follows by applying
\cite[Theorem 10]{Li2024-edr-general-domain}. First, it can be verified that the probability
density function of $P$ is $p(\bx)  = 1/S_{d-1}$
for all $\bx \in \unitsphere{d-1}$ and $p(\bx) = 0$ for all
$\bx \in \RR^d \setminus \unitsphere{d-1}$. As a result, under the setting of fixed $d$ in this paper, we have
$p(\bx) \lsim ({1+\ltwonorm{\bx}^2})^{-(d+3)}$ .

Let $\mu_j(K,\cX,\mu)$ be the $j$-th eigenvalue of the integral operator
$T_K$ associated with $K$ with distribution $\mu$ supported on $\cX$.
Define the kernel $\hat K(\bu,\bv) \defeq
\iprod{\tbu/\ltwonorm{\tbu}}{\tbv/\ltwonorm{\tbv}}
\pth{\pi - \arccos \iprod{\tbu/\ltwonorm{\tbu}}{\tbv/\ltwonorm{\tbv}} }/(2\pi)$
for all $\bu,\bv \in \cX$ and
another kernel $K_0(\bx,\bx') \defeq
\iprod{\bx}{\bx'}
\pth{\pi - \arccos \iprod{\bx}{\bx'}  }/(2\pi)$
for all $\bx,\bx' \in (\unitsphere{d})^+$,
where $(\unitsphere{d})^+ \defeq
\set{\bx' \in \RR^{d+1}, \bx' \in \unitsphere{d},
\bx'_{d+1} > 0}$.
We define the function
$\phi \colon \cX \to (\unitsphere{d})^+,
\phi(\bx) = \tbx/\ltwonorm{\tbx}$ for all $\bx \in \cX$.
Then it can be verified that the Jacobian and the Gram matrix for
$\phi$ is $J \phi =  {1}/{\ltwonorm{\tbx}} \cdot \bth{\bI_d; \bzero} -
\tbx \bx^{\top}/\ltwonorm{\tbx}^3 $, and $G = (J \phi)^{\top} J \phi$ with
$\abth{\det G} = \ltwonorm{\tbx}^{-2(d+1)}$.

For any
probability measure $\mu'$ on $\cX$, the push-forward probability measure
on $(\unitsphere{d})^+$, denoted as $\phi^* \mu'$ which is
induced by $\phi$, is defined by $(\phi^* \mu')(A)
\defeq \mu' (\phi^{-1}(\bA))$ for any set $A
\subseteq (\unitsphere{d})^+$.
We recall that $\mu$ is the probability measure of $P$.
Then it follows by
repeating the proof of \cite[Theorem 10]{Li2024-edr-general-domain} that
that
$\mu_j(K,\cX,\mu) = \mu_j(\hat K,\cX, q^2 \mu)
= \mu_j(K_0,(\unitsphere{d})^+ ,\phi^*(q^2 \mu))$ with
$q(\bx) \defeq \ltwonorm{\tbx}$ for all $\bx \in \cX$, and $K$ is defined in (\ref{eq:kernel-two-layer}).
Let $\tilde \mu  = \phi^*(q^2 \mu)$, then
$\tilde \mu = q^2 \phi^* \mu = q^2 p \cdot \phi^*(\diff \bx)$ where
$\diff \bx$ is the usual Lebesgue measure on $\cX$. Let $\tilde \sigma$ be the uniform measure
on $(\unitsphere{d})^+$, then it can be verified that
$\tilde \sigma = \abth{\det G}^{1/2} \phi^* (\diff \bx) $, and
it follows that $\tilde \mu(\tbx) =
\abth{\det G}^{-1/2} q^2 p \tilde \sigma(\tbx)
=\ltwonorm{\tbx}^{(d+3)}p(\bx) \cdot \tilde \sigma(\tbx) $.

Since $p(\bx) \lsim ({1+\ltwonorm{\bx}^2})^{-(d+3)}$,
we have $\ltwonorm{\tbx}^{(d+3)}p \lsim 1$, so it follows from
\cite[Theorem 8]{Li2024-edr-general-domain} that
$\mu_j(K_0,(\unitsphere{d})^+ ,\phi^*(q^2 \mu)) \asymp
\mu_j(K_0,(\unitsphere{d})^+,\tilde \sigma)$. Moreover, it follows
from \cite{BiettiM19,BiettiB21} that
$\mu_j(K_0,(\unitsphere{d})^+,\tilde \sigma)
 \asymp j^{-(d+1)/d}$, which completes the first part of this proposition.

For the case that $p(\bx) \lsim ({1+\ltwonorm{\bx}^2})^{-(d+3)}$, the same polynomial EDR can be obtained by repeating the above argument.
\end{proof}

\begin{remark}
\label{remark:edr-on-unitsphere-in-Rd}
[Another special case for the eigenvalue decay rate.]
We consider the case that $\cX = \unitsphere{d-1}$ and $\bth{\bbw_r}_{d+1} = 0$ for all $r \in [m]$ when training the neural network
 (\ref{eq:two-layer-nn}) by GD in Algorithm~\ref{alg:GD}, or equivalently, $\tbx = \bx$ in
 the neural network (\ref{eq:two-layer-nn}) with all the weights
 $\set{\bbw_r \in \RR^d}_{r=1}^m$ initialized by
 $\cN(\bzero,\kappa^2 \bI_{d})$. In this case we let
 $f^* \in \cH_{K_1}(\mu_0)$ where $K_1 (\bx,\bx') \defeq
 \iprod{\bx}{\bx'} \pth{\pi - \arccos  \iprod{\bx}{\bx'} }/(2\pi)$ for all $\bx,\bx' \in \cX$.
 It can be verified by repeating the proof of Theorem~\ref{theorem:LRC-population-NN-fixed-point}
 and all the results leading to
 Theorem~\ref{theorem:LRC-population-NN-fixed-point}
 that Theorem~\ref{theorem:LRC-population-NN-fixed-point}
 still hold. Furthermore, it follows
from \cite{BiettiM19,BiettiB21} that the polynomial EDR, $\lambda_j \asymp j^{-d/(d-1)}$ for all $j \ge 1$, holds for $K_1$.
In this case $\eps^2_n \asymp n^{-\frac{d}{2d-1}}$ according to
\cite[Corollary 3]{RaskuttiWY14-early-stopping-kernel-regression},
 where $\eps_n$ is the critical population rate of the kernel $K_1$.
As a result, the rate of nonparametric regression risk is
$\eps^2_n \asymp n^{-\frac{d}{2d-1}}$ which is the same minimax optimal rate
obtained by \cite{HuWLC21-regularization-minimax-uniform-spherical,
SuhKH22-overparameterized-gd-minimax}. In this way, we obtain
such minimax optimal rates obtained by \cite{HuWLC21-regularization-minimax-uniform-spherical,
SuhKH22-overparameterized-gd-minimax} as a special case of
 Theorem~\ref{theorem:LRC-population-NN-fixed-point}.
\end{remark}
\begin{figure}[!htbp]
\begin{center}
\includegraphics[width=0.7\textwidth]{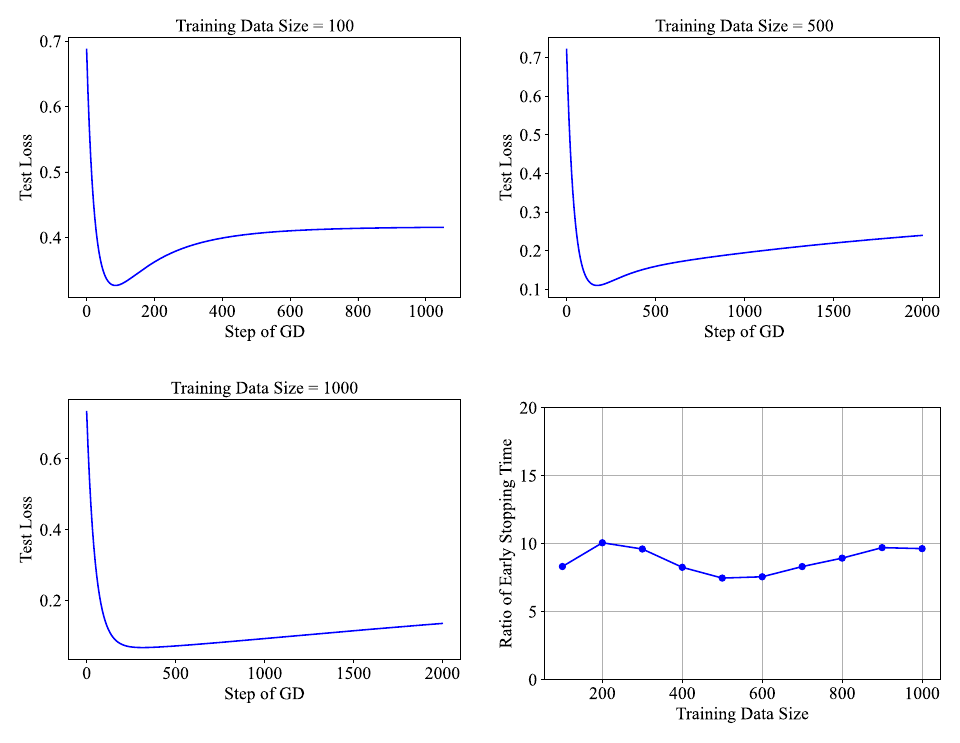}
\end{center}
\caption{Illustration of the test loss by GD and the ratio of early stopping time.}
\label{fig:simulation}
\end{figure}
\section{Simulation Study}
\label{sec:simulation}

We present simulation results in this section.
We randomly sample $n$ points $\set{\bbx_i}_{i=1}^n$  as a i.i.d. sample of random
variables distributed uniformly on the unit sphere in
$\RR^{50}$.  $n$ ranges within $[100,1000]$ with a step size of $100$.
We set the target function to $f^*(\bx) = \bs^{\top} \bx$
where $\bs \sim \Unif{\cX}$ is randomly sampled.
We also uniformly and independently sample $1000$ points on the unit sphere in $\RR^{50}$ as the test data.
We train the two-layer NN (\ref{eq:two-layer-nn}) using GD by
Algoirthm~\ref{alg:GD}
with $m \asymp n^2$ on a NVIDIA A100 GPU card with a learning rate $\eta = 0.1$, and report the test loss
in Figure~\ref{fig:simulation}. It can be observed that the early-stopping mechanism is always helpful in training neural networks with better generalization,
as the test loss initially decreases and then increases with over-training.
Figure~\ref{fig:simulation} illustrates the test loss with respect to the steps (or epochs) of GD for $n=100,500,1000$. For each $n$ in $[100,1000]$ with a step size of $100$,
we find the step of GD $\hat t_n$ where the minimum test loss is achieved, which is the empirical early stopping time.
We note that the early stopping time theoretically predicted by Corollary~\ref{corollary:minimax-nonparametric-regression}
 is $1/\hat \eps^2_n \asymp n^{(d+1)/(2d+1)}$, and we compute the ratio of early stopping time for each $n$ by
$\hat t_n/n^{(d+1)/(2d+1)}$. Such ratios for different values of $n$ are illustrated in the bottom right figure of Figure~\ref{fig:simulation}. It is observed that the ratio of early stopping time is roughly stable and distributed between
$[8,10]$, suggesting that the theoretically predicted early stopping time is empirically proportional to the empirical early stopping time.

\end{appendices}

%\newpage
\bibliographystyle{IEEEtran}
%\bibliographystyle{authordate3}
%\bibliography{ref,permute_var}
\bibliography{ref}

\end{document}

%% file: symbol_tit.tex
\newcounter{optproblem}

\newtheoremstyle{mytheoremstyle} % name
    {\topsep}                    % Space above
    {\topsep}                    % Space below
    {\normalfont}                % Body font
    {}                           % Indent amount
    {\bfseries}                   % Theorem head font
    {.}                          % Punctuation after theorem head
    {.5em}                       % Space after theorem head
    {}  % Theorem head spec (can be left empty, meaning 'normal�)

\theoremstyle{mytheoremstyle}
\newtheorem{theorem}{Theorem}[section]
\newtheorem{remark}[theorem]{Remark}
\newtheorem{proposition}[theorem]{Proposition}
\newtheorem{corollary}[theorem]{Corollary}

\newtheorem*{theorem*}{Theorem}
\newtheorem*{lemma*}{Lemma}
\newtheorem*{remark*}{Remark}

\newtheorem{lemma}[theorem]{Lemma}%[section]

\theoremstyle{remark}
\newtheorem{definition}{Definition}[section]

\DeclareMathAlphabet{\pazocal}{OMS}{zplm}{m}{n}
\DeclareMathAlphabet{\mathpzc}{OMS}{pzc}{m}{it}

\setlist[itemize]{leftmargin=*}

\renewcommand{\hat}{\widehat}

\newcommand{\bfm}[1]{\ensuremath{\mathbf{#1}}}
\newcommand{\bfsym}[1]{\ensuremath{\boldsymbol{#1}}}

\def\ba{\boldsymbol a}   \def\bA{\bfm A}

   \def\bD{\bfm D}  
\def\be{\bfm e}   \def\bE{\bfm E}

   \def\bH{\bfm H}  
   \def\bI{\bfm I}  
     
   \def\bK{\bfm K}

     \def\NN{\mathbb{N}}

   \def\bQ{\bfm Q}  
   \def\bR{\bfm R}  \def\RR{\mathbb{R}}
\def\bs{\bfm s}   \def\bS{\bfm S}  
     
\def\bu{\bfm u}   \def\bU{\bfm U}  
\def\bv{\bfm v}     
\def\bw{\bfm w}   \def\bW{\bfm W}  
\def\bx{\bfm x}   \def\bX{\bfm X}  
\def\by{\bfm y}     
   \def\bZ{\bfm Z}  
\def\bzero{\bfm 0}

 \def\cB{{\cal  B}}

 \def\cE{{\cal  E}}
 \def\cF{{\cal  F}}
 
 \def\cH{{\cal  H}}

 \def\cN{{\cal  N}}
 \def\cO{{\cal  O}}

 \def\cR{{\cal  R}}

 \def\cV{{\cal  V}}
 \def\cW{{\cal  W}}
 \def\cX{{\cal  X}}

\def\balpha{\bfsym \alpha}
\def\bbeta{\bfsym \beta}

\def\bmu{\bfsym {\mu}}

\def\bsigma{\bfsym \sigma}
\def\bSigma{\bfsym \Sigma}

\def\hlambda{\hat{\lambda}}

\def\+#1{\mathcal{#1}}
\def\-#1{\textup{#1}}

\def\set#1{\left\{ #1 \right\}}
\def\pth#1{\left( #1 \right)}
\def\bth#1{\left[ #1 \right]}
\def\abth#1{\left | #1 \right |}

\def\defeq {\coloneqq}

\newcommand{\vect}[1]{{\textup{vec}\pth{#1}}}

\DeclareMathSymbol{\relcolon}{\mathrel}{operators}{"3A}

\newcommand{\La}{\left\langle\kern-0.64ex\left\langle}
\newcommand{\Ra}{\right\rangle\kern-0.64ex\right\rangle}

\def\Norm#1#2{{\left\vert\kern-0.4ex\left\vert\kern-0.4ex\left\vert #1
    \right\vert\kern-0.4ex\right\vert\kern-0.4ex\right\vert}_{#2}}
\def\norm#1#2{{\left\|#1\right\|}_{#2}}

\def\ltwonorm#1{\norm{#1}{2}}
\def\fnorm#1{\norm{#1}{\textup{F}}}
\def\supnorm#1{\norm{#1}{\infty}}

\def \Proj  {\mathbb{P}}

\def\tr#1{\textup{tr}\left(#1\right)}

\newcommand{\1}{{\rm 1}\kern-0.25em{\rm I}}
\def\indict#1{{\rm 1}\kern-0.25em{\rm I}_{\set{#1}}}

\def \eps  {\epsilon}
\def \eps {\varepsilon}

\def \diff {{\rm d}}
\def \iprod#1#2{\left\langle #1, #2 \right\rangle}
\def\set#1{\left\{#1\right\}}

\def\ball#1#2#3{\bfm{B}^{#1}\left(#2; #3\right)}

\def\unitsphere#1{\mathbb{S}^{#1}}
\def \E {\mathbb{E}}
\def\Expect#1#2{\E_{#1}\left[#2\right]}

\def \Pr {\textup{Pr}}
\newcommand{\Prob}[1]{\Pr\left[#1\right]}
\def \Var#1{\textup{Var}\left[#1\right]}

\def \lsim {\lesssim}
\def \gsim {\gtrsim}

\newcommand{\Unif}[1]{{\mathop\mathrm{Unif}}\left( #1 \right)}

\newcommand{\beq}{\begin{equation}}
\newcommand{\eeq}{\end{equation}}
\newcommand{\beqa}{\begin{eqnarray}}
\newcommand{\eeqa}{\end{eqnarray}}
\newcommand{\beqas}{\begin{eqnarray*}}
\newcommand{\eeqas}{\end{eqnarray*}}
\def\bal#1\eal{\begin{align}#1\end{align}}
\def\bals#1\eals{\begin{align*}#1\end{align*}}
\def\bsal#1\esal{\begin{small}\begin{align}#1\end{align}\end{small}}
\def\bsals#1\esals{\begin{small}\begin{align*}#1\end{align*}\end{small}}
\def\bsfal#1\esfal{\begin{small}\begin{flalign}#1\end{flalign}\end{small}}